\documentclass[11pt,en,cite=authoryear]{note}

\usepackage[libertine]{newtxmath}

\title{A Nonparametric Statistics Approach to Feature Selection in Deep Neural Networks with Theoretical Guarantees}
\author{%
  Junye Du$^{1,\dagger}$,
  Zhenghao Li$^{1,\dagger}$,
  Zhutong Gu$^{2}$,
  Long Feng$^{1,}$\thanks{Corresponding author E-mail: lfeng@hku.hk}
}

\institute{%
  $^1$Department of Statistics and Actuarial Science, The University of Hong Kong \\
  $^2$Peking University HSBC Business School
}

\date{}

\usepackage{appendix}
\usepackage{array}
\usepackage{subfigure}
\usepackage{graphicx}
\usepackage{amssymb,amsmath,amsthm}

\usepackage{xcolor}
\usepackage{graphicx}
\graphicspath{
	{note/figure/}{note/image}{./image/}
}
\let\savehash\hash
\let\hash\relax
\let\hash\savehash
\definecolor{pinegreen}{rgb}{0.0, 0.47, 0.44}

\usepackage[utf8]{inputenc}
\usepackage{textcomp}
\usepackage[normalem]{ulem}

\usepackage{graphicx}
\usepackage{makecell}
\usepackage{siunitx}

\newtheorem{thm}{\textbf{Theorem}}
\newtheorem{col}{\textbf{Corollary}}
\newtheorem{asmp}{\textbf{Assumption}}
\newtheorem{lem}{\textbf{Lemma}}

\newtheorem{prop}{\textbf{Proposition}}

\usepackage{afterpage}
\usepackage{booktabs}
\usepackage{algorithm}
\usepackage{algorithmicx}
\usepackage{booktabs}        
\usepackage{adjustbox}       
\usepackage{threeparttable}
\usepackage{hyperref}
\usepackage{xr-hyper}
\usepackage{xr}

\usepackage{comment}

\usepackage{ulem}
\usepackage{mathrsfs}

\usepackage{amsmath,amssymb}
\usepackage{algorithm}
\usepackage[noend]{algpseudocode}   
\usepackage{arydshln}           
\usepackage{booktabs}

\newcommand{\bel}{\begin{eqnarray}\label}
\newcommand{\eel}{\end{eqnarray}}
\newcommand{\bes}{\begin{eqnarray*}}
\newcommand{\ees}{\end{eqnarray*}}
\newcommand{\bei}{\begin{itemize}}
\newcommand{\beiftnt}{\begin{itemize}\footnotesize}
\newcommand{\eei}{\end{itemize}}

\def\benu{\begin{enumerate}}
\def\eenu{\end{enumerate}}

\def\E{{\mathbb{E}}}
\def\P{{\mathbb{P}}}

\def\complex{\mathop{{\rm I}\kern-.58em\hbox{\rm C}}\nolimits}

\def\diag{\hbox{\rm diag}}

\def\mathbold{\boldsymbol} 

\def\ba{\mathbold{a}}

\def\bA{\mathbold{A}}

\def\bb{\mathbold{b}}

\def\bD{\mathbold{D}}

\def\bE{\mathbold{E}}

\def\bi{\mathbold{i}}

\def\bM{\mathbold{M}}

\def\bO{\mathbold{O}}

\def\bR{\mathbold{R}}

\def\shat{\widehat{s}}

\def\bT{\mathbold{T}}

\def\bu{\mathbold{u}}

\def\bv{\mathbold{v}}

\def\bV{\mathbold{V}}

\def\bw{\mathbold{w}}

\def\bW{\mathbold{W}}\def\hbW{{\widehat{\bW}}}

\def\bx{\mathbold{x}}

\def\bX{\mathbold{X}}

\def\by{\mathbold{y}}

\def\bz{\mathbold{z}}

\def\eps{\epsilon}

\def\lam{\lambda}

\def\hlambda{\widehat{\lambda}}

\def\bSigma{\mathbold{\Sigma}}

\def\mS{\mathcal{\bf{S}}}

\def\mS{\mathcal{S}}

\def\E{\mathbb{E}}

\def\T{\top}

\def\mS{\mathcal{S}}

\let\hat\widehat

\newcommand{\RR}{\mathbb{R}}
\newcommand{\EE}{\mathbb{E}}

\begin{document}
\maketitle

\begin{abstract}
This paper tackles the problem of feature selection in a highly challenging setting: $\EE (y|\bx) = G(\bx_{\mS_0})$, where $\mS_0$ is the set of relevant features and $G$ is an unknown, potentially nonlinear function subject to mild smoothness conditions. Our approach begins with feature selection in deep neural networks, then generalizes the results to  H{\"o}lder smooth functions by exploiting the strong approximation capabilities of neural networks.
Unlike conventional optimization-based deep learning methods, we reformulate neural networks as index models and estimate $\mS_0$  using the second-order Stein’s formula. This gradient-descent-free strategy guarantees feature selection consistency with a sample size requirement of $n = \Omega(p^2)$, where $p$ is the feature dimension. To handle high-dimensional scenarios, we further introduce a screening-and-selection mechanism that achieves nonlinear selection consistency when $n = \Omega(s \log p)$, with $s$ representing the sparsity level. Additionally, we refit a neural network on the selected features for prediction and establish performance guarantees under a relaxed sparsity assumption. Extensive simulations and real-data analyses demonstrate the strong performance of our method even in the presence of complex feature interactions.
\end{abstract}

\section{Introduction}
Feature selection is a classic statistical problem that seeks to identify a subset of features that are most relevant to the outcome. In this paper, we study feature selection in a highly general setting.
Let $y$ denote the response variable and let $\bx$ be a p-dimensional feature vector. We consider the model
\bel{model0}
\EE (y|\bx) = G(\bx_{\mS_0}), 
\eel
where $\mS_0$ represents the set of relevant features, and $G(\cdot)$ denotes an unknown, potentially nonlinear function that satisfies mild smoothness conditions, such as H{\"o}lder smoothness or Lipschitz continuity. By imposing specific structural assumptions on $G$, model (\ref{model0}) includes many classical frameworks as special cases, such as generalized linear models, additive models, etc. 
We focus on the most general setting without such constraints and aim to offer an efficient, theoretically justified method to accurately recover  $\mS_0$ without prior knowledge of the functional form of $G$. 

Feature selection in linear models has been thoroughly studied in the statistics community, resulting in a substantial body of literature that is too extensive to summarize. Lasso \citep{tibshirani2005sparsity} is arguably the most popular feature selection method due to its favorable computational efficiency and statistical properties. Numerous alternative regularization approaches have also been proposed; a very incomplete list includes Adaptive Lasso \citep{Zou06}, SCAD \citep{FanL01}, MCP \citep{Zhang10}, and more.  While Lasso relies on the strong irrepresentable condition to guarantee selection consistency \citep{ZhaoY06}, nonconvex penalties such as SCAD and MCP can achieve consistency under the weaker Restricted Eigenvalue condition \citep{feng2019sorted}.
Beyond linear models, feature selection has been studied on other settings, such as generalized linear models \citep{vandegeer2008,bunea2008}, additive models \citep{additive_witten, koltchinskii2010sparsity}, nonparametric regression \citep{benkeser2016highly}, etc. 
Notably, \citet{benkeser2016highly} 
introduced the Highly Adaptive Lasso for nonparametric regression, achieving a local prediction convergence rate faster than $n^{-1/4}$, independent of the feature dimension.

Despite the extensive literature on specific models, to our knowledge, no existing approach offers a theoretically guaranteed method for feature selection in general nonlinear models.  We aim to fill this gap by leveraging the strong approximation power of Deep Neural Networks (DNNs).
As a driving force of statistics and machine learning, DNNs have significantly advanced over the past decade and have 
been broadly applied across various domains, such as computer vision, natural language processing, and many others. 
By incorporating many layers of non-linear functions with thousands or even millions of unknown parameters, DNN can achieve remarkable approximation capabilities and outstanding predictive accuracy.

A major limitation of deep neural networks is their lack of interpretability. With a large number of hidden parameters embedded in a “black box,” DNNs are typically difficult to understand, making feature selection — one of the most direct and critical steps for model interpretation — particularly challenging. As a result, DNNs can be problematic in settings that require transparent and well-understood models, such as biomedical and genetic studies.


Despite the extensive literature on high-dimensional linear models, practical methods for feature selection in general nonlinear models and neural networks remain limited.
Notable related work includes \citet{chen2021nonlinear}, who proposed Deep Feature Selection (DFS) for identifying relevant features using DNNs, and \citet{lemhadri2021lassonet}, who developed LassoNet to induce feature sparsity via a residual network layer. 
More recently, neural network–based feature selection has also been studied in settings such as survival analysis \citep{meixide2024neural}. However, all these approaches rely on gradient-descent-type algorithms to solve highly nonconvex optimization problems, making it difficult to obtain strong theoretical guarantees.



Unlike the above optimization-based approaches, we adopt a nonparametric statistical perspective to address problem (\ref{model0}) and establish feature selection consistency without requiring knowledge of $G$. We begin by studying feature selection in deep neural networks and then show that our method achieves selection consistency for general H{\"o}lder smooth functions, leveraging the strong approximation capabilities of DNNs. To enable feature selection, we recast neural networks as index models and estimate $\mS_0$ in (\ref{model0}) using the second-order Stein’s formula \citep{stein1981estimation}. With the selected features, downstream tasks, including prediction, can subsequently be performed with theoretical guarantees. Moreover, we introduce a screening-and-selection mechanism that enables feature selection in high-dimensional settings.


Stein's formula is frequently used in index models. For example, in a single index model $\E (y|\bx) = f(\bb^\T\bx)$ where $\bx$ is a standard Gaussian random vector, the first-order Stein's formula suggests that the index vector $\bb$ is proportional to the expectation of $\bx \cdot y$, independent of the link function $f$. Stein's method and its generalized versions have been extended in various studies. For instance, \citet{plan2016generalized} studied single index models with sparse index vectors, while \citet{yang2017high, goldstein2018structured} expanded the Gaussian input assumption to include heavy-tailed or non-Gaussian scenarios. \citet{fan2023understanding} further investigated implicit regularization in single index models. Additionally, Stein's formula has been extended to multiple index models, as in \citet{yang2017estimating}, who studied non-Gaussian multiple index models via a second-order Stein’s method. Beyond index models, Stein’s method has been employed in other contexts, such as varying coefficient index models \citep{na2019high}.

Based on the Stein's formula, we introduce the first nonparametric statistical framework for neural network feature selection, offering three key advantages over existing optimization-based methods. First, our approach for deep ReLU networks is architecture-agnostic: it does not depend on the number of layers or neurons. Tuning neural network architectures is complex and lacks strong theoretical guidance; inappropriate choices can substantially degrade their performance. Second, it sidesteps the computational burden and instability of nonconvex optimization, a particular challenge in deep networks with thousands of parameters.
Third, and most importantly, by avoiding computations with gradient descent, our approach guarantees feature selection consistency for arbitrary H{\"o}lder smooth functions when $n=\Omega(s \log p)$ under Gaussian design. 
Comprehensive simulations and real genetic data analyses further demonstrated the superior performance of our approach. 

The rest of the paper is organized as follows. Section \ref{sec2} formulates the feature selection problem in deep neural networks and introduces an estimation procedure based on Stein’s formula. Sections \ref{sec3} establishes feature selection consistency for DNNs and for arbitrary H{\"o}lder smooth functions.  Section \ref{sec_high} introduces a screening mechanism for high-dimensional settings. 
Section \ref{sec_prediction} provides theoretical guarantees for prediction
accuracy using the selected features. 
Section \ref{sec_unknown} covers practical details, including tuning unknown parameters and score estimation. Section \ref{simulation} presents comprehensive simulations evaluating feature selection and prediction performance. Real-world genetic data analyses are conducted in Section \ref{sec9}, demonstrating the practical applicability of our approach. Finally, Section \ref{sec10} concludes with a discussion of key findings and implications. 




\noindent{\bf Notations:}
We use bold uppercase letters $\mathbf{A}$, $\mathbf{B}$ to denote matrices, bold lowercase 
letters $\mathbf{a}$, $\mathbf{b}$ to denote vectors. For a vector $\bv$, $\| \bv \|_{q} = (\sum_{j} | v_j |^q)^{1/q}$ is the $\ell_q$ norm. For a matrix $\bM$,  $\| \bM \|_2 = \max_{\bx \neq 0} \|\bM\bx\|_2 / {\|\bx\|_2} $ is the spectrum norm, $\|\bM \|_F = \sqrt{\sum_{j,k} \bM_{j,k}^{2} }$ is the Frobenius norm, $\|\bM \|_{\infty} = \max_j \sum_k |\bM_{j,k} | $ is the infinite norm representing the maximum absolute row sum.  For two sequences $\{x_n\}$ and $\{y_n\}$, we denote $x_n = \mathcal{O}(y_n)$ if $|x_n|\le C_1 |y_n|$ for some absolute constant $C_1$ and denote $x_n = \Omega(y_n)$ if $|x_n|\ge C_2 |y_n|$ for some absolute constant $C_2$.

\section{Feature selection via second-order Stein's formula}\label{sec2}

\subsection{Problem setup}
We study a supervised learning problem with a $p$-dimensional predictor vector $\bx$
 and a continuous scalar response $y$.
Under a sparsity assumption, only a subset of the indices $\{1,\ldots, p\}$ are relevant to the outcome $y$. Let this subset be $\mS_0$ and assume that 
$\text{Card}(\mS_0)=s\ll p$.
Our goal is to accurately recover $\mS_0$ under general nonlinear model settings. Formally, we consider the model
\begin{align}\label{model1}
    y = G(\bx_{\mS_0}) +\eps, \ \ \ G\in \mathscr{G},
\end{align}
where $\eps$ is an additive noise, $G(\cdot)$ is an arbitrary unknown function that belongs to the space of H{\"o}lder smooth functions $\mathscr{G}$, 
    \begin{align}\label{G}
   \mathscr{G}=\Big\{&G\in \mathbb{R}^s\rightarrow \mathbb{R}: \ \vert G(\bz)-G(\bz') \vert \le C\|\bz-\bz'\|_2^\beta,\ C>0,\ 0<\beta\le 1, \  \forall \ \bz, \bz'\in\mathbb{R}^s\Big\}.
    \end{align}
Clearly, $\mathscr{G}$  is a broad function class that includes all H{\"o}lder smooth functions of $s$ features. It also contains Lipschitz-continuous and continuously differentiable functions.
For simplicity, we focus on continuous outcomes with additive noise, though the framework can be extended to discrete responses with other noise structures.

 To identify $\mS_0$ under general nonlinear settings, we start by considering feature selection in deep feed-forward neural networks (FNN, \citealt{lecun2015deep}). Leveraging the strong approximation capacity of deep neural networks, we then show that the proposed approach can achieve feature selection consistency for any unknown H{\"o}lder smooth function.

Deep feed-forward neural networks, in its simplest form, can be written into the following function class
\bel{dnn}
 \mathscr{G}^{NN}  =    \left\{g\in \mathbb{R}^p\rightarrow \mathbb{R}: g(\bx)=\bW_L\sigma_L(\bW_{L-1}\cdots\sigma_2(\bW_2\sigma_1(\bW_1\bx)))\Big| \bW=(\bW_1,\ldots, \bW_L)\right\}.
\eel
Here, $NN$ stands for neural networks, $L$ is the depth of the network, $\bW=(\bW_1,\ldots, \bW_L)$ are the target unknown weight matrices
with $\bW_1\in\mathbb{R}^{k_1\times p}$, $\bW_L\in\mathbb{R}^{k_L\times 1}$, and $\bW_l\in\mathbb{R}^{k_{l}\times k_{l-1}}$ for $l=2,\ldots, L-1$. 
Each layer applies a known componentwise nonlinear activation $\sigma_l(\cdot)$; common choices include the sigmoid and the rectified linear unit (ReLU). Although $\sigma_l(\cdot)$ may vary across layers, it is often convenient to use a single activation throughout, i.e., $\sigma_1=\sigma_2=\ldots=\sigma_L$. In this paper, we focus on networks where all activations are ReLU.

To conduct feature selection in deep neural networks, we consider the model $\EE (y|\bx)= g(\bx)$ with $g \in \mathscr{G}^{NN}$ belong to the neural network function class (\ref{dnn}). 
In the case where $y$ depends only on the subset of features indexed by $\mS_0$, we have that there exists a $g \in \mathscr{G}^{NN}$ satisfying
\bel{s0}
\big\{j=1,\ldots, p,  \ \ \|\{\bW_1\}_{\cdot j}\|_2\neq 0\big\}=\mS_0.
\eel
That is to say, there exists a neural network where the first-layer weight matrix $\bW_1$ is column-wise sparse and the nonzero columns in $\bW_1$ are indexed by $\mS_0$. While we adopt the $\ell_2$-norm in (\ref{s0}), other norms could also be used equivalently.
Let $\mathscr{G}_{\mS_0}^{NN}$ denote the function class of neural networks with $\mS_0$-indexed column-sparse matrix $\bW_1$. Mathematically, 
\begin{align}\label{dnn2}
\mathscr{G}_{\mS_0}^{NN}=\Big\{&g\in \mathbb{R}^p\rightarrow \mathbb{R}: \exists \ L\in \{1,2,\cdots\}, \  \bW=(\bW_1,\ldots, \bW_L), \ \big\{j: \|\{\bW_1\}_{\cdot j}\|_2\neq 0 \big\}=\mS_0,
\cr & g(\bx)=\bW_L\sigma_L(\bW_{L-1}\cdots\sigma_2(\bW_2\sigma_1(\bW_1\bx))), \ \forall \bx\in\mathbb{R}^p\Big\}.
\end{align}
Thus, for a function $g\in \mathscr{G}_{\mS_0}^{NN}$, identifying the non-zero columns in $\bW_1$ is equivalent to finding $\mS_0$. In other words, we aim to detect $\mS_0$ in the model
\bel{model2}
y = g(\bx) +\eps, \ \ \ g\in \mathscr{G}_{\mS_0}^{NN}.
\eel

Before introducing our approach for identifying $\mS_0$ in the neural networks model (\ref{model2}), we shall mention that this approach can be extended to general nonlinear functions in (\ref{model1}) by exploiting the strong approximation power of DNN.
Indeed, with sufficient depth and width, FNNs can approximate any H{\"o}lder smooth functions to arbitrary accuracy \citep{yarotsky2017error,shen2019deep}. Formally, for any $G \in \mathscr{G}$ and any $\varepsilon_0 > 0$, there exists $g \in \mathscr{G}_{\mS_0}^{\text{NN}}$ such that
\bes
\sup_{\bx} \left|G(\bx_{\mS_0})-g(\bx)\right|\le \eps_0.
\ees
This approximation guarantee underpins our feature selection 
consistency result for general H{\"o}lder smooth function $G$. A rigorous analysis will be deferred to Section \ref{sec3}.  

\subsection{Subset estimation via second-order Stein's formula}

We now present our approach for identifying $\mS_0$ under model (\ref{model2}).
We begin by reformulating model (\ref{model2}) into a multiple index model and focusing on the column-sparse first-layer weight matrix $\bW_1$. Define 
\bel{link}
f(\bz)=\bW_L\sigma_L(\bW_{L-1}\cdots\sigma_2(\bW_2\sigma_1(\bz))),
\eel
so that $f$ serves as a link function summarizing the contribution of layers 2 through $L$
.
Then, model (\ref{model2})  can be written as
\bel{model3}
\EE(y) = f(\bW_1\bx), \ \ \ \big\{ \|\{\bW_1\}_{\cdot j}\|_2\neq 0\big\}=\mS_0.
\eel
 Our goal is to identify the set $\mS_0$ with an unknown $f$.

We shall emphasize that $\bW_1$ is not identifiable in model (\ref{model3}) when $f$ is unknown. For example, if $\bW_1$ is a solution to (\ref{model3}), $\bW_1'=\bO\bW_1$ is also a solution for an invertible matrix $\bO\in\mathbb{R}^{k_1\times k_1}$ since $\bO$ can be absorbed into $f$. 
However, the set $\mS_0$ remains identifiable because it is invariant under multiplication by an invertible matrix $\bO$, i.e., $\mS_0'=\big\{j=1,\ldots, p,  \ \ \|\{\bW'_1\}_{\cdot j}\|_2\neq 0\big\}=\mS_0$. Without loss of generality, we further assume that $k_1\le s$. Otherwise, there exists a integer $k_1^{(0)}\le s$, a matrix $\bR$ of dimension $k_1\times k_1^{(0)}$, and a row-wise sparse matrix $\bW_1^{(0)}$ of dimension $k_1^{(0)}\times p$ with the same support $\mS_0$ such that $\bW_1=\bR\bW_1^{(0)}$. In this case, the matrix $\bW_1^{(0)}$ can be treated as the new target.


Our strategy for identifying $\mS_0$ is based on the second-order Stein's lemma. To get started, we first define the score function associated with the input $\bx$. 
For any random vector $\bx\in\RR^{p}$  with density $P:\RR^{p}\rightarrow\RR$. The score function $S(\bX):\RR^{p}\rightarrow\RR^{p}$ associated with $\bx$ is defined as
\begin{align*}
S(\bx)=-\nabla_{\bx} [\log P(\bx)]=-\nabla_{\bx} P(\bx)/P(\bx).
\end{align*}
Furthermore, the second-order score function $T(\bx):\RR^{p}\rightarrow\RR^{p\times p}$ is defined as
\bes
T(\bx)=\nabla^2_{\bx} P(\bx)/P(\bx)= S(\bx)S(\bx)^\T- \nabla_{\bx} S(\bx).
\ees
Under the case with  Gaussian input $\bx\in \mathcal{N}(\boldsymbol{0},\bSigma)$, the first and second order score reduce to
$S(\bx)=\bSigma^{-1}\bx$ and  $T(\bx)=\bSigma^{-1}\bx\bx^\T\bSigma^{-1}-\bSigma^{-1}$, respectively.
 If we further let $\bz=\bW_1\bx$ and assume that both $\EE[yT(\bx)] $ and $\EE\left[\nabla^2_{\bz} f (\bW_1\bx)\right]$ are well-defined, then a second-order Stein's formula suggests
\begin{align}\label{eq:main0}
   \EE[y T(\bx)] = \bW_1^\T \cdot \EE\left[\nabla^2_{\bz} f (\bW_1\bx)\right] \cdot \bW_1.
  \end{align}
 

The equation (\ref{eq:main0}) serves as the basis for estimating $\mS_0$. Let $\bA = \mathbb{E}[yT(\bx)]$ and write its eigenvalue decomposition as $\bA = \bW^\T \bD \bW$. The second-order Stein's formula suggests that there exists an invertible matrix $\bR \in \mathbb{R}^{k_1 \times k_1}$ such that $\bW_1 = \bR \bW$. In other words,
$\bW_1$ has the same row space as $\bW$, regardless of the form of link function $f$. Consequently, ${\mS}_0$ can be obtained from the eigenvalue decomposition of $\bA$.


Given $n$ observations $(\bx_i, y_i)$, $i=1,\ldots, n$,  the sample version $(1/n)\sum_{i=1}^{n}y_i T(\bx_i)$ is an natural estimate of $\EE[y T(\bx)]$. As such, we propose the following eigenvalue-decomposition-based subset estimation:
\bel{s0hat}
&&\hat{\mS}_0=\{j=1,\ldots, p, \ \ \|\{\hbW_1\}_{\cdot j}\|_2\ge \kappa \},\cr 
&&\hbW_1=\mathrm{Eigen}_{k_1}\left(\frac{1}{n}\sum_{i=1}^{n}y_iT(\bx_i)\right), 
\eel
where $\text{Eigen}_k(\bM)$ refers to the first $k$ leading eigenvectors of $\bM$ and $\kappa>0$ is a certain threshold. Note that the leading eigenvectors are those corresponding to the eigenvalues with the largest absolute magnitudes. We summarize our approach in Algorithm \ref{alg1} below. 

\begin{algorithm}[h!]
\caption{Feature Selection via Second-order Stein's Formula\label{alg1}}
\begin{algorithmic}[1]
\Require Dataset $\bx_i \in \mathbb{R}^{p}$ and $y_i \in \mathbb{R}$ for $i=1,\ldots, n$, threshold level $\kappa$, first layer dimension $k_1$, second order score function $T(\cdot)$
\State Calculate $\hat{\bA} \gets (1/n)\sum_{i=1}^n y_i T(\bx_i)$
\State Perform rank-$k_1$ eigenvalue decomposition $\hbW_1 \leftarrow \text{Eigen}_{k_1}(\hat{\bA})$
\State \Return $\hat{\mS}_0 = \left\{j=[p]: \Vert \{\hbW_1\}_{\cdot j}\Vert_2 \ge \kappa \right\}$
\end{algorithmic}
\end{algorithm}

We emphasize that the subset-selection step in Algorithm \ref{alg1}
 is not limited to the neural network model (\ref{model2}). Since $\hbW_1$ can be estimated whenever a second-order score function is available, Algorithm \ref{alg1}
 also applies to the general model (\ref{model1}). In Section \ref{sec3}, we establish that Algorithm \ref{alg1}
 achieves feature selection consistency with high probability for both neural networks and H{\"o}lder smooth functions classes.


\subsection{Comparisons with LassoNet and DFS}
The ``sparsity in the first-layer'' assumption has been considered in the literature on feature selection for neural networks. Notable works include Deep Feature Selection (DFS, \citealt{chen2021nonlinear}) and LassoNet
\citep{lemhadri2021lassonet}. 

Specifically, LassoNet achieves feature sparsity by incorporating a skip (or residual) layer into the network, ensuring that a feature can participate in any hidden layer only if its representation in the skip layer is active. Similarly, DFS is a deep neural network algorithm designed for sparse-constrained nonconvex optimization problems, where feature sparse parameters are in the selection layer and other parameters are in the approximation layer. Theoretical properties of DFS have been developed under a Generalized Stable Restricted Hessian (GSRH) condition.

Unlike LassoNet and DFS, our proposal is a nonparametric statistical method with at least three key advantages.
First, our estimator of $\mS_0$ is independent of the neural network architecture, such as the number of layers or neurons per layer. Such structural information can be absorbed into the link function $f$, and its specific form does not affect the estimation of $\mS_0$. This is important because tuning network architectures is challenging, with limited theoretical guidance, and poorly chosen architectures can severely degrade performance.
Second, our method avoids solving highly nonconvex optimization problems via gradient descent. For deep networks with thousands or millions of parameters, gradient-based algorithms are not only computationally expensive but also prone to converging to unstable local optima, whose properties are difficult, if not impossible, to guarantee.
Third, and most importantly, by avoiding gradient-based training, our approach yields feature selection consistency for general H{\"o}lder smooth functions under fairly mild conditions. See Sections \ref{sec3_1} and \ref{sec3_2} for more details.

\section{Feature selection consistency}\label{sec3}

In this section, we show that our second-order Stein approach guarantees feature selection consistency. We first analyze the deep neural network model (\ref{model2}) in Section \ref{sec3_1}, and then extend the results to the general nonlinear model (\ref{model1}) in Section \ref{sec3_2}.

\subsection{Feature selection consistency in DNNs}\label{sec3_1}

Given $n$ observations $(\bx_i, y_i)$, $i=1,\ldots, n$, assume that $(\bx_i, y_i)$ follows the DNN model (\ref{model2}) with $g\in \mathscr{G}_{\mS_0}^{NN}$. We start by stating the assumptions for feature selection consistency.

\begin{asmp}\label{a2new}
Let $\bz=\bW_1\bx$ and $f(\bz)=\bW_L\sigma_L(\bW_{L-1}\cdots\sigma_2(\bW_2\sigma_1(\bz)))$. Assume that the matrix $\EE\left[\nabla^2_{\bz} f (\bW_1\bx)\right]$ is non-singular.
\end{asmp} 

A direct consequence of Assumption \ref{a2new} is that the rank of $\EE\left[\nabla^2_{\bz} f (\bW_1\bx)\right]$ does not exceed the sparsity in the first layer, which further implies $k_1\le s$. For a two-layer ReLU network, Assumption~\ref{a2new}
 is readily satisfied, as shown in the following proposition.

\begin{prop}\label{prop:hessian-nonsingular}
Consider a two-layer ReLU network  $f(\bW_1 \bx) = \bw_2 \sigma(\bW_1 \bx)$ with $\bx \in \mathbb{R}^p, \bW_1 \in \mathbb{R}^{k_1 \times p}, \bw_2 \in \mathbb{R}^{1\times k_1} $ and suppose $\bx \sim \mathcal{N}(\boldsymbol{0}, \boldsymbol{\Sigma})$. Then $\EE\left[\nabla^2_{\bz} f (\bW_1\bx)\right]$ is non-singular if $\bW_1$ has full row rank and $\bw_2$ has no zero entries.

\end{prop} 

While feedforward networks can be viewed as multi-index models, they differ fundamentally from Assumption \ref{a2new}: in MIMs the link function is unknown, making the assumption hard to guarantee, whereas in DNNs the link is known — the nonlinear activation composed with the weights from the second through the last layer — so the checks above can be carried out explicitly. For deeper networks ($L \ge 3$), deriving explicit nonsingularity conditions is more difficult,  and further exploration of this issue is an important direction for future research. 

\begin{asmp}\label{a3}
Let $\mu= (s/k_1)\mathop{\max}_{j\in\mS_0}\| \left\{\bW_1\right\}_{\cdot j}\|_2^2$ and $\gamma = (s/k_1) \mathop{\min}_{j\in \mS_0} \|  \left\{\bW_1\right\}_{\cdot j}\|_2^2$. 
Assume there exists constants $\underbar{c}$ and $\bar{c}$ that $\underbar{c}\le \gamma\le \mu \le \bar{c}$.
\end{asmp}

Assumption \ref{a3} imposes the incoherence condition on the first layer weight matrix $\bW_1$. It requires the coherence of matrix $\bW_1$ is bounded by constants, further suggesting that the information in $\bW_1$ is not concentrated in the top few rows. We note that similar conditions have been considered in the literature, such as \citet{fan2017ellinftyeigenvectorperturbationbound}.

\begin{asmp}\label{a3-}
    Assume that each component of the second order score, $T_{jk}(\bx)$, $1\le j,k \le p$, is sub-exponential. Further, assume that there exists a constant $M$ such that
    \begin{align*}
        \mathop{\sup}_{\left\| \bu \right\|_2 = 1} \left\| \bu^T T(\bx) \bu\right\|_{\psi_1} \le M.
    \end{align*}
\end{asmp}

 

For a random feature $z$, if it is sub-exponential, its sub-exponential $\psi_1$ norm is defined as $\| z \|_{\psi_1} = \mathop{\sup}_{k\ge 1} k^{-1}\left(\EE \left[|z|^k\right]\right)^{1/k}$. If it is sub-gaussian, its sub-gaussian 
 $\psi_2$ norm is defined as $\|z\|_{\psi_2} = \mathop{\sup}_{k\ge 1} k^{-1/2}\left(\EE \left[|z|^{k}\right]\right)^{1/k}$. For any sub-gaussian random vector $\bx$, its sub-gaussian norm is defined as  $\| \bx\|_{\psi_2} = \mathop{\sup}_{\|\bu\|_2=1} \|\bu^T\bx\|_{\psi_2}$.
 Assumption \ref{a3-}
 requires that each component of the score is sub-exponential, with a $\psi_1$ norm uniformly bounded in $p$. This condition is satisfied by a broad class of distributions for $\bx$. For example, if $\bx$ is Gaussian, the second-order score meets Assumption \ref{a3-}, as formalized in the next proposition. Similarly, if $\bx$ follows a Student’s t distribution, 
 the second-order score $T(\cdot)$ is bounded and hence sub-exponential, so Assumption \ref{a3-} also holds.


\begin{prop}\label{prop1-}
   When $\bx \sim \mathcal{N}(\boldsymbol{0}, \bSigma)$ with $\phi_{\min}(\bSigma)>0$, the score $T(\bx)$ satisfies 
    \begin{align*}
        \mathop{\sup}_{\| \bu\|_2=1} \left\| \bu^T T(\bx) \bu\right\|_{\psi_1} \le 4\phi^{-1}_{\min}(\bSigma).
    \end{align*}
\end{prop}

Now we are ready to present our main theorem on feature selection in deep neural networks.

\begin{thm}\label{thm1}
    Consider the model (\ref{model2}) with Gaussian noise $\eps$. Assume Assumptions \ref{a2new} - \ref{a3-}. Further assume the threshold level $\kappa \le c \min_{j \in \mS_0} \|\{\bW_1\}_{\cdot j}\|_2$ for certain $c>0$. 
    Suppose that $\bx$ is a sub-gaussian vector.  Then, for any $\nu>0$, when
  $n \ge C\left(\log^2 \nu + p^2\right)$ for certain positive constant $C$, Algorithm \ref{alg1} achieves feature selection consistency with probability at least $1-\nu$, i.e., 
\bes
\P(\hat{\mS}_0 =\mS_0) \ge 1-\nu.
\ees
\end{thm}

 Theorem \ref{thm1} established the feature selection consistency of our approach for deep neural networks. When the sample size $n=\Omega (p^2)$, the true support can be recovered with high probability for all functions $g$ in the neural network class $\mathscr{G}_{\mS_0}^{NN}$.
 Crucially, this guarantee does not rely on correctly specifying the network architecture. In contrast to DFS, gradient-based training of neural networks involves solving a highly nonconvex problem, may converge to unstable local optima, and typically depends on a well-specified architecture to achieve desirable properties. 
 
 

 Theorem \ref{thm1} assumes sub-Gaussian covariates $\bx$, ensuring that the empirical moment $(1/n)\sum_{i=1}^n y_i T(\bx_i)$ concentrates around its expectation $\mathbb{E}[y T(\bx)]$, thereby enabling control over the estimated eigenvectors $\widehat{\bW}_1$. 
This requirement can be further relaxed. For heavy-tailed non-sub-gaussian $\bx$, we may apply certain truncations, such as that in \citet{minsker2018sub} or \citet{yang2017high}, on $(1/n) \sum_{i=1}^n y_i T(\bx_i)$. Then, similar results on feature selection consistency can still be guaranteed.



\subsection{Feature selection consistency in H{\"o}lder smooth functions}\label{sec3_2}
We now turn to the general model (\ref{model1}) and prove feature selection consistency for a function $G \in \mathscr{G}$, where $\mathscr{G}$ denotes the class of H{\"o}lder smooth functions. To establish consistency in nonlinear settings, we first rigorously quantify the approximation gap between $\mathscr{G}$ and the neural network class $\mathscr{G}_{\mS_0}^{NN}$, following \citet{shen2019deep}.


\begin{lem}\label{lm2}
Let $G$ be a $\beta$-H{\"o}lder smooth function in $\mathscr{G}$. For any $m>0$ and $w, d \in \mathbb{N}^+$, there exists a function $g \in \mathscr{G}_{\mS_0}^{NN}$ implemented by ReLU neural network 
with width $3^{s+3}\max\left\{s\lfloor w^{1/s} \rfloor, w+1\right\}$ and depth $12d +14 + 2s$ such that
    \bel{lm2-1}
\sup_{\bx_{\mS_0}\in [-m,m]^s} \left|G(\bx_{\mS_0})-g(\bx)\right|\le Cm^\beta w^{-2\beta/s}d^{-2\beta/s},
\eel
    where $C$ is a certain positive constant. In particular, let $m = \mathcal{O}(n^{1/s})$ and $w = d = \lceil n^{\frac{2\beta+s}{8\beta} }\rceil$. 
   Then the difference between $G(\bx_{\mS_0})$ and $g(\bx)$ reduces to
\bel{lm2-2}
\sup_{\bx_{\mS_0}\in [-m,m]^s} \left|G(\bx_{\mS_0})-g(\bx)\right|
\le C n^{-1/2}.
\eel
\end{lem}

Lemma \ref{lm2} established the $\ell_\infty$ difference between functions in $\mathscr{G}$ and $\mathscr{G}_{\mS_0}^{NN}$ for $\bx_{\mS_0}\in [-m,m]^s$.
Here, $w$ and $d$ denote the network width and depth, respectively. For sufficiently large $w$ and $d$, a ReLU network can approximate any H{\"o}lder-smooth function arbitrarily well. 
Moreover, Lemma \ref{lm2}
 bounds the approximation error between $G$ and $g$ only on the compact domain $[-m,m]^s$. But when the design is sub-Gaussian, this local control is sufficient to ensure feature selection consistency under model (\ref{model1}), as established in Theorem \ref{thm2} below. On the other hand, because Assumptions \ref{a2new} and \ref{a3} are imposed on the weight matrices of a neural network, we assume that, for a H{\"o}lder-smooth function $G$, there exists a ReLU network $g \in \mathscr{G}_{\mS_0}^{NN}$ that approximates $G$ and satisfies these assumptions.



\begin{thm}\label{thm2}
    Consider model (\ref{model1}) with $G\in\mathscr{G}$ with $\eps$ being Gaussian noise.  
    Let $g$ be a function in $\mathscr{G}_{\mS_0}^{NN}$ such that (\ref{lm2-2}) holds. 
 Assume Assumptions \ref{a2new} to \ref{a3-} hold. 
 Let $\bW=\text{Eigen}_{k_1}\big(\mathbb{E}g(\bx)T(\bx)\big)$ and assume $\kappa \le (1/2) \mathop{\min}_{j\in \mS_0} \|\{\bW\}_{\cdot j}\|_2$.
Suppose that $\bx$ is a sub-Gaussian vector. Then, for any $\nu>0$, when $n \ge C \left(\log^2 \nu+p^2\right)$ for certain positive constant $C$, Algorithm \ref{alg1} achieves feature selection consistency with probability at least $1-\nu$, i.e.,
    \begin{align*}
        \P(\hat{\mS}_0 =\mS_0) \ge 1-\nu.
    \end{align*}

\end{thm}

 Theorem \ref{thm2} establishes feature selection consistency for functions in $\mathscr{G}$. To our knowledge, this is the first result that delivers feature selection consistency for general H{\"o}lder-smooth functions together with computational guarantees. 
 When $G$ admits a DNN approximation as in (\ref{lm2-2}), the same sample size order $n =\Omega(p^2)$ as in Theorem \ref{thm1}
 suffices to ensure feature selection consistency for $G\in \mathscr{G}$. 
The proof of Theorem \ref{thm2} 
 hinges on Proposition \ref{dnnprop}.

\begin{prop}[Deep neural network approximation]\label{dnnprop}
Let $G$ be any $\beta$-H{\"o}lder smooth function in $\mathscr{G}$ with $0<\beta\le1$. Suppose that $\bx$ is sub-gaussian. Let $w = d = \lceil n^{\frac{2\beta+s}{8\beta} }\rceil$, there  exists a function $g\in \mathscr{G}_{\mS_0}^{NN}$ implemented by a ReLU FNN with width $3^{s+3}\max\left\{s\lceil w^{1/s} \rceil, w+1\right\}$ and depth $12d +14 + 2d$ such that:
\begin{align}
    \EE \left\|\left[\left\{G(\bx_{\mS_0})-g(\bx)\right\}T(\bx)\right] \right\|_2 \le C n^{-1/2},
\end{align}
where $C$ is a constant.
\end{prop}

Proposition \ref{dnnprop} provides an upper bound for the deep neural network approximation error $\mathbb{E} 
    \left[\left\{G(\bx_{\mS_0})-g(\bx)\right\}T(\bx)\right] $. This enables us to control the difference between $(1/n)\sum_{i=1}^n y_i T(\bx_i)$ and $\EE g(\bx)T(\bx)$. Consequently, Stein's lemma can be applied to identify the subset $\mS_0$. 


We conclude this section by comparing the proposed approach with prior work that applies Stein’s formula to index models. In particular, \citet{yang2017estimating} also explored a second-order Stein’s method. Despite this similarity, these two works differ substantially in objectives, methodology, and theoretical results. Their framework is aimed at estimating linear coefficients in multiple index models, whereas our focus is on feature selection for deep neural networks and general smooth functions.
This difference in objectives also leads to different sparsity structures: we assume column-wise group sparsity for feature selection, while \citet{yang2017estimating} imposes element-wise sparsity.
From a theoretical perspective, \citet{yang2017estimating} derives Frobenius-norm bounds for coefficient estimation. In contrast, we use max-norm bounds to obtain entrywise control of eigenvector perturbations, since Frobenius-norm bounds are insufficient for sharp feature selection consistency. Moreover, MIMs involve an unknown link function, making related assumptions difficult to verify. For DNNs, the link function is explicitly defined by nonlinear activations and the weight matrices from the second through the final layer, enabling direct verification of required assumptions, as shown in Proposition~\ref{prop:hessian-nonsingular}.
Finally, our approach avoids the computational burden of optimizing a high-dimensional $(p \times p)$ matrix through iterative gradient or Hessian updates, resulting in significantly greater computational efficiency.


\section{A screening-and-selection mechanism for high-dimension scheme} \label{sec_high}
In this section, we extend our framework to a high-dimensional setting. As established in Section \ref{sec3}, achieving feature selection consistency under our previous analysis requires a relatively stringent sample size of $n = \Omega(p^2)$, which is prohibitive in many modern applications such as genomics. To obtain consistent selection with more favorable scaling — potentially sublinear in $p$ — we follow the paradigm of \citet{fan2008sure} for sure independence screening and develop a nonlinear analogue based on the second-order score statistics.


The nonlinear screening procedure is based on a straightforward idea. For the neural network model (\ref{model2}), define $\bz=\bW_1\bx$ and $f(\bz)=\bW_L\sigma_L(\bW_{L-1}\cdots\sigma_2(\bW_2\sigma_1(\bz)))$, then, by the second-order Stein’s formula, we have
$\EE[y T(\bx)] = \bW_1^\T \cdot \EE\left[\nabla^2_{\bz} f (\bW_1\bx)\right] \cdot \bW_1$. Rather than performing eigenvalue decomposition directly as in Section \ref{sec2}, we first examine the diagonal entries of the empirical estimate of $\mathbb{E}[y T(\bx)]$. Since $\bW_1$ is column-wise sparse with support indexed by $\mS_0$, the diagonal entries are zero whenever $j \notin \mS_0$. Based on this observation, a statistically principled screening approach is to rank the coordinates by the absolute values of the diagonal entries of  $\hat{\bA}=(1/n)\sum_{i=1}^n y_i T(\bx_i)$ and to retain those corresponding to the largest values.
Let $\zeta\in (0,1)$ denote the screening proportion. Then the $ \lfloor \zeta p\rfloor$ retained coordinates are
\begin{equation}\label{s0zeta}
    \hat{\mS}_0^\zeta = \left\{ 1\le k\le p: |\hat{\bA}_{kk}| \text{ is among the first } \lfloor \zeta p\rfloor \text{ largest of all}\right\}.
\end{equation}
This screening mechanism can be applied repeatedly to the selected features, resulting in an iterative screening procedure. When combined with the earlier eigenvalue-decomposition approach, this leads to a screening-and-selection method tailored for high-dimensional settings, as detailed in Algorithm \ref{algo2}.

\begin{algorithm}[htbp!]
\caption{Screening-and-selection algorithm \label{algo2}}
\begin{algorithmic}[1]
\Require Dataset $\bx_i \in \mathbb{R}^{p}$ and $y_i \in \mathbb{R}$ for $i=1,\ldots, n$, threshold level $\kappa$, screening level $\zeta \in (0,1]$, target dimension $p_0$, first layer dimension $k_1$, second order score function $T(\cdot)$
\State Calculate $\hat{\bA} \gets (1/n)\sum_{i=1}^n y_i T(\bx_i)$
\State Initialize $\hat{\bA}_{\mathcal{I}} \gets \hat{\bA}$, $p_{\mathcal{I}} \gets p$
\While{$p_{\mathcal{I}} > p_0$}
\State Get screening subset: $\mathcal{I} \gets \left\{ 1\le k\le p_{\mathcal{I}}: |\hat{\bA}_{kk}| \text{ is among the first } \lfloor \zeta p_{\mathcal{I}} \rfloor \text{ largest of all}\right\}$
\State Update $\hat{\bA}_{\mathcal{I}} \gets (1/n)\sum_{i=1}^n y_i T(\bx_{i,\mathcal{I}})$, $p_{\mathcal{I}} \gets \lfloor\zeta p_{\mathcal{I}} \rfloor$
\EndWhile
\State Perform rank-$k_1$ eigenvalue decomposition $\hbW_1 \leftarrow \text{Eigen}_{k_1}(\hat{\bA_{\mathcal{I}}})$
\State \Return $\hat{\mS}_0 = \left\{j=[p]: \Vert \{\hbW_1\}_{\cdot j}\Vert_2 \ge \kappa \right\}$
\end{algorithmic}
\end{algorithm}

One point that deserves particular attention is when $p>n$, the empirical covariance  $\hat\bSigma=(1/n)\sum_{i=1}^n \bx_i\bx_i^\T$ is singular and thus cannot be directly used to compute the second-order score $T(\bx)$, which relies on the inverse covariance for Gaussian features. 
 In numerical studies, we mitigate this by employing the Ledoit–Wolf shrinkage estimator \citep{ledoit2004well}. Beyond Ledoit–Wolf, one can estimate the inverse covariance via methods such as the graphical lasso \citep{friedman2008sparse}, which is effective in sparse Gaussian settings.

Now we study the selection consistency of Algorithm \ref{algo2}, the following assumptions are needed.

\begin{asmp}\label{min2rd}
    Suppose that for some constant $c>0$,
    \bes
    \mathop{\min}_{j\in\mS_0} \EE \left|\frac{\partial^2 g(\bx)}{\partial x_j^2} \right|\ge c, \ \ g\in \mathscr{G}_{\mS_0}^{NN},
    \ees
    where $x_j$ is the $j$-th entry of $\bx$.
\end{asmp}

Assumption \ref{min2rd} enforces a minimal second-order signal strength for the active features in $\mS_0$. This is well-motivated for feature selection in DNN models, where $g$ is highly nonlinear; in high-dimensional settings, such minimum-curvature conditions ensure that influential features display sufficient nonlinearity in their marginal effects.  Conversely, in a linear regression setting, Assumption \ref{min2rd} 
 is not satisfied. In such cases, the SURE screening method \citep{fan2008sure} can be employed as the default option.


\begin{asmp}\label{submatrix}
Let $\mathcal{I}\subseteq \{1, 2, \dots, p\}$ be an index set with cardinality $q=\text{Card}\{\mathcal{I}\}$ satisfying $cn^{1/3}<q\le p$, where $c>1$ is a certain constant. For any $\mathcal{I}$, define $\bV_{\mathcal{I}} = \left(\diag\left\{T(\bx_{1, \mathcal{I}})\right\},\cdots,\diag\left\{T(\bx_{n, \mathcal{I}})\right\}\right)^T\in\mathbb{R}^{n\times q}$ and assume that
 \bes
    \P \left(\lambda_{\text{max}}\left( q^{-1}\bV_{\mathcal{I}} \bV_{\mathcal{I}} ^T\right) > c_1\right) \le e^{-c_2n},
    \ees
    \noindent 
where $c_1, c_2>0$ are constants.
\end{asmp}

We refer to Assumption \ref{submatrix}
 as the concentration property, with analogous conditions considered in the literature, e.g., \citet{fan2008sure}.
It is straightforward to verify that Assumption \ref{submatrix}
 is satisfied when $\bx$
 follows a Gaussian distribution. For instance, in the simplest case 
$\bx \sim \mathcal{N}(\boldsymbol{0}, \boldsymbol{I}_{p} )$, the law of large number suggests that $q^{-1}\bV_{\mathcal{I}} \bV_{\mathcal{I}}^T$  converges to $2 \boldsymbol{I}_{n}$ when $q\rightarrow \infty$, which immediately ensures the validity of Assumption \ref{submatrix}. Given Assumptions \ref{min2rd}-\ref{submatrix}, the following lemma establishes the performance guarantee of a single screening.

\begin{lem}\label{lemmascreening1step}
 Consider model  (\ref{model2}) with Gaussian noise $\eps$. Suppose Assumptions \ref{min2rd}-\ref{submatrix} hold. Let $\zeta\in(0,1)$  denote the screening proportion, and let $ \hat{\mS}_0^\zeta$ be as in (\ref{s0zeta}). If $\zeta = cn^{-1/3}$, 
   we have
    \bel{lm1-2}
    \P\left(\mS_0\subset \hat{\mS}_0^\zeta\right) \geq 1 -  s\exp\left\{-C n^{2/3}\right\},
    \eel
where $c$ and $C$ are certain constants. 
\end{lem}

Lemma \ref{lemmascreening1step} 
 shows that, in a single screening step, the relevant feature set $\mS_0$ is contained in the screened feature set 
 $\hat{\mS}_0^\zeta$ with high probability.
Building on this result, iterative application of the screening procedure can progressively reduce the size of the feature set. 
By incorporating Theorem \ref{thm1} on eigenvalue-decomposition-based selection, selection consistency of Algorithm \ref{algo2}
 can be established in high-dimensional settings, as formalized in the following theorem.

\begin{thm}\label{algo2consistency}
    Assume the conditions in Theorem \ref{thm1}. Further assume Assumptions \ref{min2rd} - \ref{submatrix} hold and $\bx$ is a sub-gaussian vector. Suppose screening level $\zeta  = c_1 n^{-1/3}$ and the target dimension $p_0 = c_2 n^{1/3}$ for some positive constant $c_1, c_2$. Then, for any $\nu>0$, when
  $n \ge C\left(\log^{2} \nu + s \log p \right)$ for certain positive constant $C$, Algorithm \ref{algo2} achieves feature selection consistency with probability at least $1-\nu$, i.e., 
\bes
\P(\hat{\mS}_0 =\mS_0) \ge 1-\nu.
\ees
\end{thm}

For deep neural network functions $g \in \mathscr{G}_{\mS_0}^{NN}$, 
 Theorem \ref{algo2consistency}
 establishes the feature selection consistency of Algorithm \ref{algo2}
 in the high-dimensional regime, under the sample size requirement $n = \Omega(s \log p)$. This result can be extended to general $\beta$-H{\"o}lder smooth functions in a manner analogous to Section \ref{sec3_2}
. The detailed derivations are omitted here due to space constraints.

\section{Theoretical guarantees of prediction accuracy under a two-step approach} \label{sec_prediction}
In Section \ref{sec3}, we established the consistency of feature selection in neural networks and extended these results to general H{\"o}lder smooth functions. 
In this section, we introduce a two-step procedure that first applies the proposed method to select relevant features, and then retrains a neural network to minimize the MSE. We also establish theoretical guarantees for the prediction performance of this two-step approach.

Rather than the strict sparsity assumption in Section \ref{sec3}, here we consider a relaxed setting in which the first‑layer weight matrix $\bW_1$ is allowed
 to be approximately sparse, as described in the following assumption.



\begin{asmp}\label{nsa1}
    Let $\eta_j = \|\{\bW_1\}_{\cdot j}\|_2, 1\le j \le p$. Assume that
    \bel{nsa1-1}
        \mathop{\min}_{\left|\mS\right|\le s}\sum\limits_{j \notin \mS} \eta_j \le \eps_0,
    \eel
    where $\eps_0$ is a small positive constant.
\end{asmp}


Assumption \ref{nsa1} allows that, beyond the $s$ most important features, other features can still have a small influence on the output. We measure each influence using the $\ell_2$-norm of the respective column of the first-layer weight matrix $\bW_1$, and require their combined effect to be at most $\eps_0$. Under Assumption
\ref{nsa1}, let $\hat{\mS}$ denote the set of features selected by either the Algorithm \ref{alg1}, or 
the screening-and-selection Algorithm \ref{algo2}, and let $\shat= \text{Card}(\hat{\mS})$ denotes its cardinality. We now turn to the following optimization problem: 
\bel{op1}
\hat{g}_{\hat{\mS}} = \mathop{\arg\min}_{g\in\mathcal{G}_{\hat{\mS}} } \frac1n \sum\limits_{i=1}^n \left(y_i - g\left(\bx_i\right)\right)^2,
\eel
where $\mathcal{G}_{\hat{\mS}}$ is the class of ReLU neural network functions defined on the selected features. Finally, Theorem \ref{nsthm1} 
establishes a prediction error bound using the selected feature set $\hat{\mS}$.

\begin{thm}\label{nsthm1}
Let $\hat{\mS}$ be generated using the Algorithm \ref{alg1}, and let $\hat{g}_{\hat{\mS}}$
 be computed according to
(\ref{op1}). Assume Assumptions \ref{a2new}-\ref{a3-} and Assumption \ref{nsa1}. Further assume that the threshold level $\eps_0< \kappa\le c \cdot \max_{|\mS | \leq s} \min_{j \in \mS} \|\{\bW_1\}_{\cdot j} \|_2$ for a positive constant $c$. When $n \ge C_1\left(\log^2 \nu +p^2\right)$ for any $0<\nu<1$, the prediction error for a new observation $(\bx,y)$ can be bounded by
    \bel{nsthm1-1}
    \EE_{(\bx_{\text{data}},y_{\text{data}}),\hat{\mS}} \left( y - \hat{g}_{\hat{\mS}} \left(\bx\right) \right)^2 \le C_2 \left(sn^{-2/(s+8)} \log n +s\nu+\eps_0^2 \right),
    \eel
    where $(\bx_{\text{data}},y_{\text{data}})$ refers to the sample $\left\{\bx_i, y_i\right\}_{i=1}^n$ and a new data $(\bx, y)$.
    Furthermore, when $\nu = \mathop{\min}\left\{1/2,n^{-2/(s+8)} \log n\right\}$, $n\ge C_3p^2$ and $\eps_0^2<sn^{-2/(s+8)}\log n$, we have
    \bel{nsthm1-2}
    \EE_{(\bx_{\text{data}},y_{\text{data}}),\hat{\mS}} \left( y - \hat{g}_{\hat{\mS}} \left(\bx\right) \right)^2 \le C_4 sn^{-2/(s+8)} \log n .
    \eel
 Here, $C_1$ through $C_4$ denote positive constants.
\end{thm}

Theorem \ref{nsthm1} is established under the approximate sparsity assumption, yielding a final prediction error on the order of $sn^{-2/(s+8)}\log n$. Conversely, if a neural network is trained directly on the full feature vector $\bx$, the prediction error bound becomes $\mathcal{O}(pn^{-2/(p+8)}\log n)$, which is significantly larger than that in (\ref{nsthm1-2}) due to the substantial increase in the number of parameters. 
Moreover, this result can be readily extended to Algorithm \ref{algo2} when a screening step is incorporated. In that case, the sample size requirement can be reduced to $\Omega(s \log p)$.
We summarize this extension in the following corollary.

\begin{col}\label{nscol1}
   Let $\hat{\mS}$ be generated using the Algorithm \ref{algo2}, and let $\hat{g}_{\hat{\mS}}$
 be computed according to
(\ref{op1}). Assume the conditions in Theorem \ref{nsthm1} and further assume the Assumptions \ref{min2rd}-\ref{nsa1}. When
    $n \ge C_1s\left(\log^{2} \nu + \log p \right)$, the prediction error for a new observation $(\bx,y)$ can be bounded by
    \bes
    \EE_{(\bx_{\text{data}},y_{\text{data}}),\hat{\mS}} \left( y - \hat{g}_{\hat{\mS}}\left(\bx\right) \right)^2 \le C_2 \left(sn^{-2/(s+8)}\log n+s\nu+\eps_0\right),
    \ees    
    where $(\bx_{\text{data}},y_{\text{data}})$ refers to the sample $\left\{\bx_i, y_i\right\}_{i=1}^n$ and a new data $(\bx, y)$. Furthermore, when $\nu = \mathop{\min}\left\{1/2,n^{-2/(s+8)}\log n\right\}$, $n\ge C_3s\log p$ and $\eps_0^2<sn^{-2/(s+8)}\log n$, we have
    \bes
    \EE_{(\bx_{\text{data}},y_{\text{data}}),\hat{\mS}} \left( y - \hat{g}_{\hat \mS}\left(\bx\right) \right)^2 \le C_4 sn^{-2/(s+8)}\log n.
    \ees
    Here, $C_1$ through $C_4$ denote positive constants.
\end{col}

\section{Additional issues}\label{sec_unknown}
The proposed method relies on the second-order score function 
$T(\cdot)$, first layer dimension $k_1$, and the sparsity threshold $\kappa$.
In this section, we examine the estimation of these quantities and establish feature selection consistency based on these estimated values.


We first consider an eigengap-based approach for determining the first-layer dimension $k_1$. 
 For $k=1,\ldots, p$, let $\lam_k$ and $\hat{\lam}_k$ be the $k$-th eigenvalue of $\E \left(y T(\bx)\right)$ and $(1/n) \sum_{i=1}^n y_i T(\bx_i)$, respectively. Note that the eigenvalues $\hlambda_k$ (and $\lam_k$) are considered in descending order of their magnitudes, i.e., $|\hlambda_1| \ge |\hlambda_2| \ge \cdots \ge |\hlambda_p|$.  Due to Stein's formula, we have that $\lam_{k+1}=\ldots=\lam_p=0$. Accordingly, $k_1$ can be determined using the criteria: 
$\hat{k}_1 = \mathop{\max} \left\{j \in [p]: \vert\hlambda_j\vert - \vert \hlambda_{j+1}\vert > \tau\right\}$,
 where $\tau$ is a certain threshold level to be specified later. The following theorem suggests that $\hat{k}_1$ provides a consistent estimate for $k_1$. 
 

    


\begin{thm}\label{thm3}
     Let $\delta = \mathop{\min}_{i\in [k_1]} \left\{\vert \lambda_i\vert - \vert \lambda_{i+1}\vert \right\} > 0$. Suppose that $\tau < \delta/2$, then for any $\nu>0 $, when $n\ge C(\log^2 \nu + p^2)$ for certain positive constant $C$, we have 
   $\P\left(\hat{k}_1 = k_1\right) \ge 1-\nu$.
\end{thm}

In practice, the value of $\tau$ is unknown, which makes directly applying the above procedure difficult. Nevertheless, an eigengap-based estimation can still be utilized. Intuitively, $k_1$ can be identified as the index at which the eigengap exhibits a sharp drop. Specifically, by selecting $k_1$ such that
$|\hat{\lambda}_{k-1}| - |\hat{\lambda}_{k}|$ is much larger than $|\hat{\lambda}_{k}| - |\hat{\lambda}_{k+1}|$.  Formally, we adopt the strategy of gap-statistic \citep{tibshirani2001gap} in determining the number of clusters and define the absolute eigengap ratio:
$$
r(k) = \left( |\hat{\lambda}_{k-1}| - |\hat{\lambda}_{k}| \right ) / \left ( |\hat{\lambda}_{k}| - |\hat{\lambda}_{k+1}| + \gamma_{\text{reg}} \right),
$$ 
where $\gamma_{\text{reg}}$ is a small regularization constant to ensure numerical stability. Then, $k_1$ can be selected by maximizing $r(k)$. Simulation results confirm the practical effectiveness of the eigengap ratio statistic; See Section F in the supplementary material for details.


We now consider the selection of the sparsity threshold $\kappa$. Since a hard-thresholding procedure is applied in the Algorithm \ref{alg1}, determining 
$\kappa$ is equivalent to specifying $s$, the desired sparsity level. To select $s$, we employ the Bayesian Information Criterion (BIC), choosing the value that minimizes
$\text{BIC} = n \cdot \ln\big(\text{MSE}(s)\big) + \lam_s \cdot \ln(n)$
where $\text{MSE}(s)$ is the prediction MSE obtained from a neural network trained on the $s$ selected features. Strictly speaking, $\lam_s$ corresponds to the total number of parameters in a neural network with $s$ features. To avoid the intricacies of architectural tuning, we fix $\lam_s=100 s$ throughout all simulation experiments.
Empirical results indicate that this BIC-based procedure performs well in our setting. Additional visualization results are presented in Section F of the supplementary materials.

Finally, we turn to the case of an unknown score function and show that, when the input features follow a Gaussian distribution, feature selection consistency can still be guaranteed.
Let $\bx_i \sim \mathcal{N}(\mathbf{0}, \bSigma)$ be centered Gaussian random vectors, where the covariance matrix $\bSigma$ is unknown.
A natural estimate for $\bSigma$ is the sample covariance matrix $\hat{\bSigma} = (1/n) \sum\limits_{i=1}^n \bx_i\bx_i^T$. This yields the empirical second-order score function $\hat{T}(\bx) = \hat{\bSigma}^{-1}\bx\bx^T \hat{\bSigma}^{-1} - \hat{\bSigma}^{-1}$.
Accordingly, $\hbW_1$ is be obtained by computing the top-$k_1$ eigenvectors of 
$(1/n)\sum\limits_{i=1}^n y_i \hat{T}(\bx_i)$.
The following theorem establishes that the empirical second-order score ensures feature selection consistency in both deep neural network models and H{\"o}lder-smooth function classes.


\begin{thm}\label{thm4}
    Suppose $(\bx_i,y_i)$ follows model (\ref{model1}) and assume the conditions in Theorem \ref{thm2}. Suppose that $\bx_i\sim \mathcal{N}(\mathbf{0}, \bSigma)$. Let $\hat{\bSigma} = (1/n) \sum\limits_{i=1}^n \bx_i\bx_i^T$ and  $\hat{T}(\bx) = \hat{\bSigma}^{-1}\bx\bx^T \hat{\bSigma}^{-1} - \hat{\bSigma}^{-1}$.  Then, for any $\nu>0$, when
  $n \ge C\left(\log^2 \nu + p^2\right)$ for certain positive constant $C$, Algorithm \ref{alg1} guarantees feature selection consistency for H{\"o}lder smooth functions, i.e.,
$\P(\hat{\mS}_0 =\mS_0) \ge 1-\nu$.
Moreover, if $(\bx_i,y_i)$ follows model (\ref{model2}) and the conditions in Theorem \ref{thm1} hold, we have feature selection consistency for DNNs.
\end{thm}

In Theorem \ref{thm4}, the required sample size in the case of unknown covariance is no greater than that specified in Theorems \ref{thm1} and \ref{thm2}. This is because the additional error arising from score estimation, $(1/n)\sum_{i=1}^n y_i \hat{T}(\bx_i) - (1/n) \sum_{i=1}^n y_i T(\bx_i)$, can be dominated by the original error, $\EE y T(\bx) - (1/n)\sum_{i=1}^n y_i T(\bx_i)$. As a result, a sample size of the same order is sufficient.

\section{Numerical experiments} \label{simulation}
In this section, we present a comprehensive simulation study to assess the feature selection performance of the proposed method in comparison with two competing deep learning approaches, LassoNet and DFS, while including the standard Lasso as a baseline benchmark. The results are organized as follows: Section~\ref{low}
 examines feature selection performance of Algorithm \ref{alg1} in classical settings; Section~\ref{high}
 evaluates the effectiveness of the screening-and-selection mechanism in high-dimensional settings; Section~\ref{mse}
 analyzes the prediction performance when the neural network is refitted using the selected features; and Section~\ref{tdistribution}
 investigates the robustness of our algorithm under a non-Gaussian input. 

\subsection{Feature selection in classical settings}\label{low}

In this simulation study, we assess feature selection performance using the 
\emph{True Positive Rate} (TPR) and \emph{False Positive Rate} (FPR), defined as
$\text{TPR} = |\hat{\mS} \cap \mS_0|/|\mS_0|$, 
$\text{FPR} = |\hat{\mS} \ \setminus \ \mS_0|/(p - |\mS_0|)$,
where  $\mS_0$ and $\hat{\mS}$ represent the true and selected relevant feature sets, respectively.
 We begin with a standard Gaussian design and extend to $t$-distributed inputs in Section \ref{tdistribution}.

We consider five nonlinear models, each with a sparsity level of $s=5$. 
The first three correspond to index models with $k_1=5$:
\bel{sim1}
y_i= \ba^T f(\bW_1\bx_i) +\eps_i, \ \ \ i=1,\ldots, n,
 \eel
where $\eps_i$ denotes i.i.d. standard Gaussian noise $\ba\in\mathbb{R}^{k_1}$ is a vector of linear weights drawn from a uniform distribution, the nonzero rows of $\bW_1\in\mathbb{R}^{k_1\times p}$ are randomly selected and normalized to have unit $\ell_2$-norm, and $f=f_1, f_2, f_3$ is applied element-wise as specified below:
\begin{itemize}
    \item[] \textbf{Case 1:} $f_1(z) = z^2$, \ 
\textbf{Case 2:} $f_2(z) = z^4 + 2z^2 - 10\cos(z)$, \ 
\textbf{Case 3:} $f_3(z) = \exp(z) + z^4 - z^2$.
\end{itemize}
In addition to the index models, we also consider general cases
$y_i= G (\{\bx_i\}_{\mS_0})+\eps_i
$, 
with Case 4 as an additive model and Case 5 as a more complex form:
\begin{itemize}
    \item[] {\bf Case 4}: \ \ $G(\bz)=\sum_{s=1}^5 f_s(z_s)$, \ where $f_4(z) = f_2(z)$ and  $f_5(z)= z^4 + z^2 - \cos (z)$.
    \item[] {\bf Case 5}: \ \ $G(\bz) = f_1(z_1)\times f_2(z_2) + f_3(z_3) + f_1(z_4) \times f_5(z_5)$.
\end{itemize}
For the two deep learning methods, LassoNet and DFS, we adopt the default parameter settings provided in their respective packages (e.g., the initial penalty $\lam_{start}$
 in LassoNet and the intersection number $Ts$ 
 in DFS) and configure the hidden layer sizes to 100 and 50, respectively.
 We compare the performance of different methods across multiple scenarios by varying the sample size $n$ from 100 to 5000 while keeping $p=200$ (Figure~\ref{fig:low_n}), and by varying the feature dimension $p$ from 200 to 1000 with $n=2000$ fixed (Figure~\ref{fig:low_p}).

\begin{figure}[htbp!]
    \centering
    \includegraphics[width=1\linewidth]{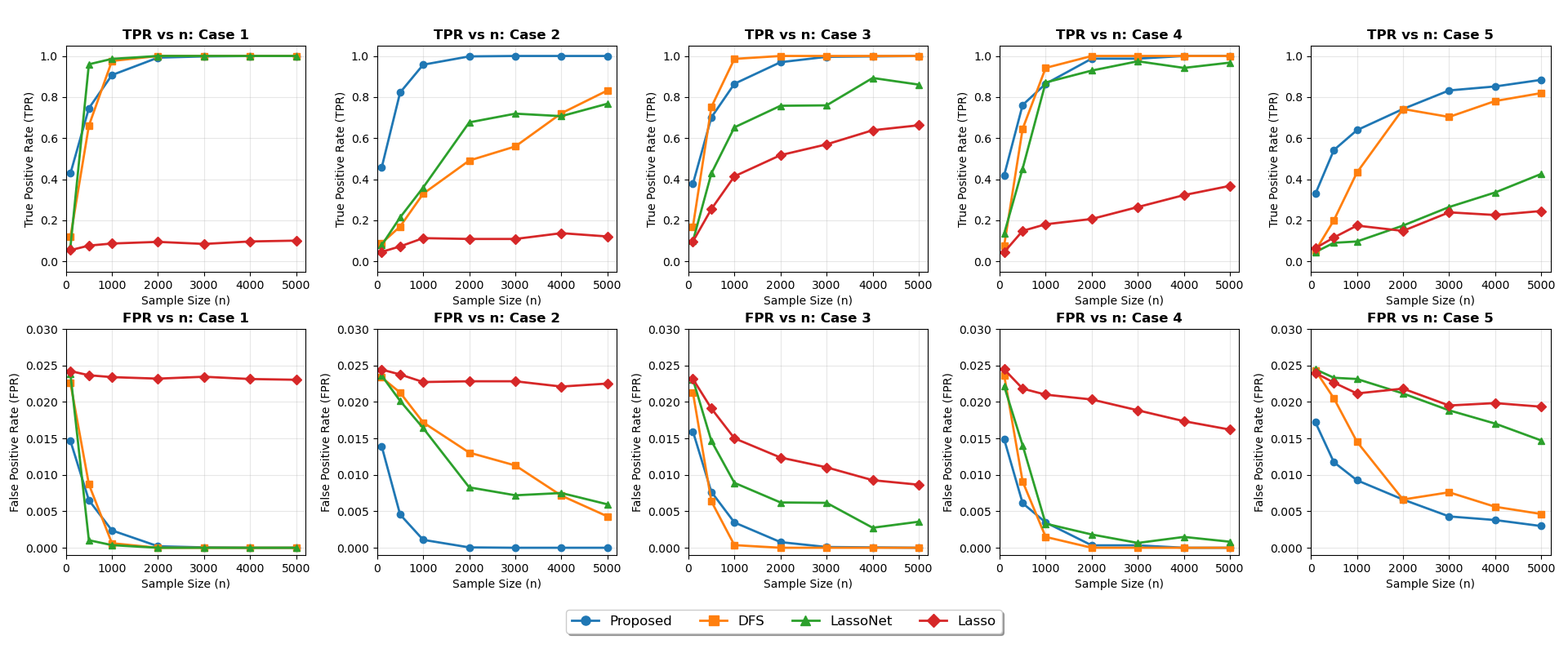}
    \caption{TPR and FPR of various methods across different $n$, with $p$ fixed at $200$. }
    \label{fig:low_n}
\end{figure}

\begin{figure}[htbp!]
    \centering
    \includegraphics[width=1\linewidth]{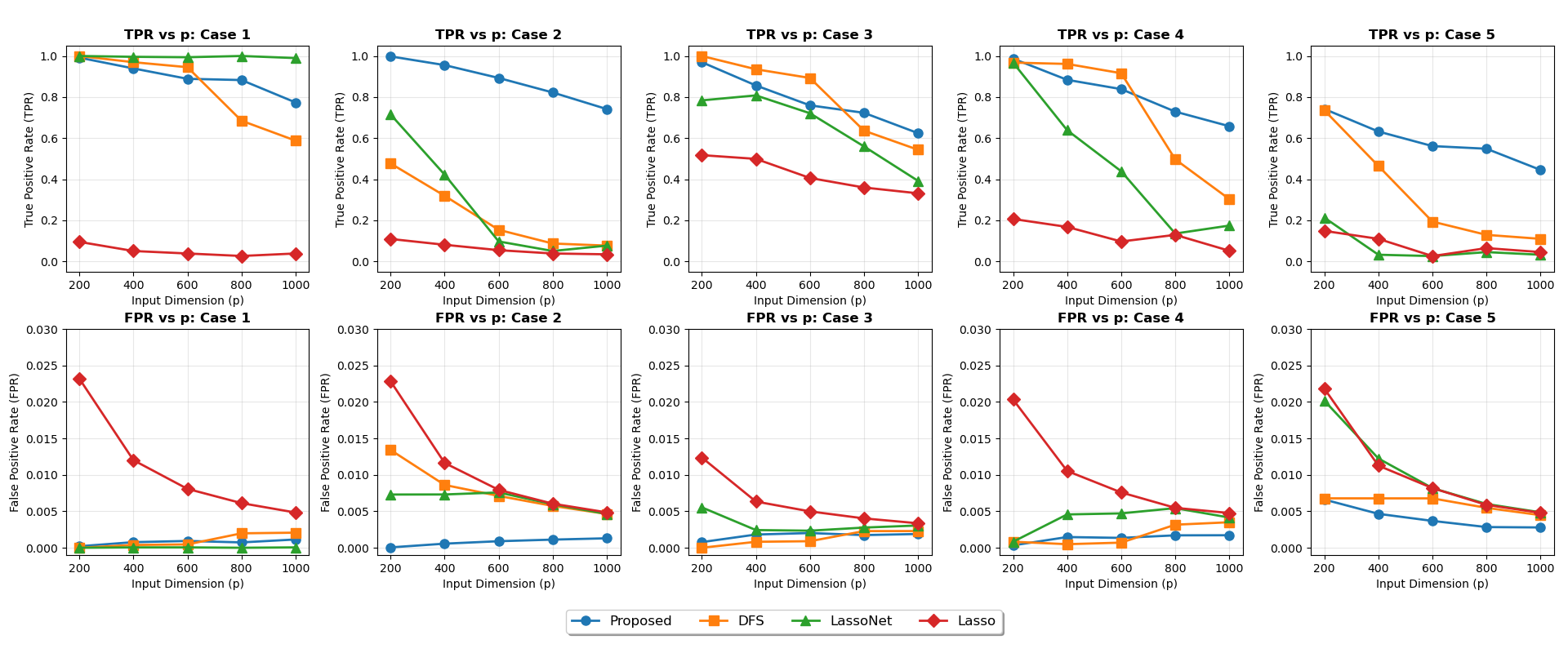}
    \caption{TPR and FPR of various methods across different $p$, with $n$ fixed at 2000.}
    \label{fig:low_p}
\end{figure}

As illustrated in Figure~\ref{fig:low_n}, the proposed method demonstrates strong performance under limited sample sizes, with particularly notable results in Case 2, highlighting its effectiveness in data-constrained settings. As $n$ increases, achieving performance on par with DFS and surpassing that of LassoNet even in Case 5 with complex interaction terms. In contrast, LassoNet’s performance declines markedly in Case 5, while Lasso performs inadequately across all sample sizes.

Figure \ref{fig:low_p} demonstrates that the proposed approach maintains robust performance as $p$ increases. While LassoNet attains near-perfect accuracy in Case 1, gradient-descent–based methods, including LassoNet, experience substantial performance degradation in Cases 2–5 when $p$ grows. Across all settings, Lasso yields the worst results, suggesting that purely linear models are inadequate for capturing the nonlinear relationships in these datasets.

\subsection{Feature selection in high-dimensional settings}\label{high}
We next examine the high-dimensional setting, comparing the proposed method with and without the screening step. As in the previous analysis, we vary the sample size $n$ from 100 to 5000 with $p=2000$ fixed (Figure~\ref{fig:high_n}), and the feature dimension $p$ from 1000 to 5000 with $n=2000$ fixed (Figure~\ref{fig:high_p}).

\begin{figure}[htbp!]
    \centering
    \includegraphics[width=1\linewidth]{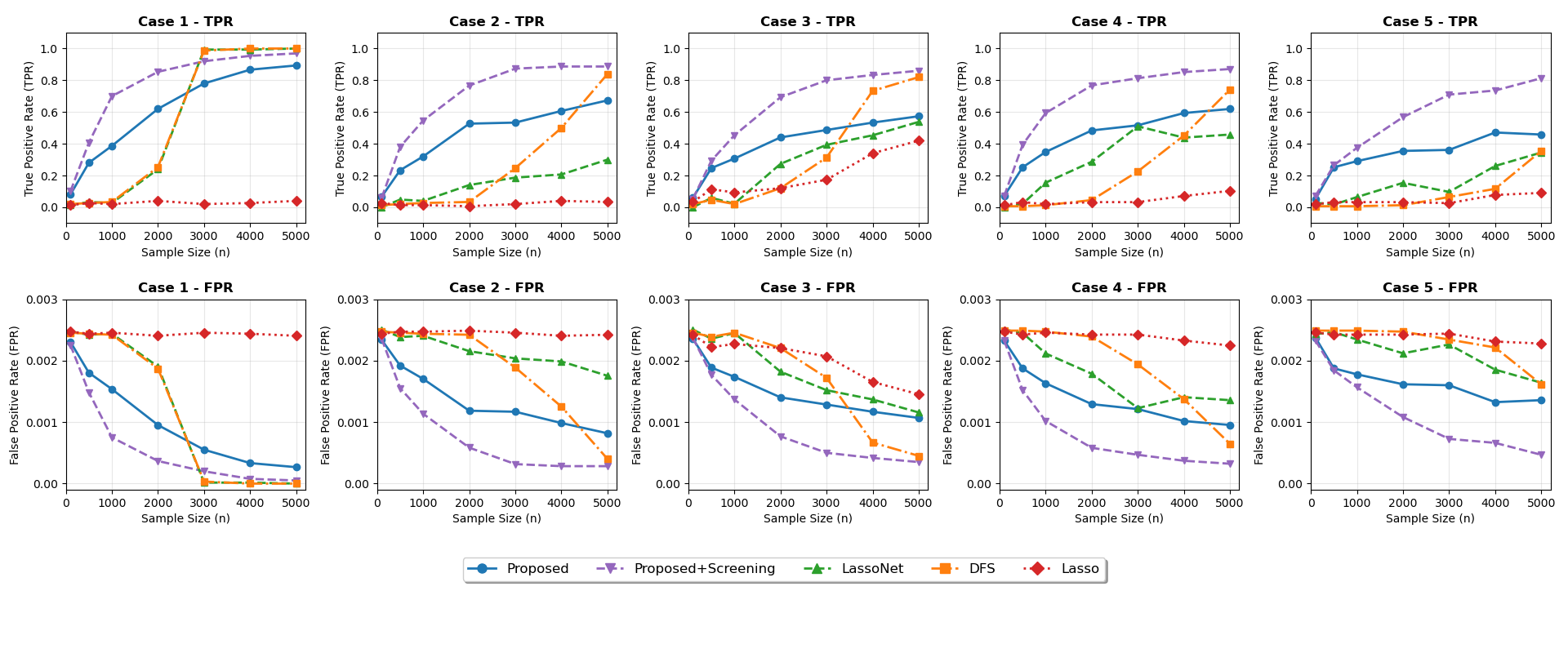}
    \caption{TPR and FPR of various methods across different $n$, with $p$ fixed at 2000.}
    \label{fig:high_n}
\end{figure}

\begin{figure}[htbp!]
    \centering
    \includegraphics[width=0.99\linewidth]{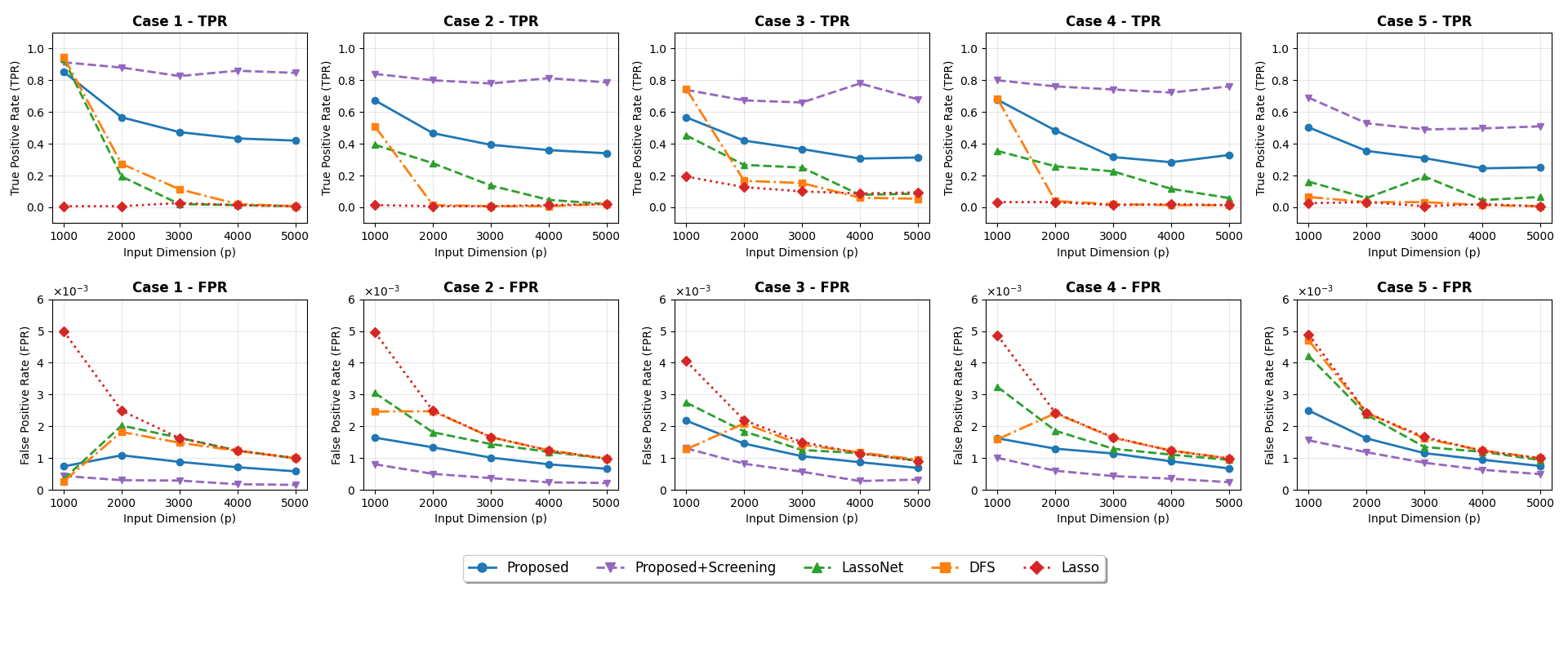}
    \caption{TPR and FPR of various methods across different $p$, with $n$ fixed at 2000. }
    \label{fig:high_p}
\end{figure}


Figure \ref{fig:high_n} shows that incorporating the screening step outperforms the original method without screening and offers a significant advantage over gradient-descent–based methods in most cases.
On the other hand, we observe that as the sample size grows, DFS performance improves substantially, indicating that gradient-descent–based methods need considerably larger samples than our nonparametric approach to reach similar accuracy. A similar pattern appears in Figure \ref{fig:high_p}, where both LassoNet and DFS deteriorate rapidly when  $p >n$. These results further demonstrate the value of our screening mechanism in effectively reducing the feature dimension.


\subsection{Prediction comparison}\label{mse}
We now evaluate the predictive performance of our approach using the two-step procedure described earlier. The predictive performance of the different methods is evaluated on a new dataset, with a prediction sample size equal to 2000. In addition to LassoNet and DFS, we also report the MSE of a full neural network model without feature selection as a benchmark.  As before, we compare performance across varying (training) sample sizes and input dimensions 
 and report the MSE in Figure \ref{fig:prediction}. 



\begin{figure}[htbp!]
    \centering
    \includegraphics[width=1\linewidth]{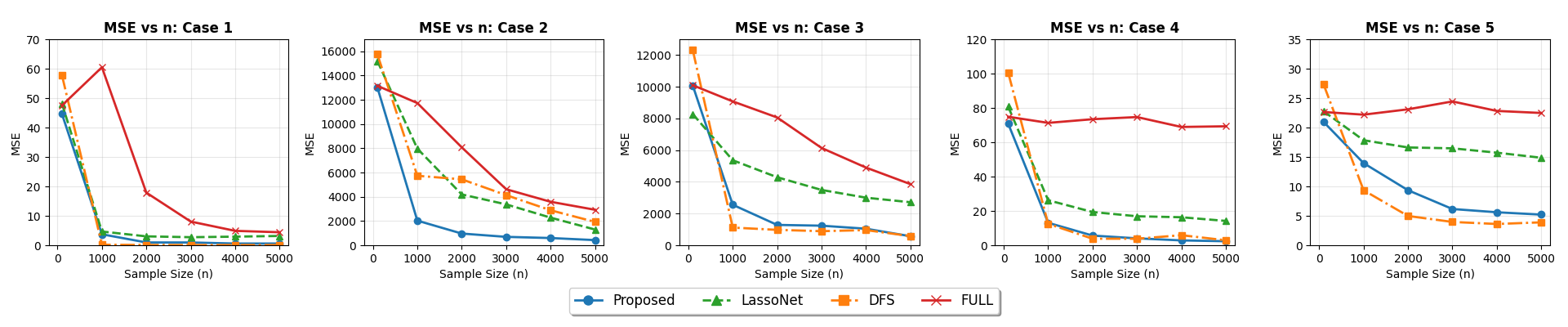}  \includegraphics[width=1\linewidth]{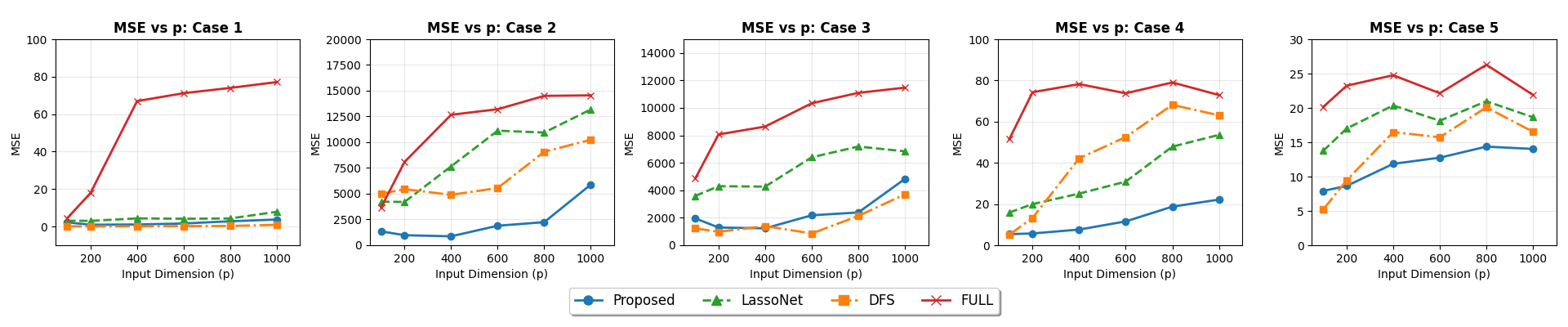}
    \caption{Row 1: Prediction MSE
     of various methods across different $n$, with $p$ fixed at 200. Row 2: Prediction MSE 
    of various methods across different $p$, with $n$ fixed at 2000.}
    \label{fig:prediction}
\end{figure}

From the first row of Figure \ref{fig:prediction}, our two-step method clearly achieves a lower prediction MSE than LassoNet and DFS when the sample size is small. As the sample size increases, LassoNet attains predictive performance comparable to ours. The full model without feature selection performs worst in nearly all settings, demonstrating the necessity of feature selection.
Likewise, the second row of Figure \ref{fig:prediction}
 shows that as the feature dimension increases, our approach is much more robust: it maintains a stable, low MSE as the number of features grows, whereas the gradient-descent-based methods and the full-model baseline exhibit worsening performance.

We also compare the computational efficiency of our method with the competing approaches, with runtimes reported in Table 5 of the supplementary materials. Because the computation times of LassoNet and DFS are highly sensitive to their respective tuning parameters --- the path multiplier for LassoNet and the intersection parameter for DFS --- we report their runtimes over a range of values. The results show that our two-step approach delivers markedly superior computational efficiency, consistently running about three times faster than DFS variants and more than ten times faster than LassoNet across all evaluated sample sizes, highlighting the computational advantages of our nonparametric approach over gradient-based methods.


\subsection{Feature selection performance under t-distribution}\label{tdistribution}
Our approach depends only on the second-order score function and is therefore readily applicable to non-Gaussian distributions. As an illustration, we present simulation results for a non-Gaussian input generated from a 
$t$-distribution with 7 degrees of freedom, with the results summarized in Figure \ref{fig:t_p}. 

\begin{figure}[htbp!]
    \centering
    \includegraphics[width=1\linewidth]{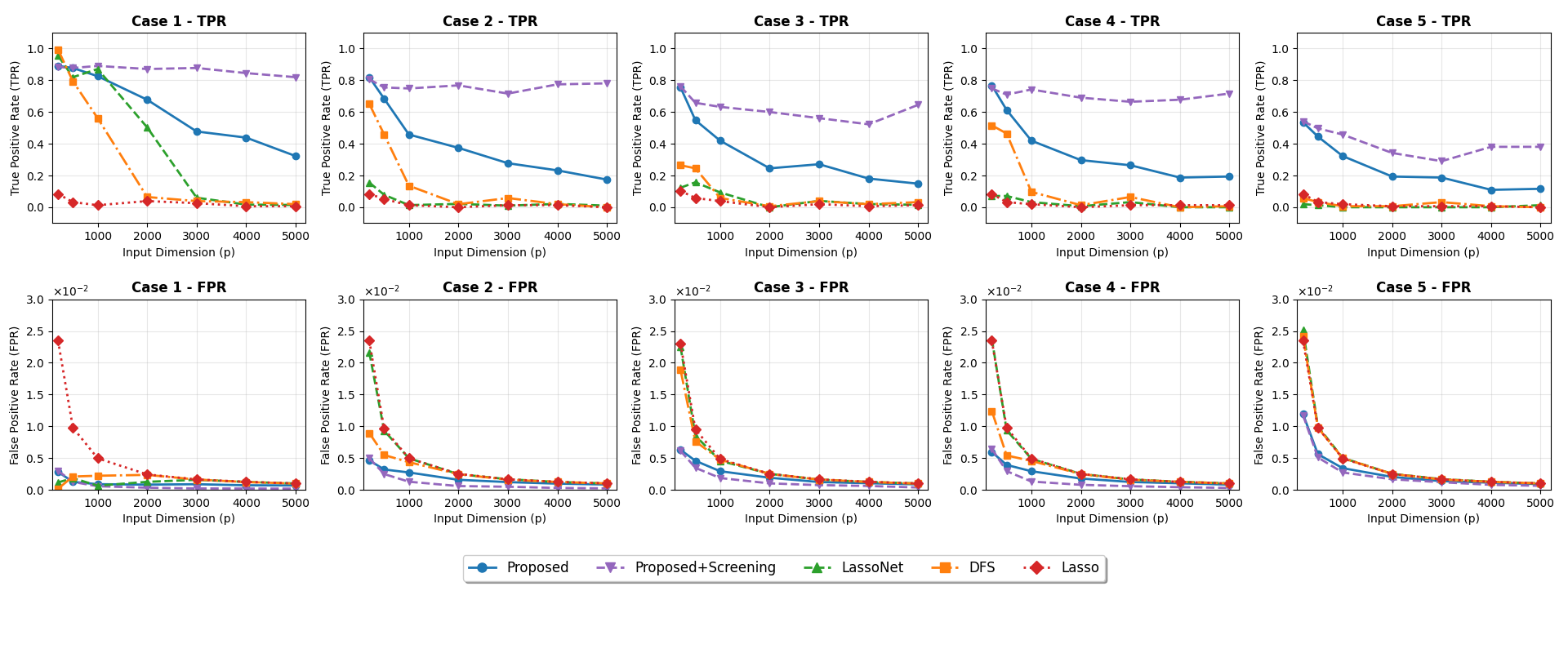}
    \caption{TPR and FPR of various methods across different $p$ under t-distributed input with degrees of freedom 7. The sample size is fixed to be 2000.}
    \label{fig:t_p}
\end{figure}

This evaluation focuses on a high-dimensional setting, comparing our method with and without the screening step. Surprisingly, the results show that our approach demonstrates even greater robustness than LassoNet and DFS—particularly when using the screening-and-selection mechanism—whose performance remains largely unaffected by the increased feature dimension.
A possible explanation is that both LassoNet and DFS are highly sensitive to the underlying network architecture, and their performance could potentially be improved with careful tuning of network structure. In our experiments, we used the default architecture design, which may have led to suboptimal results for these methods.
In contrast, our approach offers a computational advantage by eliminating the need for architecture tuning, while maintaining strong performance across different input distributions. These findings further validate the effectiveness of our method in diverse high-dimensional scenarios.

\section{The ADNI data analysis} \label{sec9}
In this section, we apply the proposed method to analyze real genetic data from the 
Alzheimer’s Disease Neuroimaging Initiative (ADNI).
The ADNI (\href{https://adni.loni.usc.edu/}{https://adni.loni.usc.edu/}) study is aimed at advancing research and developing treatments for Alzheimer's disease (AD).  It includes a comprehensive collection of clinical, imaging, genetic, and other biomarker data.

Our analysis focuses on the genotype data, specifically the single-nucleotide polymorphisms (SNP) data, to study the progression of AD as measured by the Mini-Mental State Examination (MMSE) score. The MMSE is widely used to assess cognitive function and serves as a reference for AD. Scores on the MMSE range from 0 to 30, with higher values reflecting better cognitive performance. In clinical interpretation, scores below 17 are typically associated with moderate to severe cognitive impairment, whereas scores of 24 or above are generally indicative of normal cognitive function.
For our analysis, we filtered patients with MMSE scores from three phases of the study: ADNI-1, ADNI-GO, and ADNI-2, resulting in a dataset of 755 samples. We applied the sure independence screening method \citep{fan2008sure} to reduce the feature dimension, narrowing the number of SNPs to 377. Our goal is to identify the features among these 377 SNPs that are most associated with MMSE scores.


\begin{table}[htbp]
\centering
\begin{tabular}{p{3cm} p{3cm} p{8cm}}
\toprule
\makecell[c]{SNP} & \makecell[c]{Genes} & \makecell[c]{Reported brain-related/cognitive trait(s)} \\
\midrule

\makecell[c]{rs2286735 \\ rs3121458}  &  \makecell[c]{NRAP}  & \makecell[c]{Recessive dilated cardiomyopathy\\ \citet{koskenvuo2021biallelic}} \\ 
\cdashline{1-3}[0.5pt/1pt]
\makecell[c]{rs1472228}  & \makecell[c]{RERE} & \makecell[c]{Developmental delay and intellectual disability \\
\citet{Scott2019} \\ Neurodevelopmental Disorders \\ \citet{niehaus2022phenotypic}} \\
\cdashline{1-3}[0.5pt/1pt]
\makecell[c]{rs942201 \\ rs1107345} & \makecell[c]{IL2RA} & \makecell[c]{Cerebral Palsy \\ \citet{qiao2022association}}  \\
\cdashline{1-3}[0.5pt/1pt]
\makecell[c]{rs1425861}  &  \makecell[c]{LINC02725}   &    \makecell[c]{N/A}   \\
\cdashline{1-3}[0.5pt/1pt]
\makecell[c]{rs1825975}  &  \makecell[c]{PSD3}  &  \makecell[c]{Alzheimer Disease in Hippocampus \\ \citet{quan2020related}} \\
\cdashline{1-3}[0.5pt/1pt]
\makecell[c]{rs9711441} &  \makecell[c]{KIF5C}  &  \makecell[c]{Epilepsy and Psychomotor retardation \\ \citet{banerjee2024novel}}   \\
\cdashline{1-3}[0.5pt/1pt]
\makecell[c]{rs2412971 \\ rs2412973} & \makecell[c]{HORMAD2}   &  \makecell[c]{N/A}  \\
\bottomrule
\end{tabular}
\caption{Associated SNPs identified by the proposed approach}
\label{table:gene}
\end{table}

\begin{table}[htbp]
\centering
\resizebox{\textwidth}{!}{
\begin{tabular}{cccccc}
\toprule
 & \text{Proposed} & \text{LassoNet} & \text{DFS} & \text{Lasso} \\
\midrule
\text{Proposed} & [24] & \makecell{rs3121458(NRAP)\\ rs2412971(HORMAD2) \\ (PSD3)} & \makecell{rs11624198\\ rs2286735(NRAP) \\ (IL2RA)} & \text{None} \\
\cdashline{1-5}[0.5pt/1pt]
\text{LassoNet} & & [24] & \makecell{rs1481596(DLC1)\\ rs13009814\\ rs9907824(NXN) \\ (NRAP)} & \makecell{rs2385522(FER1L6)} \\
\cdashline{1-5}[0.5pt/1pt]
\text{DFS} & & & [24] & \makecell{rs9856161} \\
\cdashline{1-5}[0.5pt/1pt]
\text{Lasso} & & & & [27] \\
\bottomrule
\end{tabular}
}
\caption{Overlap of SNPs and their corresponding genes across different methods. The diagonal elements represent the number of detected SNPs by four methods.}
\label{table:overlap}
\end{table}

We applied the proposed method and detected 24 SNPs significantly associated with MMSE scores. Notably, our approach employs Gaussian score functions, even when the input data may not strictly follow a Gaussian distribution. Of these SNPs, 10 correspond to known genes, as summarized in Table \ref{table:gene}. 
A literature review revealed that most of these biomarkers are well-documented and linked to brain or cognition-related disorders. For instance, SNP rs1825975 in the \textit{PSD3} gene has been reported to play a role in Alzheimer’s disease through pathways such as amyloid-beta formation and actin cytoskeleton reorganization \citep{quan2020related}. This suggests that \textit{PSD3} could be a promising target for both diagnostic and therapeutic strategies, offering valuable insights into the molecular mechanisms of Alzheimer’s disease. Similarly, SNP rs1472228 in the \textit{RERE} gene has been associated with neurodevelopmental disorders, influencing cognitive and psychiatric traits including developmental delay, intellectual disability, and autism spectrum disorder \citep{niehaus2022phenotypic}.


We also applied LassoNet and DFS, in addition to standard Lasso, to identify significant SNPs. The complete list of SNPs detected by these methods is provided in the supplementary material for brevity. Table \ref{table:overlap} summarizes the overlap of identified biomarkers across the four approaches.
Interestingly, the gene \textit{NRAP} was detected by three non-linear methods; this gene is primarily linked to the development of cardiomyopathy, indicating a need of further investigation. In contrast, Lasso exhibited relatively low overlap with the other methods, suggesting that linear models may be insufficient in this context and that non-linear approaches could be more suitable. Overall, our method identified a greater number of SNPs supported by existing literature as being associated with cognitive traits, underscoring its effectiveness.

\section{Discussion}\label{sec10}
In this paper, we introduced a a nonparametric framework for feature selection in neural networks and general nonlinear models. Compared to gradient-descent-based methods, our approach offers greater computational efficiency while providing rigorous theoretical guarantees.
Specifically, under Gaussian design, our approach achieves feature selection consistency for H{\"o}lder smooth functions when $n=\Omega(p^2)$, while the screening-and-selection mechanism ensures consistency for neural network functions when $n=\Omega(s \log p)$. Furthermore, we employ a two-step procedure to train a neural network using the selected features and provide theoretical guarantees for predictive performance under a relaxed sparsity condition. Simulation studies demonstrate the superior performance of our method across a wide range of functions, even in the presence of complex interaction terms. In high-dimensional regimes, our approach with the screening mechanism consistently outperforms competing gradient-descent-based methods. Analyses of real genetic datasets further underscore its practical applicability in real-world settings.

The theoretical guarantees presented in this work are established under a Gaussian design, where score estimation is relatively straightforward. While our simulations confirm the effectiveness of the method under a 
$t$-distribution, score estimation becomes more challenging in general nonparametric distribution settings. This issue has been investigated in prior work, such as \citet{hyvarinen2005estimation}. Moving forward, we aim to explore alternative score estimation techniques and assess their impact on feature selection performance in neural networks.
\\

\noindent {\bf Acknowledgement}

The real genetic data used in this study is publicly available at \href{https://adni.loni.usc.edu}{https://adni.loni.usc.edu}. For interested readers to implement our approach, we have developed a Python package, ``steinfs'' in \href{https://pypi.org/project/steinfs/}{https://pypi.org/project/steinfs} with specified PyEnv. Besides, all the simulation results in this paper are accessible at \href{https://github.com/EnkiDoctor/Nonparamatric-feature-selection-based-on-Steins-Formula}{github.com/Stein-Feature-Selection}.

\bibliography{reference}


\appendix
\newpage

\addtolength{\abovedisplayskip}{4pt}           
\addtolength{\belowdisplayskip}{8pt}           
\addtolength{\abovedisplayshortskip}{3pt}
\addtolength{\belowdisplayshortskip}{7pt}
\addtolength{\jot}{3pt}                        

\setcounter{equation}{0}
\setcounter{figure}{0}
\setcounter{table}{0}
\setcounter{thm}{0}
\setcounter{lem}{0}
\setcounter{prop}{0}
\setcounter{col}{0}
\setcounter{defn}{0}
\setcounter{algorithm}{0}  
\setcounter{section}{0}     

\begin{center}
   {\bf \large Supplementary Material to\\ ``A Nonparametric Statistics Approach to Feature Selection in Deep Neural Networks with Theoretical Guarantees''}
\end{center}

\allowdisplaybreaks 

\newcounter{tempLemma}
\newcounter{fakelemma}

In the supplementary material, we provide proof for Theorem \ref{thm1} to \ref{thm4} along with several additional lemmas and propositions. Furthermore, we include more detailed information on the implementation of the numerical examples and real data analysis.

\section{Proof of Proposition \ref{prop:hessian-nonsingular}}
\begin{prop}\label{prop:hessian-nonsingular}
    For a two-layer ReLU activated neural network \( f(\bW_1 \bx) = \bW_2 \sigma(\bW_1 \bx) \), when $\bx \sim \mathcal{N}(\boldsymbol{0}, \boldsymbol{\Sigma})$, $\EE\left[\nabla^2_{\bz} f (\bW_1\bx)\right]$ is non-singular if all elements of \(\bW_2\) are non-zero.
\end{prop}

\begin{proof}
\quad For a two-layer neural network $f(\bW_1 \bx) = \bW_2 \sigma(\bW_1 \bx)$, where $\bx \in                   \mathbb{R}^p, \bW_1 \in \mathbb{R}^{k_1 \times p}, \bW_2 \in \mathbb{R}^{1\times k_1} $ and  
    $\sigma$ is the ReLU activation function. The Hessian matrix $H_z = \nabla_z^2 f(\bW_1\bx)$ has the form of: 
    \begin{align}
          H_z = \text{diag}(\bW_{21}\delta(z_1), \bW_{22}\delta(z_2), \cdots, \bW_{2k_1}\delta(z_{k_1}))
    \end{align}
    
\noindent where $\delta(z_j)$ is the Dirac delta function satisfying $\int_{-\infty}^{\infty} \delta(z) dz= 1$ and $\delta(z) = 0$ for $z \neq 0$. Since $\bx$ follows a Gaussian distribution of $N(\textbf{0}, \boldsymbol\Sigma)$, so that $\bz = \bW_1 \bx$ follows the distribution of $N(\textbf{0}, \boldsymbol\Sigma_2)$ where $\boldsymbol\Sigma_2 = \bW_1\boldsymbol\Sigma\bW_1^T$ so that: 
    \begin{align}
        \mathbb{E}\left[\delta (z)  \right] = \int_{-\infty}^{\infty} \delta(z) p_z(z) dz =  \int_{-\infty}^{\infty} \delta(z) \frac{1}{\sqrt{2\pi \sigma_z^2}} e^{-\frac{z^2}{2\sigma_z^2}} dz = \frac{1}{\sqrt{2\pi \sigma_z^2}} 
    \end{align}
where $\sigma_z$ is the standard deviation of $z$. So that the expectation of $ \nabla_z^2 f(\bx)$ has following form: 
\begin{align}
     \mathbb{E} \left[\nabla_z^2 f(\bW_1\bx) \right] = \text{diag}\left( \frac{\bW_{21}}{\sqrt{2\pi \sigma_{11}^2}}, \frac{\bW_{22}}{\sqrt{2\pi \sigma_{22}^2}}, \cdots, \frac{\bW_{2k_1}}{\sqrt{2\pi \sigma_{k_1 k_1}^2}} \right)
\end{align}

\noindent where $\sigma_{ii}$ is the $i-$th diagonal element of the covariance matrix $\bW_1\boldsymbol\Sigma\bW_1^T$. If all elements of \(\bW_2\) are non-zero, $\EE\left[\nabla^2_{\bz} f (\bW_1\bx)\right]$ is non-singular.
\end{proof}

\section{Proof of Theorems in Section \ref{sec3}}
\setcounter{lem}{2}
\subsection{Proof of Proposition \ref{prop1-}}
\begin{lem}\label{subexpinequality}
    Suppose that $x$ is a sub-gaussian random variable, then $x^2$ is a sub-exponential random variable and:
    \begin{align}
        \left\| x\right\|_{\psi_2}^2 \le \left\| x^2\right\|_{\psi_1} \le 2\left\| x\right\|_{\psi_2}^2
    \end{align}
\end{lem}

\begin{proof}
    \quad For $k\ge 1$, we have: 
\begin{equation}
    \begin{aligned}
        \frac{1}{k} \left(\EE \left| x^2\right|^k\right)^{1/k} = \frac{1}{k} \left(\EE \left| x\right|^{2k}\right)^{1/k} \le \frac{1}{k}  \left( \sqrt{2k} \left\| x\right\|_{\psi_2}\right)^2  = 2\left\| x\right\|_{\psi_2}^2
    \end{aligned}
\end{equation}

\noindent then $x^2$ is a sub-exponential random variable and:
    \begin{equation}
        \begin{aligned}
        \left\| x^2\right\|_{\psi_1} = \mathop{\sup}_{k\ge 1}\frac{1}{k} \left(\EE \left| x^2\right|^k\right)^{1/k} \le 2\left\| x\right\|_{\psi_2}^2
    \end{aligned}
    \end{equation}

    \noindent On the other hand, for $k\ge 1$, by Cauchy's inequality, we have:
\begin{equation}
   \begin{aligned}
        \frac{1}{\sqrt{k}} \left(\mathbb{E}\left| x\right|^k\right)^{1/k} \le \frac{1}{\sqrt{k}}\left(\mathbb{E}\left| x\right|^{2k}\right)^{1/2k} = \frac{1}{\sqrt{k}} \left(\mathbb{E}\left| x^2\right|^{k}\right)^{1/2k} \le \left\| x^2\right\|_{\psi_1}^{1/2}
    \end{aligned} 
\end{equation}
    
\noindent therefore,
\begin{equation}
   \begin{aligned}
        \left\| x^2\right\|_{\psi_1} \ge \left\{\mathop{\sup}_{k\ge 1} \frac{1}{\sqrt{k}} \left(\mathbb{E}\left| x\right|^k\right)^{1/k}\right\}^2 = \left\| x\right\|_{\psi_2}^2
    \end{aligned} 
\end{equation}

\noindent Combining with the above two conclusions, we have:
\begin{equation}
   \begin{aligned}
        \left\| x\right\|_{\psi_2}^2 \le \left\| x^2\right\|_{\psi_1} \le 2\left\| x\right\|_{\psi_2}^2
    \end{aligned} 
\end{equation}
    
\end{proof}

\begin{prop}\label{prop1-}
   When $\bx \sim \mathcal{N}(\bold{0}, \bSigma)$ with $\phi_{\min}(\bSigma)>0$, the score $T(\bx)$ satisfies: 
    \begin{equation}
        \begin{aligned}
        \mathop{\sup}_{\| \bu\|_2=1} \left\| \bu^T T(\bx) \bu\right\|_{\psi_1} \le 4\phi^{-1}_{\min}(\bSigma)
        \end{aligned}
    \end{equation}
\end{prop}

\begin{proof}
    \quad Denote $\bSigma^{-1}\bx$ as $\bz$, then for any $\bu\in \mathbb{R}^p$ s.t. $\Vert \bu\Vert_2 = 1$, $\bu^T \bz = \bu^T \boldsymbol{\Sigma}^{-1} \bx$ follows the centered Gaussian distribution, then by Lemma \ref{subexpinequality}, $(\bu^T \bz)^2$ is sub-exponential random variable and:
    \begin{align}
        \left\| (\bu^T \bz)^2\right\|_{\psi_1} &\le 2\left\| \bu^T \bz\right\|_{\psi_2}^2 \le 2 \text{Var}(\bu^T \bz) 
        = 2 \bu^T \boldsymbol{\Sigma}^{-1} \bu 
        \le \frac{2}{\phi_\text{min}(\boldsymbol{\Sigma})}
    \end{align}

\noindent  Note that $T(\bx) = \bSigma^{-1} \bx\bx^T \bSigma^{-1} - \bSigma^{-1}$, then:
\begin{equation}
    \begin{aligned}
        \bu^T T(\bx) \bu &= \bu^T\left(\bSigma^{-1} \bx\bx^T \bSigma^{-1} -  \bSigma^{-1}\right) \bu \\
        &= \bu^T\left(\bSigma^{-1} \bx\bx^T \bSigma^{-1} - \EE \bSigma^{-1} \bx\bx^T \bSigma^{-1}\right) \bu \\
        &= \left(\bu^T\bSigma^{-1} \bx\right)^2 - \EE\left(\bu^T\bSigma^{-1} \bx\right)^2 \\
        &= \left(\bu^T\bz\right)^2 - \EE\left(\bu^T\bz\right)^2
    \end{aligned}
\end{equation}
    
\noindent Therefore, $\bu^T T(\bx) \bu$ follows sub-exponential distribution and:
    \begin{align}
        \left\| \bu^T T(\bx)\bu\right\|_{\psi_1}  = \left\| \left(\bu^T\bz\right)^2 - \EE\left(\bu^T\bz\right)^2 \right\|_{\psi_1} 
        \le 2 \left\|\left(\bu^T\bz\right)^2\right\|_{\psi_1} \le \frac{4}{\phi_\text{min}(\boldsymbol{\Sigma})}
    \end{align}
\end{proof}

\subsection{Proof of Theorem \ref{thm1}}




\begin{lem}\label{lm3}
    Let $\bW$ be a symmetric $p\times p$ matrix, and let $\mathcal{N}_\epsilon$ be an $\epsilon$-net of $S^{p-1}$ for some $\epsilon\in [0,1/2)$. Then,
    \begin{align}
        \left\| \bW \right\|_2 = \mathop{\sup}_{\bu\in S^{p-1}} \left| \langle \bW \bu, \bu\rangle\right| \le (1-2\epsilon)^{-1} \mathop{\sup}_{\bu\in \mathcal{N}_\epsilon} \left| \langle \bW\bu,\bu\rangle\right|.
    \end{align}
\end{lem}
\begin{proof}
    \quad Let $\bu\in S^{p-1}$ satisfies that,
    \begin{align}
        \left\|\bW\right\|_2 = \left| \langle \bW \bu, \bu\rangle\right|
    \end{align}
    then choose $\bu_0 \in \mathcal{N}_\epsilon$ such that,
    \begin{align}
        \left\|\bu-\bu_0\right\|_2\le \eps
    \end{align}
    By triangle inequality, we have:
\begin{equation}
   \begin{aligned}
        \left| \langle \bW\bu_0,\bu_0\rangle\right| &\ge \left| \langle \bW\bu,\bu\rangle\right| - \left| \langle \bW\bu_0,\bu_0\rangle-\langle \bW\bu,\bu\rangle\right| \\
        &= \left\|\bW\right\|_2 - \left| \langle \bW\bu_0,\bu_0-\bu\rangle+\langle \bW\left(\bu-\bu_0\right),\bu\rangle\right| \\
        &\ge \left\|\bW\right\|_2 - \left| \langle \bW\bu_0,\bu_0-\bu\rangle\right| - \left|\langle \bW\left(\bu-\bu_0\right),\bu\rangle\right| \\
        &\ge \left\|\bW\right\|_2 - \left\|\bW\right\|_2\cdot\left\|\bu_0\right\|_2\cdot\left\|\bu_0-\bu\right\|_2 - \left\|\bW\right\|_2\cdot\left\|\bu-\bu_0\right\|_2\cdot\left\|\bu\right\|_2 \\
        &\ge (1-2\eps)\left\|\bW\right\|_2
    \end{aligned} 
\end{equation}
    
\noindent  therefore,
    \begin{align}
        \left\|\bW\right\|_2\le (1-2\eps)^{-1}\left| \langle \bW\bu_0,\bu_0\rangle\right| \le (1-2\epsilon)^{-1} \mathop{\sup}_{\bu\in \mathcal{N}_\epsilon} \left| \langle \bW\bu,\bu\rangle\right|.
    \end{align}
\end{proof}

\begin{lem}\label{lm4}
    The unit Euclidean sphere $S^{p-1}$ equipped with the Euclidean metric satisfies for every $\epsilon>0$ that
    \begin{align}
        |\mathcal{N}_\eps | \le \left(1+\frac2\epsilon\right)^p
    \end{align}
    where $\mathcal{N}_\eps$ is the $\eps$-net of $S^{p-1}$.
\end{lem}
\noindent Lemma \ref{lm4} is the Corollary 4.2.13 in \cite{Vershynin2018HighDimensionalP}.

\vspace{1em}
\begin{lem}\label{lmrelu}
    Suppose that $\bX \in \mathbb{R}^p$ is a sub-gaussian random vector, let $g \in \mathscr{G}_{\mathcal{S}_0}^{NN}$ be a ReLU deep neural network. Then $g(\bX)$ is a sub-gaussian random variable.
\end{lem}

\begin{proof}
\quad  Since $g \in \mathscr{G}_{\mathcal{S}_0}^{NN}$, we can suppose that:
\begin{align}
    g\left(\bx\right)=\bW_L\sigma_L\left(\bW_{L-1}\cdots\sigma_2\left(\bW_2\sigma_1\left(\bW_1\bx\right)\right)\right)
\end{align}
Let $\bW_1' = \left\{\bW_1\right\}_{\cdot \mS_0}$ be the matrix composed of non-zero columns of $\bW_1$, then it is obvious that $\left\|\bW_1'\right\|_2 = \left\|\bW_1\right\|_2$. We then define
\begin{equation}
    \begin{aligned}
    g_1\left(\by\right) = \bW_L\sigma_L\left(\bW_{L-1}\cdots\sigma_2\left(\bW_2\sigma_1\left(\bW_1'\by\right)\right)\right),\quad \by \in \mathbb{R}^s.
\end{aligned}
\end{equation}

\noindent Then for $\bx\in\mathbb{R}^p$, $g\left(\bx\right) = g_1\left(\bx_{\mS_0}\right)$. We first prove that $g_1$ is a Lipschitz continuous function with the Lipschitz constant $L = \prod\limits_{l=1}^L \| \bW_l\|_2$. Notice that ReLU function is Lipschitz continuous with the Lipschitz constant $1$, then for any vector $\by,\bz\in \mathbb{R}^s$, we have:
\begin{align*}
        \vert g_1(\by)- g_1(\bz)\vert &= \left| \bW_L\left[\sigma_L(\bW_{L-1}\cdots\sigma_2(\bW_2\sigma_1(\bW_1\by))) - \sigma_L(\bW_{L-1}\cdots\sigma_2(\bW_2\sigma_1(\bW_1\bz)))\right]\right| \\
        &\le \|\bW_L\|_2 \cdot\| \sigma_L(\bW_{L-1}\cdots\sigma_2(\bW_2\sigma_1(\bW_1\by))) - \sigma_L(\bW_{L-1}\cdots\sigma_2(\bW_2\sigma_1(\bW_1\bz)))\| \\
        &\le \|\bW_L\|_2 \cdot\| \bW_{L-1}\cdots\sigma_2(\bW_2\sigma_1(\bW_1\by)) - \bW_{L-1}\cdots\sigma_2(\bW_2\sigma_1(\bW_1\bz)) \| \\
        &\le \prod\limits_{l=2}^L \|\bW_l\|_2\cdot \|\bW_1'\|_2 \cdot \| \by-\bz\| \\
        &= \prod\limits_{l=1}^L \|\bW_l\|_2\cdot \| \by-\bz\|    \\
        &= L \| \by-\bz\|  \refstepcounter{equation}\tag{\theequation}
\end{align*}
     the last inequality is derived from recursive approach. Therefore, $g_1$ is $L$-Lipschitz continuous.

\noindent Let $\bX = \left(X_1,\cdots, X_p\right)^T$, since $\bX$ is a sub-gaussian random vector, $\left\|\bX_{\mS_0}\right\| = \left(\sum_{j\in\mS_0} X_j^2 \right)^{1/2}$ is a sub-gaussian random variable which satisfies that:
    \begin{align}
        \left\| \left(\sum\limits_{j\in\mS_0} X_j^2\right)^{1/2} \right\|_{\psi_2} \le \left\|\sum\limits_{j\in\mS_0} X_j^2\right\|_{\psi_1}^{1/2} \le \left(\sum\limits_{j\in\mS_0} \left\|X_j^2\right\|_{\psi_1}\right)^{1/2} \le \left(2\sum\limits_{j\in\mS_0} \left\|X_j\right\|_{\psi_2}^2\right)^{1/2}
    \end{align}
    
\noindent  then for $k \ge 1 $, we have:
\begin{equation}
     \begin{aligned}
        \frac1{\sqrt{k}}\left(\EE\left[\left| g(\bX)-g(\boldsymbol{0} )\right|^k\right]\right)^{1/k} &= \frac1{\sqrt{k}}\left(\EE\left[\left| g_1(\bX_{\mS_0})-g_1(\boldsymbol{0} )\right|^k\right]\right)^{1/k} \\
        &\le \frac L{\sqrt{k}}\left(\EE\left[\left\|\bX_{\mS_0}\right\|^k\right]\right)^{1/k} \\
        &\le L \left\| \left\|\bX_{\mS_0}\right\| \right\|_{\psi_2} \\
        &\le  L \left(2\sum\limits_{j\in\mS_0} \left\|X_j\right\|_{\psi_2}^2\right)^{1/2} \\
        &\le \sqrt{2s}L \left\|\bX\right\|_{\psi_2}
    \end{aligned}
\end{equation}
   
\noindent  which shows that $g(\bX)-g(\boldsymbol{0})$ is sub-gaussian random variable with:
    \begin{align}
        \left\|g(\bX)-g(\boldsymbol{0})\right\|_{\psi_2} \le \sqrt{2s}L \left\|\bX\right\|_{\psi_2}
    \end{align}
    
  \noindent Therefore, $g(\bX)$ is sub-gaussian random variable and:
    \begin{align}
        \left\|g(\bX)\right\|_{\psi_2} \le \left|g(\boldsymbol{0})\right| + \sqrt{2s}L \left\|\bX\right\|_{\psi_2}
    \end{align}
\end{proof}

\begin{lem}\label{lmbeta}
    Suppose that $\bX \in \mathbb{R}^s$ is a sub-gaussian random vector and $G\in \mathscr{G}$ is a $\beta$-H{\"o}lder smooth function with $0<\beta\le1$. Then $G(\bX)$ is a sub-gaussian random variable.
\end{lem}

\begin{proof}
\quad    Note that $0<\beta \le 1$, for $k \ge 1 $, we have:
\begin{equation}
     \begin{aligned}
        \frac1{\sqrt{k}} \left(\EE\left[\left| G(\bX)-G(\boldsymbol{0} )\right|^k\right]\right)^{1/k} & \le \frac{C}{\sqrt{k}} \left(\EE\left[\left\| \bX \right\|^{k\beta}\right]\right)^{1/k} \\
        &\le \frac{C}{\sqrt{k}} \left(\EE\left[\left\| \bX \right\|^{k}\right]\right)^{\beta/k} \\
        &\le \frac{C}{\sqrt{k}} \cdot \left(\sqrt{k} \left\| \left\| \bX\right\| \right\|_{\psi_2} \right)^{\beta} \\
        &\le C\left\| \left\| \bX\right\| \right\|_{\psi_2}^\beta \\
        &\le C\left(2\sum\limits_{i=1}^s \left\|X_i\right\|_{\psi_2}^2\right)^{\beta/2} \\
        &\le  (2s)^{\beta/2} C\left\|\bX\right\|_{\psi_2}^\beta
    \end{aligned}  
\end{equation}

\noindent where $C$ is a constant, then $G(\bX)-G(\boldsymbol{0})$ is a sub-gaussian random variable. Furthermore, $G(\bx)$ is a sub-gaussian random variable and:
    \begin{align}
        \left\|G(\bX)\right\|_{\psi_2} \le \left|G(\boldsymbol{0})\right| + (2s)^{\beta/2} C\left\|\bX\right\|_{\psi_2}^\beta
    \end{align}
\end{proof}

For an integer $d>0$, let $P_d$ be the set of partitions of $[d]$ into nonempty, pairwise disjoint sets. For a partition $\mathcal{J} = \left\{J_1,\cdots, J_k\right\}$, and a $d$-indexed tensor $\bA = \left(a_{\bi}\right)_{\bi\in[n]^d}$, define:
\begin{align}
    \left\| \bA \right\|_\mathcal{J} = \mathop{\sup}\left\{\sum\limits_{\bi\in[n]^d} a_{\bi}\prod\limits_{l=1}^k \bu_{\bi_{J_l}}^{(l)}: \left\| \bu^{(l)} \right\|_2 \le 1, \bu^{(l)} \in \mathbb{R}^{n^{\left|J_l\right|}}, 1\le l\le k \right\}
\end{align}

\begin{lem}\label{lm5new}
    Let $\bx = (x_1,\cdots,x_n)$ be a random vector with independent components, such that for all $i\le n$, $\|x_i\|_{\psi_2} \le M$. Then for every polynomial $f: \mathbb{R}^n \to \mathbb{R}$ of degree $D$ and any $t>0$,
    \begin{align}
        \P\left(\vert f(\bx)-\EE f(\bx)\vert\ge t\right) \le 2\exp\left\{-\frac
        1{C_D} \eta_f(t) \right\} 
    \end{align}
    where $\eta_f(t) = \mathop{\min}_{1\le d\le D} \mathop{\min}_{\mathcal{J}\in P_d} \left(\frac{t}{M^d\| \EE \mathcal{D}^d f(\bx)\|_{\mathcal{J}}}\right\}^{2/\vert\mathcal{J}\vert}$, $\mathcal{D}^df$ denote the $d$-th derivative of $f$ and $C_D$ is a positive constant.
\end{lem}

\noindent Lemma \ref{lm5new} is the Theorem 1.4 in \cite{adamczak2015concentration}. We borrow the technique of Lemma S6.2 from \cite{su2024estimating} and give the tail inequality for the product of sub-gaussian variable and sub-exponential variable. We summarize this as the following Lemma \ref{lm8new}.

\vspace{1em}
\begin{lem}\label{lm8new}
    Let $\left\{(x_i,y_i)\right\}_{i=1}^n$ be a set of independent random pairs with $\mathop{\sup}_{1\le i\le n} \|x_i\|_{\psi_2}\le M_{x}$ and $\mathop{\sup}_{1\le i\le n} \|y_i\|_{\psi_1} \le M_{y}$. Let $M = M_{x}M_{y}$, then there exists constants $c, d>0$ such that for any $t \ge \frac{dM}{n} $,
    \begin{align}
        \P\left(\left|\frac1n\sum\limits_{i=1}^n \left[x_iy_i -\EE x_iy_i\right]\right| \ge t\right) \le 4\exp\left\{-c\mathop{\min}\left[ \left(\frac{t\sqrt{n}}{M}\right)^2, \left(\frac{tn}{M}\right)^{2/3}\right]\right\}
    \end{align}
\end{lem}

\noindent The only difference between Lemma \ref{lm8new} and Lemma S6.2 from \cite{su2024estimating} is that the constructed polynomial function $F$ is of order three instead of order four, and the constructed sub-gaussian random variables are $z_i^+ = \left\{\left(x_iy_i\right)^+\right\}^{1/3}$ and $z_i^- = \left\{\left(x_iy_i\right)^-\right\}^{1/3}$, where $u^+ = \max\left\{u, 0\right\}$ and $u^- = \max\left\{-u, 0\right\}$. Then Lemma \ref{lm5new} can be applied directly to prove Lemma \ref{lm8new}. 

\vspace{1em}
\begin{lem}\label{e2bound}
Assume the conditions in Theorem \ref{thm1}, we define matrix $\bE_1 = \frac{1}{n}\sum_{i=1}^{n} y_i T(\bx_i) - \EE yT(\bx)$. For any $\nu > 0$, when $n > \frac{1}{c^2}(\log \frac{4}{\nu}+ \log 9 \cdot p)^2 $, with probability at lease $1-\nu$, we have:
\begin{align}
     \left\| \bE_1 \right\|_2 =  \left\| \frac{1}{n}\sum_{i=1}^{n} y_i T(\bx_i) - \EE yT(\bx) \right\|_2  \le  \frac{2M}{\sqrt{cn}} \sqrt{\log \frac{4}{\nu} + \log 9 \cdot p}
\end{align}
where $c, M >0$ are constants. 
\end{lem}

\begin{proof}
\quad By Assumption \ref{a3-}, for any $\bu\in \mathbb{R}^p$ s.t. $\Vert \bu\Vert_2 = 1$, $\bu^T T(\bx) \bu$ is a sub-exponential random variable with $\left\| \bu^T T(\bx) \bu\right\|_{\psi_1} \le M_1$ for a given constant $M_1$. Since $\bx$ is a sub-gaussian vector, by Lemma \ref{lmrelu}, $g(\bx)$ is a sub-gaussian random variable. Combining with the assumption that $\eps$ is Gaussian noise, we can obtain that $y$ is a sub-gaussian random variable with:
\begin{equation}
    \begin{aligned}
    \left\| y\right\|_{\psi_2} &\le \left\|g(\bx)\right\|_{\psi_2} + \left\|\eps\right\|_{\psi_2} \\
    &\le \left|g(\boldsymbol{0})\right| + \sqrt{2s}L \left\|\bx\right\|_{\psi_2} + \sigma_\eps\\
    &\triangleq M_2
    \end{aligned}
\end{equation}

\noindent where $\sigma_\eps$ is the standard deviation of $\eps$. Let $M = M_1M_2$ be a finite constant. According to Lemma \ref{lm8new}, there exists constants $c,d >0$, when $t \ge \frac{dM}{n}$,
\begin{align*}
    \P \left( \vert \bu^T \bE_1 \bu\vert \ge t  \right)  &=  
    \P\left(\left|\frac1n\sum\limits_{i=1}^n \left[ \bu^Ty_iT(\bx_i)\bu - \EE(\bu^T yT(\bx) \bu)
  \right]\right|  \ge t \right) \\ &\le 4\exp\left\{-c\mathop{\min}\left[ \left(\frac{t\sqrt{n}}{M}\right)^2, \left(\frac{tn}{M}\right)^{2/3}\right]\right\} \refstepcounter{equation}\tag{\theequation}
\end{align*}

\noindent Therefore, let $S_\epsilon^{p-1}$ be the $\epsilon$-net of $S^{p-1}$, by Lemma \ref{lm3} and Lemma \ref{lm4}, we have:
    \begin{equation}
         \begin{aligned}
        \P\left(\left\| \bE_1 \right\|_2 \ge 2t\right) &= \P\left(\mathop{\sup}_{\bu \in S^{p-1}} \left| \bu^T \bE_1 \bu\right| \ge 2t\right) \\
        &\le \P \left(\mathop{\sup}_{\bu \in S_{1/4}^{p-1}} \left| \bu^T \bE_1 \bu \right| \ge t \right) \\
        &\le  \left| S_{1/4}^{p-1} \right| \cdot 4\exp\left\{-c\mathop{\min}\left[ \left(\frac{t\sqrt{n}}{M}\right)^2, \left(\frac{tn}{M}\right)^{2/3}\right]\right\} \\
        &\le  4 \cdot 9^p \exp\left\{-c\mathop{\min}\left[ \left(\frac{t\sqrt{n}}{M}\right)^2, \left(\frac{tn}{M}\right)^{2/3}\right]\right\}
    \end{aligned}
    \end{equation}

\noindent Basing on above, let $t = \frac{M}{\sqrt{cn}} \sqrt{\log \frac{4}{\nu} + \log 9 \cdot p} \ge \frac{dM}{n}$, when $n > \frac{1}{c^2}(\log \frac{4}{\nu}+ \log 9 \cdot p)^2 $, we have:
\begin{equation}
    \begin{aligned}
    \P\left(\left\|\bE_1\right\|\ge2t\right) \le 4 \cdot 9^p \exp\left\{-c\left(\frac{t\sqrt{n}}{M}\right)^2\right\} = \nu 
\end{aligned}
\end{equation}

\noindent therefore, with probability at least $1-\nu$,
\begin{align}
    \Vert \bE_1 \Vert_2 \le 2t = \frac{2M}{\sqrt{cn}} \sqrt{\log \frac{4}{\nu} + \log 9 \cdot p}
\end{align}

\end{proof}

\begin{lem}\label{A2inf}
    For any matrix $\bW\in \mathbb{R}^{p\times d}$, we have:
    \begin{align}
        \frac{1}{\sqrt{p}} \left\| \bW \right\|_2 \le \Vert \bW \Vert_\infty \le \sqrt{d}\Vert \bW\Vert_2
    \end{align}
\end{lem}

\begin{proof}
\quad For any $\bu \in \mathbb{R}^d$, $\Vert \bu \Vert_2 = 1$, we have:
    \begin{align}
        \Vert \bW \bu \Vert_2 \leq \sqrt{p} \Vert \bW \bu \Vert_{\infty} \leq \sqrt{p} \Vert \bW \Vert_\infty \Vert \bu \Vert_{\infty} \leq \sqrt{p} \Vert \bW \Vert_\infty
    \end{align}
    then we have, 
    \begin{align}
        \frac{1}{\sqrt{p}} \Vert \bW\Vert_2 = \frac{1}{\sqrt{p}}\mathop{\sup}_{\Vert \bu \Vert_2 = 1} \Vert \bW\bu\Vert_2 \leq  \Vert \bW \Vert_\infty
    \end{align}

   \noindent For the second inequality, note that for any $\bu \in \mathbb{R}^d$ such that $\Vert \bu \Vert_\infty = 1$, we have $\Vert \bu \Vert_2 \le \sqrt{d}$, therefore,
   \begin{equation}
       \begin{aligned}
        \Vert \bW \Vert_\infty  = \mathop{\sup}_{\Vert \bu\Vert_\infty = 1} \Vert \bW \bu\Vert_\infty \le \mathop{\sup}_{\Vert \bu \Vert_\infty = 1}\Vert \bW \bu\Vert_2 \le \mathop{\sup}_{\Vert \bu \Vert_\infty = 1}\Vert \bW\Vert_2 \cdot \Vert \bu \Vert_2 \le \sqrt{d}\Vert \bW\Vert_2
    \end{aligned}
   \end{equation}
\end{proof}

\begin{lem}\label{eigenvectorbound}
    Assume the conditions in Theorem \ref{thm1}, let $\delta = \mathop{\min}_{i\in [k_1]} \left\{\vert \lambda_i\vert - \vert \lambda_{i+1}\vert \right\} > 0$, $h(\nu, p) = \log \frac{4}{\nu} + \log 9 \cdot p$, then for any $\nu>0$, when $n>\max\left\{ \left(\frac{1}{c_1^2}+\frac{c_2k_1^6\mu^4}{\lambda_{k_1}^2}\right)h^2(\nu,p), \frac{4M^2}{c_1\delta^2}h(\nu,p)\right\}$,  with probability at least $1- \nu$:
    
    
    \begin{align}\label{eigenvectorbound-1}
        \|\bW - \hat\bW_1\|_\text{max} \le \frac{2M}{\sqrt{c_1n}}\left(\frac{k_1^4\mu^2}{\vert \lambda_{k_1}\vert}+ \frac{k_1^{3/2}\mu^{1/2}}{\delta\sqrt{p} }\right)\sqrt{\log \frac{4}{\nu} + \log 9 \cdot p}
    \end{align}
    where $M, c_1, c_2$ are positive constants.
\end{lem}

\begin{proof}
\quad   Let $\bE_1 = \frac{1}{n}\sum_{i=1}^{n} y_i T(\bx_i) - \EE yT(\bx)$. By Lemma \ref{e2bound}, when $n > \frac{1}{c_1^2}(\log \frac{4}{\nu}+ \log 9 \cdot p)^2 $, with probability at least $1-\nu$:
    \begin{align}
        \Vert \bE_1 \Vert_2 \le \frac{2M_1}{\sqrt{c_1n}} \sqrt{\log \frac{4}{\nu} + \log 9 \cdot p}
    \end{align}
    where $M_1, c_1$ are positive constants. Then from Lemma \ref{A2inf}, we have:
    \begin{align}
        \Vert \bE_1 \Vert_\infty \le \sqrt{p}\Vert \bE_1 \Vert_2 \le \frac{2M_1\sqrt{p}}{\sqrt{c_1n}} \sqrt{\log \frac{4}{\nu} + \log 9 \cdot p}
    \end{align}
    
    \noindent Let $\mu(\bW) = (s/k_1) \cdot \max \| \bW_{\cdot j}  \|_2^{2}$,  there exist an invertible matrix $\bO$ so that $\bW = \bO \bW_1$ and $\mu(\bW) = \mu(\bO \bW_1) \leq c_0 \mu$. Note that if $n>\max\left\{ \frac{c_2k_1^6\mu^4}{\lambda_{k_1}^2}h^2(\nu,p), \frac{4M_1^2}{c_1\delta^2}h(\nu,p)\right\}$ for some positive constant $c_2$, then $|\lambda_{k_1}| = \Omega\left(k_1^3\mu^2(\bW)\left\|\bE_1\right\|_\infty \right)$ and $\Vert \bE_1\Vert_2 < \delta$, 
    by the Theorem 3 in \cite{fan2017ellinftyeigenvectorperturbationbound}, we have:
    \begin{equation}
        \begin{aligned}
        \|\bW - \hat\bW_1\|_\text{max}&\le M_2 \left(\frac{k_1^4\mu^2(\bW)\Vert \bE_1\Vert_\infty}{\vert \lambda_{k_1}\vert \sqrt{p}}+ \frac{k_1^{3/2}\mu^{1/2}(\bW)\Vert \bE_1 \Vert_2}{\delta \sqrt{p}}\right) \\
        &\le \frac{2M_1M_2'}{\sqrt{c_1n}}\left(\frac{k_1^4\mu^2}{\vert \lambda_{k_1}\vert}+ \frac{k_1^{3/2}\mu^{1/2}}{\delta\sqrt{p} }\right)\sqrt{\log \frac{4}{\nu} + \log 9 \cdot p}
    \end{aligned}
    \end{equation}
    where $M_2, M_2'>0$ are constants. 
    
    \noindent Therefore, let $M = \max\left\{M_1, M_1M_2'\right\}$, when $n>\max\left\{ \left(\frac{1}{c_1^2}+\frac{c_2k_1^6\mu^4}{\lambda_{k_1}^2}\right)h^2(\nu,p), \frac{4M^2}{c_1\delta^2}h(\nu,p)\right\}$, the inequality (\ref{eigenvectorbound-1}) holds with probability at least $1-\nu$.
    
\end{proof}

\begin{thm}\label{thm1}
    Consider the model (\ref{model2}) with Gaussian noise $\eps$. Assume Assumptions \ref{a2new} - \ref{a3-}. Further assume the threshold level $\kappa \le c \min_{j \in \mS_0} \|\{\bW_1\}_{\cdot j}\|_2$ for certain $c>0$. 
    Suppose that $\bx$ is a sub-gaussian vector.  Then, for any $\nu>0$, when
  $n \ge C\left(\log^2 \nu + p^2\right)$ for certain positive constant $C$, Algorithm \ref{alg1} achieves feature selection consistency with probability at least $1-\nu$, i.e., 
\bes
\P(\hat{\mS}_0 =\mS_0) \ge 1-\nu.
\ees
\end{thm}

\begin{proof}
\quad  By Lemma \ref{eigenvectorbound}, when $n>\max\left\{ \left(\frac{1}{c_1^2}+\frac{c_2k_1^6\mu^4}{\lambda_{k_1}^2}\right)h^2(\nu,p), \frac{4M^2}{c_1\delta^2}h(\nu,p)\right\}$, with probability at least $1-\nu$:
    \begin{equation}
        \begin{aligned}
        \|\bW - \hat\bW_1\|_\text{max} \le \frac{2M}{\sqrt{c_1n}}\left(\frac{k_1^4\mu^2}{\vert \lambda_{k_1}\vert}+ \frac{k_1^{3/2}\mu^{1/2}}{\delta\sqrt{p} }\right)\sqrt{\log \frac{4}{\nu} + \log 9 \cdot p}
        \end{aligned}
    \end{equation}
    
    \noindent For $j\in \hat\mS_0$, when $n > \frac{4M^2}{c_1\kappa^2}\left(\frac{k_1^{9/2}\mu^2}{\vert \lambda_{k_1}\vert}+\frac{k_1^2\mu^{1/2}}{\delta \sqrt{p}}\right)^2\left(\log \frac{4}{\nu} + \log 9 \cdot p\right)$, we have:
    
    \begin{equation}
        \begin{aligned}
        \left\|\bW_{\cdot j}\right\|_2 &\ge \|\{\hat\bW_1\}_{\cdot j}\|_2 - \sqrt{k_1}\Vert \bW-\hat\bW_1\Vert_\text{max} \\
        &\ge \kappa - \sqrt{k_1}\Vert \bW-\hat\bW_1 \Vert_\text{max} \\
        &\ge \kappa - \sqrt{k_1} \cdot\frac{2M}{\sqrt{c_1n}}\left(\frac{k_1^4\mu^2}{\vert \lambda_{k_1}\vert}+ \frac{k_1^{3/2}\mu^{1/2}}{\delta\sqrt{p} }\right)\sqrt{\log \frac{4}{\nu} + \log 9 \cdot p}\\
        &> 0
    \end{aligned}
    \end{equation}
    
   \noindent then $j\in \mS_0$. Furthermore, $ \hat\mS_0 \subseteq \mS_0$. On the other hand, note that there exists an invertible matrix $\bO$ such that $\bW = \bO \bW_1$, let $c = 1/(2\|\bO^{-1}\|_2)$, for $j \in \mS_0$, we have: 
        \begin{align*}
        \|\hat\bW_{\cdot j}\|_2 &\ge \|\bW_{\cdot j}\|_2 - \sqrt{k_1}\|\bW - \hat\bW_1\|_\text{max} \\
        &\ge \min_{j \in \mS_0} \|\bW_{\cdot j}\|_2 - \sqrt{k_1} \cdot \frac{2M}{\sqrt{c_1n}}\left(\frac{k_1^4\mu^2}{\vert \lambda_{k_1}\vert}+ \frac{k_1^{3/2}\mu^{1/2}}{\delta\sqrt{p} }\right)\sqrt{\log \frac{4}{\nu} + \log 9 \cdot p}\\
        &\ge \|\bO^{-1}\|_2^{-1}\min_{j \in \mS_0} \|\{\bW_1\}_{\cdot j}\|_2 - \kappa \\
        &= 2c\min_{j \in \mS_0} \|\{\bW_1\}_{\cdot j}\|_2 - \kappa \\
        &\ge \kappa   \refstepcounter{equation}\tag{\theequation}
    \end{align*}

   \noindent then $j \in \hat\mS_0$. Furthermore, $\mS_0 \subseteq \hat\mS_0$.

 \noindent Combining the above two conclusions, under the condition that  $n \ge C\left(\log^2 \nu + p^2\right)$ for certain positive constant $C$,  with probability at least $1-\nu$, we have $\mS_0 = \hat\mS_0$.

\end{proof}

\subsection{Proof of Lemma \ref{lem:fake}}

\refstepcounter{fakelemma}
\noindent \textbf{Lemma }1.\label{lem:fake} 
Let $G$ be any $\beta$-H{\"o}lder smooth function in $\mathscr{G}$. For any $m>0$ and $w, d \in \mathbb{N}^+$, there exists a function $g \in \mathscr{G}_{\mathcal{S}_0}^{NN}$ implemented by ReLU neural network 
with width $3^{s+3}\max\left\{s\lfloor w^{1/s} \rfloor, w+1\right\}$ and depth $12d +14 + 2s$ such that
    \bel{lm2-1}
\sup_{\bx_{\mS_0}\in [-m,m]^s} \left|G(\bx_{\mS_0})-g(\bx)\right|\le Cm^\beta w^{-2\beta/s}d^{-2\beta/s},
\eel
    where $C$ is a certain positive constant. In particular, let $m = \mathcal{O}(n^{1/s})$ and $w = d = \lceil n^{\frac{2\beta+s}{8\beta} }\rceil$. 
   Then the difference between $G(\bx_{\mS_0})$ and $g(\bx)$ reduces to
\bel{lm2-2}
\sup_{\bx_{\mS_0}\in [-m,m]^s} \left|G(\bx_{\mS_0})-g(\bx)\right|
\le C n^{-1/2}.
\eel

\begin{proof}
    \quad By Theorem 4.3 of \cite{shen2019deep}, there exists a ReLU neural network function $g_1$ with width width $3^{s+3}\max\left\{s\lfloor w^{1/s} \rfloor, w+1\right\}$ and depth $12d +14 + 2s$ such that:
    \begin{equation}
        \begin{aligned}
        \sup_{\by\in [-m,m]^s} \left|G(\by)-g_1(\by)\right|\le 19\sqrt{s}\omega_G^{[-m,m]^s}\left(2mw^{-2/s}d^{-2/s}\right)
    \end{aligned}
    \end{equation}
    
\noindent  where $\omega_f^I$ is the modulus of continuity of $f$ defined by:
\begin{equation}
    \begin{aligned}
        \omega_f^I(r) := \mathop{\sup} \left\{\left| f(\bx)-f(\by)\right|: \bx, \by \in I, \left|\bx-\by\right|_2 \le r\right\}
    \end{aligned}
\end{equation}
    
\noindent   Since $G \in \mathscr{G}$ is a $\beta$-H{\"o}lder smooth function, we have $\omega_G^I(r) \le C_1 r^\beta$ for certain constant $C_1$, then:
\begin{equation}
    \begin{aligned}
        \sup_{\by\in [-m,m]^s} \left|G(\by)-g_1(\by)\right|&\le 19\sqrt{s} \cdot C_1\left(2mw^{-2/s}d^{-2/s}\right)^\beta \\
        &= 19C_1\sqrt{s}\cdot 2^\beta m^\beta w^{-2\beta/s}d^{-2\beta/s}
    \end{aligned}
\end{equation}
    
\noindent Suppose that,
\begin{equation}
    \begin{aligned}
        g_1(\by) = \bW_L\sigma_L(\bW_{L-1}\cdots\sigma_2(\bW_2\sigma_1(\bW_1\by)))
    \end{aligned}
\end{equation}
    
\noindent   with $\bW_1\in \mathbb{R}^{k_1\times s}$, we define matrix $\bW_1' \in \mathbb{R}^{k_1\times p}$ such that:
\begin{equation}
    \begin{aligned}
        &\left\{\bW_1'\right\}_{\cdot \mS_0} = \bW_1, \\
        &\left\{\bW_1'\right\}_{\cdot j} = \boldsymbol{0}, \ j \notin \mS_0
    \end{aligned}
\end{equation}
    
\noindent then let, 
\begin{equation}
    \begin{aligned}
        g(\bx) = \bW_L\sigma_L(\bW_{L-1}\cdots\sigma_2(\bW_2\sigma_1(\bW_1'\bx))), \quad \bx \in \mathbb{R}^p
    \end{aligned}
\end{equation}

\noindent It is obvious that $g(\bx) = g_1(\bx_{\mS_0})$. Therefore, we have:
\begin{equation}
    \begin{aligned}
        \mathop{\sup}_{\bx_{\mS_0}\in [-m,m]^s} \left|G\left(\bx_{\mS_0}\right)-g(\bx)\right| = \mathop{\sup}_{\bx_{\mS_0}\in [-m,m]^s} \left|G\left(\bx_{\mS_0}\right)-g_1\left(\bx_{\mS_0}\right)\right|\le C_2m^\beta w^{-2\beta/s}d^{-2\beta/s} 
    \end{aligned}
\end{equation}
    
\noindent where $C_2 = 19C_1\sqrt{s}\cdot 2^\beta$. \\
\noindent Furthermore, if $m \le C_3n^{1/s}$ for some constant $C_3$ and $w = d = \lceil n^{\frac{2\beta+s}{8\beta} }\rceil$, let $C = C_2C_3^\beta$, we have:
\begin{align}
        \mathop{\sup}_{\bx_{\mS_0}\in [-m,m]^s} \left|G(\bx_{\mS_0})-g(\bx)\right| &\le C_2m^\beta w^{-2\beta/s}d^{-2\beta/s} \le C_2C_3^\beta n^{\beta/s} \cdot n^{-\left(2\beta+s\right)/4s} \cdot n^{-\left(2\beta+s\right)/4s} = C n^{-1/2} 
\end{align}
\end{proof}
\subsection{Proof of Proposition \ref{dnnprop}}
\begin{prop}[Deep Neural Network Approximation]\label{dnnprop}
Let $G$ be any $\beta$-H{\"o}lder smooth function in $\mathscr{G}$ with $0<\beta\le1$. Suppose that $\bx$ is sub-gaussian. Let $w = d = \lfloor n^{\frac{2\beta+s}{8\beta} }\rfloor$, there  exists a function $g\in \mathscr{G}_{\mathcal{S}_0}^{NN}$ implemented by a ReLU FNN with width $3^{s+3}\max\left\{s\lfloor w^{1/s} \rfloor, w+1\right\}$ and depth $12d +14 + 2s$ such that:
 \begin{align}
    \left\| \mathbb{E} 
    \left[\left\{G(\bx_{\mS_0})-g(\bx)\right\}T(\bx)\right] \right\|_2 \le C n^{-1/2}
\end{align}
where $C$ is a constant.
\end{prop}

\begin{proof}
\quad By Lemma \ref{lem:fake}, we can find $g \in \mathscr{G}_{\mathcal{S}_0}^{NN}$, such that:
\begin{equation}
    \begin{aligned}
        \sup_{\bx_{\mS_0}\in [-m,m]^s} \left|G(\bx_{\mS_0})-g(\bx)\right|\le C_1 n^{-1/2}
    \end{aligned}
\end{equation}
\noindent where $m = \|\bx \|_{\psi_2} \cdot n^{1/s}$ and $C_1$ is a positive constant. By Lemma \ref{lmrelu} and \ref{lmbeta}, $g(\bx)$ and $G(\bx_{\mS_0})$ are sub-gaussian random variables with finite sub-gaussian norm $M_g$ and $M_G$ respectively, then $G(\bx_{\mS_0}) - g(\bx)$ is a sub-gaussian random variable and:

\begin{equation}
    \begin{aligned}
        \left\|G(\bx_{\mS_0}) - g(\bx)\right\|_{\psi_2} \le \left\|G(\bx_{\mS_0})\right\|_{\psi_2}+\left\|g(\bx)\right\|_{\psi_2} = M_G+M_g
    \end{aligned}
\end{equation}
    
\noindent According to Assumption \ref{a3-}, for any $ \bu\in \mathbb{R}^p$ s.t. $\Vert \bu \Vert_2 = 1$, $\bu^T T(\bx) \bu$ is a sub-exponential random variable and $\left\| \bu^T T(\bx) \bu\right\|_{\psi_1} \le M_T$ for some constant $M_T$. Then by Cauchy's inequality,

    \begin{align*}
        &\left| \bu^T\mathbb{E} \left[\left\{G(\bx_{\mS_0})-g(\bx)\right\}T(\bx)\right] \bu \right| \\
        \le& \left| \bu^T\mathbb{E}_{\bx_{\mS_0} \in [-m,m]^s} \left[\left\{G(\bx_{\mS_0})-g(\bx)\right\}T(\bx)\right] \bu \right|
        + \left| \bu^T\mathbb{E}_{\bx_{\mS_0} \notin [-m,m]^s} \left[\left\{G(\bx_{\mS_0})-g(\bx)\right\}T(\bx)\right] \bu \right| \\
        \le& C_1 n^{-1/2} \cdot \mathbb{E}_{\bx_{\mS_0} \in [-m,m]^s}\left| \bu^T T(\bx)\bu\right| + \left\{\mathbb{E} \left[\left\{G(\bx_{\mS_0})-g(\bx)\right\}\cdot\bu^TT(\bx)\bu\right]^2 \cdot \P\left(\bx_{\mS_0} \notin [-m,m]^s\right)\right\}^{1/2} \\
        \le& C_1 n^{-1/2} \cdot \mathbb{E}\left| \bu^T T(\bx)\bu\right| + \left\{\mathbb{E} \left[G(\bx_{\mS_0})-g(\bx)\right]^4 \cdot \mathbb{E}\left[\bu^TT(\bx)\bu\right]^4\right\}^{1/4} \cdot \left\{\P\left(\bx_{\mS_0} \notin [-m,m]^s\right)\right\}^{1/2} \\
        \le& C_1 n^{-1/2} \cdot\left\| \bu^T T(\bx) \bu\right\|_{\psi_1} + 2 \left\|G(\bx_{\mS_0})-g(\bx)\right\|_{\psi_2}\cdot 4\left\| \bu^T T(\bx) \bu\right\|_{\psi_1} \cdot \left\{\P\left(\bx_{\mS_0} \notin [-m,m]^s\right)\right\}^{1/2} \\
        \le& C_1M_Tn^{-1/2} + 8M_T(M_g+M_G)\cdot \prod\limits_{j\in\mS_0} \left\{\EE \left[|x_j|/m\right]\right\}^{1/2} \\
        \le& C_1M_Tn^{-1/2} + 8M_T(M_g+M_G)\cdot \prod\limits_{j\in\mS_0} \left\| x_j\right\|_{\psi_2}^{1/2} \cdot m^{-s/2} \\
        \le&C_1M_Tn^{-1/2} + 8M_T(M_g+M_G)\cdot m^{-s/2}\left\| \bx\right\|_{\psi_2}^{s/2}  \\
        =&\left\{C_1M_T+ 8M_T(M_g+M_G) \right\} n^{-1/2}  \refstepcounter{equation}\tag{\theequation}
    \end{align*}

\noindent  Take $C = C_1M_T+ 8M_T(M_g+M_G)$, we have:
\begin{equation}
    \begin{aligned}
        \left\|\mathbb{E} \left[\left\{G(\bx_{\mS_0})-g(\bx)\right\}T(\bx)\right]\right\|_2 = \mathop{\sup}_{\|\bu\|_2=1} \left| \bu^T\mathbb{E} \left[\left\{G(\bx_{\mS_0})-g(\bx)\right\}T(\bx)\right] \bu \right| \le Cn^{-1/2}.
    \end{aligned}
\end{equation}
    
\end{proof}

\subsection{Proof of Theorem \ref{thm2}}
\begin{lem}\label{lm10}
    Assume the conditions in Theorem \ref{thm2} hold, then for any $\nu>0$, when $n>\max\left\{ \left(\frac{1}{c_1^2}+\frac{c_2k_1^6\mu^4}{\lambda_{k_1}^2}\right)h_1^2(\nu,p), \frac{4M^2}{c_1\delta^2}h_1(\nu,p)\right\}$,  with probability at least $1- \nu$:
    \begin{align}\label{lm10-1}
        \|\bW - \hat\bW_1\|_\text{max} \le \frac{2M}{\sqrt{c_1n}}\left(\frac{k_1^4\mu^2}{\vert \lambda_{k_1}\vert}+ \frac{k_1^{3/2}\mu^{1/2}}{\delta\sqrt{p} }\right)\left(\sqrt{\log \frac{4}{\nu} + \log 9 \cdot p}+K\right)
    \end{align}
    where $h_1(\nu,p) = h(\nu,p)+K$ and $c_1, c_2, M, K$ are positive constants.
\end{lem}

\begin{proof}
    \quad Define matrix $\bE_2 = \frac{1}{n}\sum_{i=1}^{n} y_i T(\bx_i) - \EE g(\bx)T(\bx)$. By Lemma \ref{e2bound}, when $n > \frac{1}{c_1^2}\left(\log \frac{4}{\nu}+\right.\\
    \left.\log 9 \cdot p\right)^2 $, with probability at least $1-\nu$:
\begin{equation}
    \begin{aligned}
        \left\| \frac{1}{n}\sum_{i=1}^{n} y_i T(\bx_i) - \EE yT(\bx) \right\|_2  \le  \frac{2M_1}{\sqrt{c_1n}} \sqrt{\log \frac{4}{\nu} + \log 9 \cdot p}
    \end{aligned}
\end{equation}
    
\noindent  where $c_1, M_1>0$ are constants, combining with Proposition \ref{dnnprop}, we have:
\begin{equation}
    \begin{aligned}
        \left\| \bE_2 \right\|_2 &\le \left\| \frac{1}{n}\sum_{i=1}^{n} y_i T(\bx_i) - \EE yT(\bx) \right\|_2 + \left\| \mathbb{E} 
        \left[\left\{G(\bx_{\mS_0})-g(\bx)\right\}T(\bx)\right] \right\|_2 \\
        &\le \frac{2M_1}{\sqrt{c_1n}} \sqrt{\log \frac{4}{\nu} + \log 9 \cdot p} + \frac{M_2}{\sqrt{n}}
    \end{aligned}
\end{equation}
    
\noindent where $M_2>0$ is a constant, then from Lemma \ref{A2inf}, we can further obtain that:
\begin{equation}
    \begin{aligned}
        \left\| \bE_2 \right\|_\infty \le \sqrt{p} \left\| \bE_2 \right\|_2 \le \frac{2M_1\sqrt{p}}{\sqrt{c_1n}} \sqrt{\log \frac{4}{\nu} + \log 9 \cdot p} + \frac{M_2\sqrt{p}}{\sqrt{n}}
    \end{aligned}
\end{equation}
    
\noindent Let $K = \frac{\sqrt{c}M_2}{2M_1}$ and $h_1(\nu,p) = h(\nu,p)+K$, if $n>\max\left\{ \frac{c_2k_1^6\mu^4}{\lambda_{k_1}^2}h_1^2(\nu,p), \frac{4M_1^2}{c_1\delta^2}h_1(\nu,p)\right\}$ for some positive constant $c_2$, then $|\lambda_{k_1}| = \Omega\left(k_1^3\mu^2(\bW)\left\|\bE_2\right\|_\infty\right)$ and $\Vert \bE_2\Vert_2 < \delta$, 
    by the Theorem 3 in \cite{fan2017ellinftyeigenvectorperturbationbound}, we have:
    \begin{equation}
        \begin{aligned}
            \|\bW - \hat\bW_1\|_\text{max}&\le M_3 \left(\frac{k_1^4\mu^2(\bW)\Vert \bE_2\Vert_\infty}{\vert \lambda_{k_1}\vert \sqrt{p}}+ \frac{k_1^{3/2}\mu^{1/2}(\bW)\Vert \bE_2 \Vert_2}{\delta \sqrt{p}} \right) \\
            &\le \frac{2M_1M_3'}{\sqrt{c_1n}}\left (\frac{k_1^4\mu^2}{\vert \lambda_{k_1}\vert}+  \frac{k_1^{3/2}\mu^{1/2}}{\delta\sqrt{p} }\right)\left(\sqrt{\log \frac{4}{\nu} + \log 9 \cdot p}+K\right)
        \end{aligned}
    \end{equation}
    where $M_3, M_3'$ are positive constants. Therefore, let $M = \max\left\{M_1, M_1M_3'\right\}$, when $n>\max\left\{ \left(\frac{1}{c_1^2}+\frac{c_2k_1^6\mu^4}{\lambda_{k_1}^2}\right)h_1^2(\nu,p), \frac{4M^2}{c_1\delta^2}h_1(\nu,p)\right\}$, the inequality (\ref{lm10-1}) holds with probability at least $1-\nu$.
\end{proof}

\begin{thm}\label{thm2}
    Consider model (\ref{model1}) with $G\in\mathcal{G}$ with $\eps$ being Gaussian noise.  
    Let $g$ be a function in $\mathscr{G}_{\mathcal{S}_0}^{NN}$ such that (\ref{lm2-2}) holds. 
 Assume Assumptions \ref{a2new} to \ref{a3-} hold. 
 Let $\bW=\text{Eigen}_{k_1}\big(\mathbb{E}g(\bx)T(\bx)\big)$ and assume $\kappa \le (1/2) \mathop{\min}_{j\in \mathcal{S}_0} \|\{\bW\}_{\cdot j}\|_2$.
Suppose that $\bx$ is a sub-gaussian vector. Then, for any $\nu>0$, when $n \ge C \left(\log^2 \nu +p^2\right)$ for certain positive constant $C$, Algorithm \ref{alg1} achieves feature selection consistency with probability at least $1-\nu$, i.e.,
    \begin{align*}
        \P(\hat{\mS}_0 =\mS_0) \ge 1-\nu.
    \end{align*}

\end{thm}

\begin{proof}
    \quad  By Lemma \ref{eigenvectorbound}, when $n>\max\left\{ \left(\frac{1}{c_1^2}+\frac{c_2k_1^6\mu^4}{\lambda_{k_1}^2}\right)h_1^2(\nu,p), \frac{4M^2}{c_1\delta^2}h_1(\nu,p)\right\}$, with probability at least $1-\nu$:
    \begin{equation}
        \begin{aligned}
        \|\bW - \hat\bW_1\|_\text{max} \le \frac{2M}{\sqrt{c_1n}}\left(\frac{k_1^4\mu^2}{\vert \lambda_{k_1}\vert}+ \frac{k_1^{3/2}\mu^{1/2}}{\delta\sqrt{p} }\right)\left(\sqrt{\log \frac{4}{\nu} + \log 9 \cdot p}+K\right)
        \end{aligned}
    \end{equation}
    where $h_1(\nu,p) = h(\nu, p) + K$ and $c_1,c_2,M,K$ are positive constants.
    
    \noindent  Then for $j\in \hat\mS_0$, when $n > \frac{4M^2}{c_1\kappa^2}\left(\frac{k_1^{9/2}\mu^2}{\vert \lambda_{k_1}\vert}+\frac{k_1^2\mu^{1/2}}{\delta \sqrt{p}}\right)^2\left(\sqrt{\log \frac{4}{\nu} + \log 9 \cdot p}+K\right)^2$, we have:

    \begin{align*}
        \|\bW_{\cdot j}\|_2 &\ge \|\{\hat\bW_1\}_{\cdot j}\|_2 - \sqrt{k_1}\Vert \bW-\hat\bW_1\Vert_\text{max} \\
        &\ge \kappa - \sqrt{k_1}\Vert \bW-\hat\bW_1 \Vert_\text{max} \\
        &\ge \kappa - \sqrt{k_1} \cdot\frac{2M}{\sqrt{c_1n}}\left(\frac{k_1^4\mu^2}{\vert \lambda_{k_1}\vert}+ \frac{k_1^{3/2}\mu^{1/2}}{\delta\sqrt{p} }\right)\left(\sqrt{\log \frac{4}{\nu} + \log 9 \cdot p}+K\right)\\
        &> 0   \refstepcounter{equation}\tag{\theequation}
    \end{align*}

   \noindent then $j\in \mS_0$. Furthermore, $ \hat\mS_0 \subseteq \mS_0$. On the other hand, let $j \in \mS_0$, then we have:
        \begin{align*}
        \|\hat\bW_{\cdot j}\|_2 &\ge \|\bW_{\cdot j}\|_2 - \sqrt{k_1}\|\bW - \hat\bW_1\|_\text{max} \\
        &\ge \min_{j \in \mS_0} \|\bW_{\cdot j}\|_2 - \sqrt{k_1} \cdot \frac{2M}{\sqrt{c_1n}}\left(\frac{k_1^4\mu^2}{\vert \lambda_{k_1}\vert}+ \frac{k_1^{3/2}\mu^{1/2}}{\delta\sqrt{p} }\right)\left(\sqrt{\log \frac{4}{\nu} + \log 9 \cdot p}+K\right)\\
        &\ge \min_{j \in \mS_0} \|\bW_{\cdot j}\|_2 - \kappa \\
        &\ge \kappa
        \refstepcounter{equation}\tag{\theequation}
    \end{align*}
then $j \in \hat\mS_0$. Furthermore, $\mS_0 \subseteq \hat\mS_0$. Combining the above two conclusions, under the condition that  $n \ge C\left(\log^2 \nu + p^2\right)$ for certain positive constant $C$,  with probability at least $1-\nu$, we have $\mS_0 = \hat\mS_0$.

\end{proof}

\section{Proof of Theorems in Section \ref{sec_high}}
\subsection{Proof of Lemma \ref{lem:fake2}:}

\refstepcounter{fakelemma}
\noindent \textbf{Lemma }2.\label{lem:fake2} 
 Consider model  (\ref{model2}) with Gaussian noise $\eps$. Suppose Assumptions \ref{min2rd}-\ref{submatrix} hold. Let $\zeta\in(0,1)$  denote the screening proportion, and let $ \hat{\mS}_0^\zeta$ be as in (\ref{s0zeta}). If $\zeta = cn^{-1/3}$, 
   we have
    \bel{lm1-2}
    \P\left(\mS_0\subset \hat{\mS}_0^\zeta\right) \geq 1 -  s\exp\left\{-C n^{2/3}\right\},
    \eel
where $c$ and $C$ are certain constants.

\begin{proof}
    \quad On the one hand, according to the Stein's formula and assume Assumption \ref{min2rd} holds, for $k \in \mS_0$ there exits  a  constant $c_1$ so that:
    \begin{align}
        | \bA_{kk}| =  \EE \left|\frac{\partial^2 g(\bx)}{\partial x_k^2} \right|\ge c_1, \ \ g\in \mathscr{G}_{\mathcal{S}_0}^{NN},\  k \in\mS_0
    \end{align}
    According to Assumption \ref{a3-}, $T_{jk}(\bx)$ is a sub-exponential random variable and suppose that $\mathop{\sup}_{1\le j,k\le p} \left\| T_{jk}(\bx) \right\|_{\psi_1} \le M_1$ for some constant $M_1$. Since $\bx$ is a sub-gaussian random vector, by Lemma \ref{lmrelu}, $g(\bx)$ is also a sub-gaussian random variable. Combining with $\epsilon$ is a gaussian noise, then $y$ is a sub-gaussian variable with $\|y\|_{\psi_2} \leq M_2$. Let $M = M_1 M_2$, according to Lemma \ref{lm8new}, there exists  constants $d_1, d_2 >0$ so that when $n \geq 2d_2M/ c_1$, 
    \begin{equation}
    \begin{aligned}
    \P \left( \mathop{\sup}_{k\in\mS_0} \left| \hat{\bA}_{kk}-\bA_{kk}\right| \ge c_1/2  \right)  &=   
    \P\left(\mathop{\sup}_{k\in\mS_0}\left|\frac1n\sum\limits_{i=1}^n \left[ y_iT_{kk}(\bx_i) - \EE(yT_{kk}(\bx) )
    \right]\right|  \ge c_1/2 \right) \\ &\le 4s\exp\left\{-d_1\mathop{\min}\left[ \left(\frac{c_1\sqrt{n}}{2M}\right)^2, \left(\frac{c_1n}{2M}\right)^{2/3}\right]\right\}
    \end{aligned}
    \end{equation}

    \noindent Note that for $k \in \mS_0$, we have: 
    \begin{align}
        |\hat{\bA}_{kk}| \ge |\bA_{kk}| - |\hat{\bA}_{kk}-\bA_{kk}| \ge c_1  - |\hat{\bA}_{kk}-\bA_{kk}|
    \end{align}
    therefore, 
     \begin{equation}
        \begin{aligned}
            \P \left( \mathop{\sup}_{k\in\mS_0} \left| \hat{\bA}_{kk}\right| \ge c_1/2  \right)  &\ge  \P \left( \mathop{\sup}_{k\in\mS_0} \left| \hat{\bA}_{kk}-\bA_{kk}\right| \le  c_1/2  \right) \\
            &\ge 1 - 4s\exp\left\{-d_1\mathop{\min}\left[ \left(\frac{c_1\sqrt{n}}{2M}\right)^2, \left(\frac{c_1n}{2M}\right)^{2/3}\right]\right\}
        \end{aligned}
    \end{equation}

\noindent  On the other hand, we consider,
    \begin{equation}\label{thm3pf-2}
        \begin{aligned}
            \left\| \diag\left\{\hat{\bA}\right\}\right\|_2^2 &= \left\| \diag\left\{\frac1n\sum\limits_{i=1}^n y_iT(\bx_i)\right\}\right\|_2^2 \\
            &= \left\| \frac1n \by^T \bV_{[p]} \right\|_2^2 \\
            &= \frac{1}{n^2} \by^T\bV_{[p]}\bV_{[p]}^T\by \\
            &\le \frac{1}{n^2} \lambda_{\text{max}}(\bV_{[p]}\bV_{[p]}^T) \by^T\by
        \end{aligned}
    \end{equation}
    where $\by = (y_1,\cdots,y_n)^T$ and $[p] = \{1, 2, \cdots,p\}$. Since $y$ is sub-gaussian random variable, by Lemma \ref{subexpinequality}, $y^2$ is sub-exponential random variable with:
    \begin{align}
        \left\|y^2\right\|_{\psi_1} \le 2\left\|y\right\|_{\psi_2}^2 \le 2M_2^2
    \end{align}

\noindent According to Bernstein's inequality, for some constant $d_3>0$,
    \begin{equation}\label{thm3pf-3}
        \begin{aligned}
            \P\left(\frac1n \by^T\by > \EE y^2+1\right) = \P\left(\frac1n \sum\limits_{i=1}^n y_i^2 - \EE y^2> 1\right) \le \exp\left\{-d_3\mathop{\min}\left(\frac{1}{4M_2^2},\frac{1}{2M_2}\right)n\right\}
        \end{aligned}
    \end{equation}

\noindent Combining with Assumption \ref{submatrix}, for some constants $c_2,c_3>0$, we have:
\begin{equation}\label{thm3pf-4}
        \begin{aligned}
        \P\left(\left\| \diag\left\{\hat{\bA}\right\}\right\|_2^2 > \frac{c_2p}{n}\left(\EE y^2+1\right)\right) & \le \P\left(\frac{1}{n^2} \lambda_{\text{max}}\left(\bV_{[p]}\bV_{[p]}^T\right) \by^T\by > \frac{c_2p}{n}\left(\EE y^2+1\right)\right) \\
        &\le \P\left(\lambda_{\text{max}}\left(p^{-1}\bV_{[p]}\bV_{[p]}^T\right) > c_2\right) + \P\left(\frac{1}{n}\by^T\by> \EE y^2+1\right) \\
        &\le \exp\left\{-c_3n\right\} + \exp\left\{-d_3\mathop{\min}\left(\frac{1}{4M_2^2},\frac{1}{2M_2}\right)n\right\}
    \end{aligned}
\end{equation}

\noindent therefore, let $c_4 = \min\left\{c_3,\frac{d_3}{4M_2^2},\frac{d_3}{2M_2}\right\}$, we can obtain that,
    \bel{thm3pf-5}
    \P\left(\left\| \diag\left\{\hat{\bA}\right\}\right\|_2^2 \le \frac{c_2p}{n}\left(\EE y^2+1\right)\right) \ge 1 - 2e^{-c_4n}
    \eel
\noindent Further let $C_1 = \min\left\{\frac{c_1^2d_1}{4M^2},\frac{c_1^{2/3}d_1}{(2M)^{2/3}}, c_4\right\}$, we have:
    \bel{thm3pf-6}
    \P\left(\left\| \diag\left\{\hat{\bA}\right\}\right\|_2^2 \le \frac{c_2p}{n}\left(\EE y^2+1\right) \text{ and } \mathop{\sup}_{k\in\mS_0} \left| \hat{\bA}_{kk}\right| \ge \frac{c_1}{2}\right) \ge 1 - (4s+2)\exp\left\{-C_1n^{\frac{2}{3}}\right\}
    \eel
\noindent In other words, with probability at least $1- (4s+2)\exp\left\{-C_1n^{2/3}\right\}$, 
    \begin{align}
            \text{Card} \left\{ 1\le j \le p: \left|\hat{\bA}_{jj}\right| \ge \mathop{\min}_{k\in\mS_0} \left|\hat{\bA}_{kk}\right|\right\} \le \frac{4}{c_1^2}\left\| \diag\left\{\hat{\bA}\right\}\right\|_2^2 \le \frac{4c_2p}{c_1^2n}\left(\EE y^2+1\right)
    \end{align}
\noindent where $\text{Card}\left\{\cdot\right\}$ is the cardinality of a set. Therefore, when $\zeta = cn^{-1/3} \ge \frac{4c_2}{c_1^2n}\left(\EE y^2+1\right)$, we have:
    \bel{thm3pf-7}
    \P\left(\mS_0\subset\hat{\mS}_0^\zeta\right) \ge 1- (4s+2)\exp\left\{-C_1 n^{2/3}\right\} \geq 1 - s\exp\left\{-C n^{2/3}\right\}
    \eel
\end{proof}

\subsection{Proof of Theorem \ref{algo2consistency}}
\begin{lem}\label{lemmascreening}
    Under conditions of Lemma \ref{lemmascreening1step} with $\zeta = c_1 n^{-1/3}$, $p_0 = c_2 n^{1/3}$ and  $\log p \leq c_3 n /s$ for some positive constant $c_1, c_2, c_3$. Starting from $p$ features, let $\hat{\mS}_{0,R}^\zeta$ be the final index set after $R$ rounds of iterative screening steps. Then, for some $C>0$, we have:
    \begin{equation}
                \mathbb{P}\left(\mS_0 \subset \hat{\mS}_{0, R}^\zeta\right) \geq 1 - \exp\left\{-Cn^{2/3} \right\}
    \end{equation}

\end{lem}

\begin{proof}
    \quad Denote $\hat{\mS}_{0,r}^\zeta$ to be the selected index set after $r$ screening rounds, according to Lemma \ref{lemmascreening1step}, we have: 
    \begin{align}
        \P\left(\mS_0\subset \hat{\mS}_{0,1}^\zeta\right) \geq 1 - s\exp\left\{-C'n^{2/3}\right\}
    \end{align}
    \noindent Since Assumption $\ref{submatrix}$ holds for any index subset $\mathcal{I}$ satisfying $ cn^{1/3} < \text{Card}\{\mathcal{I}\} \leq p$, therefore: 
    \begin{align}
        \P\left(\mS_0\subset \hat{\mS}_{0,r}^\zeta   | \mS_0\subset \hat{\mS}_{0,r-1}^\zeta \right) \geq 1 - s\exp\left\{-C'n^{2/3}\right\}
    \end{align}
    According to Bonferroni's inequality, we could obtain: 
    \begin{align}
         \P\left(\mS_0\subset \hat{\mS}_{0,r}^\zeta \right) \geq 1 -  r\cdot s\exp\left\{-C'n^{2/3}\right\}, \ 1\leq r\leq R
    \end{align}
    \noindent Since
    \begin{align}
         \text{Card}\{\hat{\mS}_{0, R-1}^\zeta\} = \zeta^{R-1}p  \geq  p_0 = c_2 n^{1/3}
    \end{align}
    and $\zeta = c_1 n^{-1/3}$, then there exists $c_4 > 0$ so that
    \begin{align}
        R \leq  c_4 \frac{\log p }{\log n} \leq c_3 c_4 \frac{n}{s \log n}
    \end{align}
 As such, a slightly increase of the constant $C'$ will give us: 
    \begin{align}
        \P\left(\mS_0 \subset \hat{\mS}_{0,R}^\zeta\right) \geq 1 - \exp\left\{-Cn^{2/3}\right\}
    \end{align}
\end{proof}

\begin{thm}\label{algo2consistency}
    Assume the conditions in Theorem \ref{thm1}. Further assume Assumptions \ref{min2rd} - \ref{submatrix} hold and $\bx$ is a sub-gaussian vector. Suppose screening level $\zeta  = c_1 n^{-1/3}$ and the target dimension $p_0 = c_2 n^{1/3}$ for some positive constant $c_1, c_2$. Then, for any $\nu>0$, when
  $n \ge C\left(\log^{2} \nu + s \log p \right)$ for certain positive constant $C$, Algorithm \ref{algo2} achieves feature selection consistency with probability at least $1-\nu$, i.e., 
\bes
\P(\hat{\mS}_0 =\mS_0) \ge 1-\nu.
\ees
\end{thm}

\begin{proof}
\quad According to Lemma \ref{lemmascreening}, when $\log p \leq c n /s$ and $n \geq C_1 \log^{3/2} (2/\nu) $, we have: 
    \begin{align}
        \P\left(\mS_0 \subset \hat{\mS}_{0}^\zeta\right)= 1 - \frac{\nu}{2}
    \end{align}

\noindent Since the cardinality of the final index set after $R$ rounds of screening is  $\text{Card}\{\hat{\mS}_{0}^\zeta\} \leq c_2 n^{1/3}$. Combining the results from Theorem \ref{thm1}, when $n \geq C_2 \log^2 \frac{2}{\nu}$ for certain constant $C_2$, we have: 
\begin{align}
    \P\left(\hat{\mS_0} = \mS_0 |\mS_0 \subset \hat{\mS}_{0}^\zeta\right) \geq 1 - \frac{\nu}{2}
\end{align}
\noindent Basing on above, we could obtain: 
\begin{equation}
    \begin{aligned}
        \P\left(\hat{\mS_0} = \mS_0 \right)  & =   \P\left(\hat{\mS_0} = \mS_0 |\mS_0 \subset \hat{\mS}_{0}^\zeta\right) \P\left(\mS_0 \subset \hat{\mS}_{0}^\zeta\right) \\
        & \geq 1 - \nu
    \end{aligned}
\end{equation}

\noindent Combining all the results above, when $n \geq C \left(\log^2{\nu}  + s \log p\right) $ for certain constant $C$, Algorithm 2 can achieve feature consistency with probability at least $1- \nu$, i.e, 
\bes
\P(\hat{\mS}_0 =\mS_0) \ge 1-\nu.
\ees
\end{proof}

\section{Proof of Theorem \ref{nsthm1}}
According to the main text, in this section, the response $y$ could be stated as: 
\begin{align}\label{nsmodel1} 
    y = f(\bW_1 \bx) + \epsilon, \ \bW_1 \in \RR^{k_1 \times p}.   
\end{align}
where $f$ is neural network function with first-layer weight $\bW_1$ being approximately sparse.

\begin{lem}\label{nslm1}
    Let $f: \mathbb{R}^{s}\to \mathbb{R}$ be any Lipschitz function. Suppose that $\bx\in \mathbb{R}^s$ is sub-Gaussian. For given integers $w,d>0$, there exists a function $g$ implemented by ReLU FNN with width $3^{s+3}\max\left\{s\lfloor w^{1/s} \rfloor, w+1\right\}$ and depth $12d +14 + 2s$ such that:
    \bel{lm7-1-1}
    \EE_{\bx} \left( f\left(\bx \right) - g\left(\bx\right) \right)^2 \le Cs(wd)^{-4/(s+4)}
    \eel
    where $C$ is a certain positive constant.
\end{lem}

\begin{proof}
 \quad   For any $R>0$, by Theorem 4.3 of \cite{shen2019deep}, there exists a ReLU neural network function $g$ with width width $3^{s+3}\max\left\{s\lfloor w^{1/s} \rfloor, w+1\right\}$ and depth $12d +14 + 2s$ such that:
   \begin{align}
           \mathop{\sup}_{\bx\in [-R,R]^s} \left|f(\bx)-g(\bx)\right|\le 19\sqrt{s}\omega_f^{[-R,R]^s}\left(2Rw^{-2/s}d^{-2/s}\right)
   \end{align} 

    \noindent Since $f$ is Lipschitz function, there exists $L>0$ such that:
\begin{align}
        \omega_f^{[-R,R]^s}\left(2Rw^{-2/s}d^{-2/s}\right) \le L \cdot2Rw^{-2/s}d^{-2/s}
\end{align}
    then,
    \begin{align}
            \mathop{\sup}_{\bx\in [-R,R]^s} \left|f(\bx)-g(\bx)\right|\le 38L\sqrt{s}Rw^{-2/s}d^{-2/s}
    \end{align}

    \noindent Let $M = \sup_{1\le j \le p} \| x_j\|_{\psi_2}$, by Lemma \ref{lmrelu}, $f(\bx)$ and $g(\bx)$ are sub-Gaussian random variables, moreover, the sub-Gaussian norm of $f(\bx)-g(\bx)$ satisfies that:
    \begin{align}
            \left\|f(\bx)-g(\bx)\right\|_{\psi_2} \le C_1\sqrt{s}M
    \end{align}
where $C_1>0$ is a certain constant.
    \noindent Then by Cauchy's inequality,
    \begin{equation}
        \begin{aligned}
            \EE_{\bx} \left( f\left(\bx \right) - g\left(\bx\right) \right)^2 &= \EE_{\bx} \left[\left( f\left(\bx \right) - g\left(\bx\right) \right)^2\cdot\left(I_{\bx\in[-R,R]^s}(\bx) + I_{\bx\notin[-R,R]^s}(\bx)\right)\right] \\ 
            &\le \left(38L\sqrt{s}Rw^{-2/s}d^{-2/s}\right)^2 + \left[\EE_{\bx} \left( f\left(\bx \right) - g\left(\bx\right) \right)^4 \cdot \P\left(\bx\notin [-R,R]^s\right)\right]^{1/2} \\
            &\le \left(38L\sqrt{s}Rw^{-2/s}d^{-2/s}\right)^2 + \left(2\left\|f\left(\bx \right) - g\left(\bx\right)\right\|_{\psi_2}\right)^2\cdot \prod\limits_{j=1}^s \left[\EE \left|x_j\right|/R\right]^{1/2} \\
            &\le \left(38L\sqrt{s}Rw^{-2/s}d^{-2/s}\right)^2 + \left(2C_1\sqrt{s}M\right)^2\cdot \prod\limits_{j=1}^s \left\|x_j\right\|_{\psi_2}^{1/2} \cdot R^{-s/2} \\
            &\le \left(38L\sqrt{s}Rw^{-2/s}d^{-2/s}\right)^2 + \left(2C_1\sqrt{s}M\right)^2\cdot M^{s/2} \cdot R^{-s/2}
        \end{aligned}
    \end{equation}
    take $R = M\left(wd\right)^{8/s(s+4)}$ and let $ C = \left(38LM\right)^2+ \left(2C_1M\right)^2$, we have:
    \begin{align}
            \EE_{\bx} \left( f\left(\bx \right) - g\left(\bx\right) \right)^2 \le Cs\left(wd\right)^{-4/(s+4)}
    \end{align}
\end{proof}

\begin{lem}\label{nslm1.5}
    Let $\bz_i(1\le i\le n)$ be the independent random variables and $\eps_i(1\le i\le n)$ be the Rademacher random variables. Suppose that the constant $K>0$ and the function space $\mathcal{G}$ satisfy that for any $g\in \mathcal{G}$ and $1\le i \le n$,  $g(\bz_i)$ is a sub-Exponential random variable and $\| g(\bz_i)\|_{\psi_1} \le K$. Then, when $n >\log\left(\left|\mathcal{G}\right|\right)$, we have:
    \bel{nslm1.5-1}
    \EE_{\left\{\bz_i,\eps_i\right\}_{i=1}^n}  \mathop{\sup}_{g\in\mathcal{G}}\frac1n \sum\limits_{i=1}^n \eps_i g(\bz_i) \le \sqrt{\frac{8e^2K^2\log \left(\left|\mathcal{G}\right|\right)}{n}}
    \eel
\end{lem}
\begin{proof}
\quad  For any $t>0$, by Jensen's inequality, we have:

        \begin{align*}
                \EE_{\left\{\bz_i,\eps_i\right\}_{i=1}^n}  \mathop{\sup}_{g\in\mathcal{G}}\frac1n \sum\limits_{i=1}^n \eps_i g(\bz_i) &= \frac{1}{nt} \EE_{\left\{\bz_i,\eps_i\right\}_{i=1}^n}  \mathop{\sup}_{g\in\mathcal{G}}\sum\limits_{i=1}^n t\eps_i g(\bz_i)\\
                &\le \frac{1}{nt} \log \left(\EE_{\left\{\bz_i,\eps_i\right\}_{i=1}^n} \exp\left\{\mathop{\sup}_{g\in\mathcal{G}}\sum\limits_{i=1}^n t\eps_i g(\bz_i)\right\}\right) \\
                &\le \frac{1}{nt} \log \left(\EE_{\left\{\bz_i,\eps_i\right\}_{i=1}^n} \sum\limits_{g\in\mathcal{G}} \exp\left\{\sum\limits_{i=1}^n t\eps_i g(\bz_i)\right\}\right) \\
                &= \frac{1}{nt} \log \left(\sum\limits_{g\in\mathcal{G}}\EE_{\left\{\bz_i,\eps_i\right\}_{i=1}^n}  \prod\limits_{i=1}^ne^{ t\eps_i g(\bz_i)}\right) \\
                &= \frac{1}{nt} \log \left(\sum\limits_{g\in\mathcal{G}}\prod\limits_{i=1}^n\EE  e^{ t\eps_i g(\bz_i)}\right)
                \refstepcounter{equation}\tag{\theequation}
        \end{align*}
note that when $t \le 1/(\sqrt{2}eK)$, for any $1\le i \le n$, we have:
        \begin{align*}
            \EE  e^{ t\eps_i g(\bz_i)} &= \EE \sum\limits_{k=0}^\infty \frac{t^{2k} g^{2k}(\bz_i)}{(2k)!} \le 1+ \sum\limits_{k=1}^\infty \frac{t^{2k} \left(2k \left\|g(\bz_i)\right\|_{\psi_1}\right)^{2k}}{(2k)!} \\
            &\le 1+ \sum\limits_{k=1}^\infty \frac{t^{2k} \left(2k K\right)^{2k}}{\left(2k/e\right)^{2k}} = \sum\limits_{k=0}^\infty \left(eKt\right)^{2k} = \frac{1}{1- (eKt)^2} \le e^{2(eKt)^2}  \refstepcounter{equation}\tag{\theequation}
        \end{align*}
    then,
    \begin{equation}
        \begin{aligned}
            \EE_{\left\{\bz_i,\eps_i\right\}_{i=1}^n}  \mathop{\sup}_{g\in\mathcal{G}}\frac1n \sum\limits_{i=1}^n \eps_i g(\bz_i) &\le \frac{1}{nt} \log \left(\sum\limits_{g\in\mathcal{G}}\prod\limits_{i=1}^n\EE  e^{ t\eps_i g(\bz_i)}\right) \\
            &\le \frac{1}{nt} \log \left(\left|\mathcal{G}\right|e^{ 2n(eKt)^2}\right) \\
            &=\frac{\log\left(\left|\mathcal{G}\right|\right)}{nt} + 2e^2K^2t
        \end{aligned}
    \end{equation}
    holds for any $t>0$. Since $n >\log\left(\left|\mathcal{G}\right|\right)$, let $t = \sqrt{\frac{\log\left(\left|\mathcal{G}\right|\right)}{2ne^2K^2}} \le \frac{1}{\sqrt{2}eK}$, we have:
    \begin{align}
            \EE_{\left\{\bz_i,\eps_i\right\}_{i=1}^n}  \mathop{\sup}_{g\in\mathcal{G}}\frac1n \sum\limits_{i=1}^n \eps_i g(\bz_i) \le \sqrt{\frac{8e^2K^2\log \left(\left|\mathcal{G}\right|\right)}{n}}
    \end{align}
\end{proof}
\begin{lem}\label{nslm2}
    Let
\begin{equation}\label{nslm2-1}
    \begin{aligned}
        &\mathcal{L}(g) = \EE_{(\bx,y)} \left( y - g\left(\bx_{\hat{\mS}}\right) \right)^2 \\
        &\hat{\mathcal{L}}(g) = \frac1n \sum\limits_{i=1}^n \left(y_i - g\left(\left\{\bx_i\right\}_{\hat{\mS}}\right)\right)^2
    \end{aligned}
\end{equation}
then for function space $\mathcal{G}$, when $n > c \text{Pdim}_{\mathcal{G}}$ for positive constant $c$, then we have:
\bel{nslm2-2}
\EE_{\left\{\bx_i,y_i\right\}_{i=1}^n} \mathop{\sup}_{g\in\mathcal{G}}\left|\mathcal{L}(g) - \hat{\mathcal{L}}(g)\right| \le Cs\sqrt{\frac{\text{Pdim}_{\mathcal{G}}\log \left(B_\mathcal{G}n\right)}{n}}
\eel
where $C > 0$ is a  positive constant,  $\text{Pdim}_{\mathcal{G}}$ is the Pseudo dimension of $\mathcal{G}$ and $B_\mathcal{G}= \sup_{g\in\mathcal{G}} \left\|g\right\|_\infty$ is the global bound of $\mathcal{G}$.
\end{lem}
\begin{proof}
\quad  Let $\left\{\bx_i',y_i'\right\}_{i=1}^n$ be the independent duplicates of $\left\{\bx_i,y_i\right\}_{i=1}^n$, let $\eps_i$ be i.i.d Rademacher random variables that are independent of $\bx_i,y_i,\bx_i',y_i', 1\le i\le n$. Note that:
    \begin{align}
            \mathcal{L}(g) = \EE_{\left\{\bx_i',y_i'\right\}_{i=1}^n} \frac1n\sum\limits_{i=1}^n\left( y_i' - g\left(\left\{\bx_i'\right\}_{\hat{\mS}}\right) \right)^2
    \end{align}
    then by Jensen's inequality, we have:
    \begin{equation}
        \begin{aligned}
            &\EE_{\left\{\bx_i,y_i\right\}_{i=1}^n} \mathop{\sup}_{g\in\mathcal{G}}\left|\mathcal{L}(g) - \hat{\mathcal{L}}(g)\right| \\
            =& \EE_{\left\{\bx_i,y_i\right\}_{i=1}^n}\mathop{\sup}_{g\in\mathcal{G}}\left|\EE_{\left\{\bx_i',y_i'\right\}_{i=1}^n} \frac1n\sum\limits_{i=1}^n\left( y_i' - g\left(\left\{\bx_i'\right\}_{\hat{\mS}}\right) \right)^2 - \hat{\mathcal{L}}(g)\right| \\
            \le&\EE_{\left\{\bx_i,y_i,\bx_i',y_i'\right\}_{i=1}^n}\mathop{\sup}_{g\in\mathcal{G}}\left| \frac1n\sum\limits_{i=1}^n\left( y_i' - g\left(\left\{\bx_i'\right\}_{\hat{\mS}}\right) \right)^2 - \frac1n\sum\limits_{i=1}^n\left( y_i - g\left(\left\{\bx_i\right\}_{\hat{\mS}}\right) \right)^2\right| \\
            =& \EE_{\left\{\bx_i,y_i,\bx_i',y_i',\eps_i\right\}_{i=1}^n}\mathop{\sup}_{g\in\mathcal{G}}\left|\frac1n\sum\limits_{i=1}^n\eps_i\left[\left( y_i' - g\left(\left\{\bx_i'\right\}_{\hat{\mS}}\right) \right)^2- \left( y_i - g\left(\left\{\bx_i\right\}_{\hat{\mS}}\right) \right)^2\right]\right|\\
            \le&2\EE_{\left\{\bx_i,y_i,\eps_i\right\}_{i=1}^n} \mathop{\sup}_{g\in\mathcal{G}}\left|\frac1n\sum\limits_{i=1}^n\eps_i\left( y_i - g\left(\left\{\bx_i\right\}_{\hat{\mS}}\right) \right)^2\right| \\
            \triangleq& 2\mathcal{R}(\mathcal{G})
        \end{aligned}
    \end{equation}
\noindent   According to Lemma \ref{subexpinequality} and Lemma \ref{lmrelu}, for $1\le i\le n$ and $g\in \mathcal{G}$, $\left(y_i-g\left(\left\{\bx_i\right\}_{\hat{\mS}}\right)\right)^2$ is a sub-Exponential random variable and there exists a constant $M>0$ such that:
    \begin{equation}
        \begin{aligned}
                \left\|\left(y_i-g\left(\left\{\bx_i\right\}_{\hat{\mS}}\right)\right)^2\right\|_{\psi_1} &\le 2\left\|y_i-g\left(\left\{\bx_i\right\}_{\hat{\mS}}\right)\right\|_{\psi_2}^2 \\
                &\le 2\left(\left\|y_i\right\|_{\psi_2}+\left\|g\left(\left\{\bx_i\right\}_{\hat{\mS}}\right)\right\|_{\psi_2}\right)^2\\
                &\le 2\left( \sqrt{s} M +B_{\mathcal{G}}\right)^2
        \end{aligned}
    \end{equation}
    in addition, for $g_1, g_2\in \mathcal{G}, \|g_1-g_2\|_\infty \le \delta$ and $1\le i\le n$:
        \begin{align*}
            &\left|\left(y_i-g_1\left(\left\{\bx_i\right\}_{\hat{\mS}}\right)\right)^2-\left(y_i-g_2\left(\left\{\bx_i\right\}_{\hat{\mS}}\right)\right)^2\right| \\
            =& \left|g_1\left(\left\{\bx_i\right\}_{\hat{\mS}}\right) - g_2\left(\left\{\bx_i\right\}_{\hat{\mS}}\right)\right|\cdot \left|2y_i - g_1\left(\left\{\bx_i\right\}_{\hat{\mS}}\right) - g_2\left(\left\{\bx_i\right\}_{\hat{\mS}}\right)\right| \\
            \le & 2\delta\left(\left|y_i\right|+ B_{\mathcal{G}}\right) \refstepcounter{equation}\tag{\theequation}
        \end{align*}
    then for any $\delta>0$, let $\mathcal{G}_\delta$ be the $\delta$-net of $\mathcal{G}$, by Lemma \ref{nslm1.5}, when $n> \log(2 |\mathcal{G_{\delta}}|)$, we have:
        \begin{align*}
            \mathcal{R}(\mathcal{G})&=\EE_{\left\{\bx_i,y_i,\eps_i\right\}_{i=1}^n} \mathop{\sup}_{g\in\mathcal{G}}\left|\frac1n\sum\limits_{i=1}^n\eps_i\left( y_i - g\left(\left\{\bx_i\right\}_{\hat{\mS}}\right) \right)^2\right| \\
            &\le 2\delta\left(\sqrt{s}M+B_{\mathcal{G}}\right) + \EE_{\left\{\bx_i,y_i,\eps_i\right\}_{i=1}^n} \mathop{\sup}_{g\in \mathcal{G}_\delta} \left|\frac1n\sum\limits_{i=1}^n\eps_i\left( y_i - g\left(\left\{\bx_i\right\}_{\hat{\mS}}\right) \right)^2\right| \\
            &\le 2\delta\left(\sqrt{s}M+B_{\mathcal{G}}\right) + \sqrt{\frac{8e^2(\sqrt{s}M+B_{\mathcal{G}})^4\log\left(2\left|\mathcal{G}_\delta\right|\right)}{n}} \refstepcounter{equation}\tag{\theequation}
        \end{align*}
    by Theorem 12.2 in \cite{Anthony_Bartlett_1999}, we have:
    \begin{align}
            \log \left|\mathcal{G}_\delta\right| \le \text{Pdim}_\mathcal{G}\log\frac{2e B_{\mathcal{G}}n}{\delta\text{Pdim}_{\mathcal{G}}}
    \end{align}
therefore, let $\delta = n^{-1}$ and $C = 16e(M+B_{\mathcal{G}})^2$, when $n > c \text{Pdim}_{\mathcal{G}} $ for certain positive constant $c$, then we have $n > \log (2|\mathcal{G}_{\delta}|)$ so that:
    \begin{equation}
        \begin{aligned}
            \EE_{\left\{\bx_i,y_i\right\}_{i=1}^n} \mathop{\sup}_{g\in\mathcal{G}}\left|\mathcal{L}(g) - \hat{\mathcal{L}}(g)\right| &\le 2\mathcal{R}(\mathcal{G}) \\
            &\le 4\delta(\sqrt{s}M+B_{\mathcal{G}}) + 2 \sqrt{\frac{8e^2(\sqrt{s}M+B_{\mathcal{G}})^4\log\left(2\left|\mathcal{G}_\delta\right|\right)}{n}} \\
            &\le \frac{4(\sqrt{s}M+B_{\mathcal{G}})}{n} + 2 \sqrt{\frac{8e^2(\sqrt{s}M+B_{\mathcal{G}})^4\left(\log2 + \text{Pdim}_\mathcal{G}\log\left(2eB_{\mathcal{G}} n^2\right)\right)}{n}} \\
            &\le 4\sqrt{\frac{16e^2(\sqrt{s}M+B_{\mathcal{G}})^4\text{Pdim}_\mathcal{G}\log\left(B_{\mathcal{G}}n\right)}{n}} \\
            &\leq Cs\sqrt{\frac{\text{Pdim}_\mathcal{G}\log\left(B_{\mathcal{G}}n\right)}{n}}
        \end{aligned}
    \end{equation}
\end{proof}

\begin{lem}\label{nslm3}
    Suppose that the ReLU neural network space $\mathcal{G}$ is the neural network space. When $n\ge C_1\left(\log^2 \nu + p^2\right)$, we have:
    \bel{nslm3-1}
        \mathop{\inf}_{\bar{g}\in \mathcal{G}} \EE_{\hat{\mS}} \mathcal{L}(\bar{g}) \le C\left( s(W_{\mathcal{G}} D_{\mathcal{G}})^{-4/(s+4)} + s\nu + \eps_0^2\right)
    \eel
    where $W_{\mathcal{G}}$ and $W_{\mathcal{G}}$ are width and depth parameter of $\mathcal{G}$ and $C,C_1>0$ are certain constants.
\end{lem}

\begin{proof}
\quad   Let $\mS\subset [p], \left|\mS\right|=s_0 \leq s$ such that:
\begin{align}
        \sum\limits_{j\notin \mS} \eta_j \le \eps_0
\end{align}
and for certain positive constant $c$, 
\begin{align}
    \kappa \leq c \min_{j \in \mS} || \{\bW_1\}_{\cdot j} ||_2
\end{align}
then by Theorem \ref{thm1}, for any $\nu>0$, when $n\ge C_1\left(\log^2 \nu + p^2\right)$ for some constant $C_1$, we have:
\begin{align}
        \P \left(\hat{\mS} = \mS\right)  \ge 1- \nu
\end{align}
in addition, by Lemma \ref{nslm1}, there exists a function $g \in \mathcal{G}$ such that:
\begin{align}
        \EE_{\bx}  \left(f\left(\bW_1^{\mS}\bx_{\mS}\right) - g(\bx_{\mS})\right)^2 \le C_2 s_0 (W_{\mathcal{G}} D_{\mathcal{G}})^{-4/(s_0+4)}
\end{align}
where $C_2>0$ is a constant and $\bW_1^{\mS} = \left\{\bW_1\right\}_{\cdot \mS}$. Then:
    \begin{equation}
        \begin{aligned}
            \mathop{\inf}_{\bar{g}\in \mathcal{G}} \EE_{\hat{\mS}} \mathcal{L}(\bar{g}) &\le  \EE_{\hat{\mS}} \mathcal{L}(g) \\
            &= \EE_{\left(\bx ,y \right), \hat{\mS}}\left( y - g\left(\bx_{\hat{\mS}}\right) \right)^2 \\
            &= \EE_{\left(\bx,y\right),\hat{\mS}} \left[\left( y - g\left(\bx_{\hat{\mS}}\right) \right)^2 \left| \hat{\mS} = \mS\right.\right]\cdot \P\left(\hat{\mS} = \mS\right) \\
            &\ \ \ + \EE_{\left(\bx,y\right),\hat{\mS}} \left[\left( y - g\left(\bx_{\hat{\mS}}\right) \right)^2 \left| \hat{\mS} \neq \mS\right.\right]\cdot \P\left(\hat{\mS} \neq \mS\right) \\
            &\le \EE_{\left(\bx ,y\right)}\left( y - g\left(\bx_{\mS}\right) \right)^2 + (\sqrt{s} M+ B_{\mathcal{G}})^2 \cdot \nu
        \end{aligned}
    \end{equation}
    where $ \left\|y\right\|_{\psi_2} \le \sqrt{s} M$. What's more, by Lemma \ref{subexpinequality}, \ref{lmrelu} and Assumption \ref{nsa1}, we have:
        \begin{align*}
            \EE_{\left(\bx ,y\right)}\left( y - g\left(\bx_{\mS}\right) \right)^2 &= \EE_{\bx }\left( f\left(\bW_1 \bx\right) - g\left(\bx_{\mS}\right) \right)^2 + \EE \epsilon^2 \\
            &\le 2\EE_{\bx} \left(f\left(\bW_1 \bx\right) - f\left(\bW_1^{\mS}\bx_{\mS}\right)\right)^2 + 2 \EE_{\bx }\left( f\left(\bW_1^{\mS} \bx_{\mS}\right) - g\left(\bx_{\mS}\right) \right)^2 + \EE \epsilon^2  \\
            &\le 2L^2\EE_{\bx}\left\|\bW_1 \bx-\bW_1^{\mS} \bx_{\mS}  \right\|_2^2 + 2C_2s_0(W_{\mathcal{G}} D_{\mathcal{G}})^{-4/(s_0+4)} \\
            & \le 2L^2 \EE_{\bx}\left(\sum\limits_{j \notin \mS} \eta_j \left|x_j\right|\right)^2 + 2C_2s_0(W_{\mathcal{G}} D_{\mathcal{G}})^{-4/(s_0+4)} \\
            &\le 4L^2 \left\|\sum\limits_{j \notin \mS} \eta_j \left|x_j\right|\right\|_{\psi_2}^2 + 2C_2s_0(W_{\mathcal{G}} D_{\mathcal{G}})^{-4/(s_0+4)}  \\
            &\le 4L^2\left(\sum\limits_{j \notin \mS} \eta_j \cdot \mathop{\sup}_{1 \leq k \leq p} \left\|x_k\right\|_{\psi_2}\right)^2+ 2C_2s_0(W_{\mathcal{G}} D_{\mathcal{G}})^{-4/(s_0+4)} \\
            &\le 4L^2\eps_0^2\left(\mathop{\sup}_{1 \leq k \leq p} \left\|x_k\right\|_{\psi_2}\right)^2+ 2C_2s_0(W_{\mathcal{G}} D_{\mathcal{G}})^{-4/(s_0+4)} \\
            &\le 4L^2\eps_0^2\left(\mathop{\sup}_{1 \leq k \leq p} \left\|x_k\right\|_{\psi_2}\right)^2+ 2C_2s(W_{\mathcal{G}} D_{\mathcal{G}})^{-4/(s+4)} 
            \refstepcounter{equation}\tag{\theequation}
        \end{align*}
\noindent   where $L>0$ is constant. Therefore, let $C = \sup \left\{(M+ B_{\mathcal{G}})^2, 4L^2\left(\mathop{\sup}_{1 \leq k \leq p} \left\|x_k\right\|_{\psi_2}\right)^2, 2C_2\right\}$, when $n\ge C_1\left(\log^2 \nu + p^2\right)$, we have:
\begin{align}
        \mathop{\inf}_{\bar{g}\in \mathcal{G}} \EE_{\hat{\mS}} \mathcal{L}(\bar{g}) \le C\left( s(W_{\mathcal{G}} D_{\mathcal{G}})^{-4/(s+4)} + s\nu + \eps_0^2\right)
\end{align}

\end{proof}

\begin{thm}\label{nsthm1}
Let $\hat{\mS}$ be generated using the Algorithm \ref{alg1}, and let $\hat{g}_{\hat{\mS}}$
 be computed according to
(\ref{op1}). Assume Assumptions \ref{a2new}-\ref{a3-} and Assumption \ref{nsa1}. Further assume that the threshold level $\eps_0< \kappa\le c \cdot \max_{|\mS | \leq s} \min_{j \in \mS} \|\{\bW_1\}_{\cdot j} \|_2$ for a positive constant $c$. When $n \ge C_1\left(\log^2 \nu +p^2\right)$ for any $0<\nu<1$, the prediction error for a new observation $(\bx,y)$ can be bounded by
    \bel{nsthm1-1}
    \EE_{(\bx_{\text{data}},y_{\text{data}}),\hat{\mS}} \left( y - \hat{g}_{\hat{\mS}} \left(\bx_{\hat{\mS}}\right) \right)^2 \le C_2 \left(sn^{-2/(s+8)} \log n +s\nu+\eps_0^2 \right),
    \eel
    where $(\bx_{\text{data}},y_{\text{data}})$ refers to the sample $\left\{\bx_i, y_i\right\}_{i=1}^n$ and a new data $(\bx, y)$.
    Furthermore, when $\nu = \mathop{\min}\left\{1/2,n^{-2/(s+8)} \log n\right\}$, $n\ge C_3p^2$ and $\eps_0^2<sn^{-2/(s+8)}\log n$, we have
    \bel{nsthm1-2}
    \EE_{(\bx_{\text{data}},y_{\text{data}}),\hat{\mS}} \left( y - \hat{g}_{\hat{\mS}} \left(\bx_{\hat{\mS}}\right) \right)^2 \le C_4 sn^{-2/(s+8)} \log n .
    \eel
 Here, $C_1$ through $C_4$ denote positive constants.
\end{thm}

\begin{proof}
\quad Let $W_\mathcal{G}, D_\mathcal{G}, S_\mathcal{G}$ and $B_\mathcal{G}$ be the width, depth, size and bound of the ReLU network space $\mathcal{G}$. For any $\bar{g}\in \mathcal{G}$, we have:
\begin{equation}
    \begin{aligned}
        \EE_{\left(\bx,y\right)} \left( y - \hat{g}_{\hat \mS}\left(\bx_{\hat{\mS}}\right) \right)^2 &=\mathcal{L}(\hat{g}_{\hat \mS}) \\
        &=\left(\mathcal{L}(\hat{g}_{\hat \mS}) - \hat{\mathcal{L}}(\hat{g}_{\hat \mS})\right) + \left(\hat{\mathcal{L}}(\hat{g}_{\hat \mS}) - \hat{\mathcal{L}}(\bar{g})\right) + \left(\hat{\mathcal{L}}(\bar{g}) - \mathcal{L}(\bar{g})\right) + \mathcal{L}(\bar{g}) \\
        &\le 2\mathop{\sup}_{g\in \mathcal{G}} \left|\mathcal{L}(g) - \hat{\mathcal{L}}(g)\right| + \mathcal{L}(\bar{g})
    \end{aligned}
\end{equation}
then by Lemma \ref{nslm2} and \ref{nslm3}, there exists constans $C, C_1>0$, when $n \ge C_1\left(\log^2 \nu +p^2\right)$, we have:
\begin{equation}
    \begin{aligned}
        \EE_{(\bx_{\text{data}},y_{\text{data}}),\hat{\mS}} \left( y - \hat{g}_{\hat \mS}\left(\bx_{\hat{\mS}}\right) \right)^2 &\le 2\EE_{\left\{\bx_i,y_i\right\}_{i=1}^n,\hat{\mS}}\mathop{\sup}_{g\in \mathcal{G}} \left|\mathcal{L}(g) - \hat{\mathcal{L}}(g)\right| + \mathop{\inf}_{\bar{g}\in \mathcal{G}}\EE_{\hat{\mS}}\mathcal{L}(\bar{g}) \\
        &\le Cs\sqrt{\frac{\text{Pdim}_\mathcal{G}\log\left(B_{\mathcal{G}}n\right)}{n}} + C\left( s(W_\mathcal{G}D_\mathcal{G})^{-4/(s+4)} + s\nu + \eps_0^2 \right) \\
        &= C\left(s\sqrt{\frac{\text{Pdim}_\mathcal{G}\log\left(B_{\mathcal{G}}n\right)}{n}}+ s(W_\mathcal{G}D_\mathcal{G})^{-4/(s+4)} + s\nu + \eps_0^2 \right)
    \end{aligned}
\end{equation}
By \cite{bartlett2019nearly}, the pseudo dimension of $\mathcal{G}$ can be bounded by:
\begin{align}
    \text{Pdim}_\mathcal{G} \le c_1 D_\mathcal{G}S_{\mathcal{G}}\log S_\mathcal{G}
\end{align}
where $c_1>0$ is certain constant. Note that $S_{\mathcal{G}} \le c_2W_\mathcal{G}^2D_\mathcal{G} $ for certain constant $c_2$, take $W_{\mathcal{G}}  = D_{\mathcal{G}} = \lfloor n^{(s+4)/4(s+8)}\rfloor$, we have:
\begin{equation}
    \begin{aligned}
        &\sqrt{\frac{\text{Pdim}_\mathcal{G}\log\left(B_{\mathcal{G}}n\right)}{n}}+ (W_\mathcal{G}D_\mathcal{G})^{-4/(s+4)} \\
        \leq& \sqrt{\frac{c_1c_2W_\mathcal{G}^2D_\mathcal{G}^2\log \left(c_2W_\mathcal{G}^2D_\mathcal{G}\right)\log\left(B_{\mathcal{G}}n\right)}{n}} + (W_\mathcal{G}D_\mathcal{G})^{-4/(s+4)} \\
        \leq&\sqrt{\frac{c_1c_2n^{(s+4)/(s+8)}\log \left(c_2n^{3(s+4)/4(s+8)}\right)\log\left(B_{\mathcal{G}}n\right)}{n}} + n^{(s+4)/2(s+8)\cdot[-4/(s+4)]} \\
        \leq& c_3 n^{-2/(s+8)}\log n
    \end{aligned}
\end{equation}
where $c_3>0$ is positive constant. Therefore, let $C_2 = C\cdot\max\left\{c_3, 1\right\}$, we have:
\begin{equation}
    \begin{aligned}
        \EE_{(\bx_{\text{data}},y_{\text{data}}),\hat{\mS}} \left( y - \hat{g}_{\hat \mS}\left(\bx_{\hat{\mS}}\right) \right)^2 &\leq C\left(s\sqrt{\frac{\text{Pdim}_\mathcal{G}\log\left(B_{\mathcal{G}}n\right)}{n}}+ s(W_\mathcal{G}D_\mathcal{G})^{-4/(s+4)} + s\nu + \eps_0^2 \right) \\
        &\leq C\left(c_3sn^{-2/(s+8)}\log n + s\nu + \eps_0^2 \right) \\
        &\leq C_2\left(sn^{-2/(s+8)}\log n+ s\nu + \eps_0^2\right)
    \end{aligned}
\end{equation}
Furthermore, when $\nu = \mathop{\min}\left\{1/2,n^{-2/(s+8)}\log n \right\}$, $n\ge C_3p^2$ and $\eps_0^2<sn^{-2/(s+8)}\log n$, take $C_4 = 3C_2$, we have:
\begin{equation}
    \begin{aligned}
        \EE_{(\bx_{\text{data}},y_{\text{data}}),\hat{\mS}} \left( y - \hat{g}_{\hat \mS}\left(\bx_{\hat{\mS}}\right) \right)^2 &\leq C_2\left(sn^{-2/(s+8)}\log n+ s\nu + \eps_0^2\right) \\
        &\leq  C_2\left(sn^{-2/(s+8)}\log n + sn^{-2/(s+8)}\log n + sn^{-2/(s+8)}\log n \right) \\
        &\le C_4sn^{-2/(s+8)}\log n
    \end{aligned}
\end{equation}

\end{proof}

\section{Proof of Theorems in Section \ref{sec_unknown}}
\subsection{Proof of Theorem \ref{thm3}}

\begin{lem}\label{eigenvaluebound}
    Assume the conditions in Theorem \ref{thm3}, for any $\nu>0$, when $n>\max\left\{\frac{1}{c^2}h^2(\nu,p), \frac{16M^2}{c\delta^2}h(\nu,p)\right\}$, with probability at least $1-\nu$,
    \begin{align}
        \mathop{\sup}_{1\le j \le p} \left|\lambda_j - \hat{\lambda}_j\right| \le \frac{2M}{\sqrt{cn}}\sqrt{\log\frac{4}{\nu} + \log 9 \cdot p}
    \end{align}
    where $c, M$ are positive constants.
\end{lem}

\begin{proof}
    \quad For any matrix $\bW$, let $\Lambda_j(\bW)$ be the $j$-th largest eigenvalue of $\bW$. Then there exists a one-to-one map $\pi:[p] \mapsto [p]$ such that:
\begin{equation}
    \begin{aligned}
        \lambda_j = \Lambda_{\pi(j)}\left(\EE yT(\bx)\right), \quad 1\le j\le p
    \end{aligned}
\end{equation}

\noindent By Lemma \ref{e2bound}, when $n > \frac{1}{c^2}\left(\log \frac{4}{\nu}+ \log 9 \cdot p\right)^2 $, with probability at least $1-\nu$,
\begin{equation}
    \begin{aligned}
        \left\| \frac{1}{n}\sum_{i=1}^{n} y_i T(\bx_i) - \EE yT(\bx) \right\|_2  \le  \frac{2M}{\sqrt{cn}} \sqrt{\log \frac{4}{\nu} + \log 9 \cdot p}
    \end{aligned}
\end{equation}
where $c, M$ are positive constants. Then by Weyl's inequality, for $1\le j \le p$, we have:
\begin{equation}
    \begin{aligned}
        \left|\Lambda_{j}\left(\EE yT(\bx)\right)-\Lambda_{j}\left(\frac{1}{n}\sum_{i=1}^{n} y_i T(\bx_i)\right)\right| &\le \left\| \frac{1}{n}\sum_{i=1}^{n} y_i T(\bx_i) - \EE yT(\bx) \right\|_2  \\
        &\le  \frac{2M}{\sqrt{cn}} \sqrt{\log \frac{4}{\nu} + \log 9 \cdot p}
    \end{aligned}
\end{equation}
    
\noindent Note that $n>\frac{16M^2}{c\delta^2}\left(\log \frac{4}{\nu}+\log 9\cdot p\right)$, then for $1\le j\le k_1$,

    \begin{align*}
        &\left|\Lambda_{\pi(j)}\left(\frac{1}{n}\sum_{i=1}^{n} y_i T(\bx_i)\right)\right|-\left|\Lambda_{\pi(j+1)}\left(\frac{1}{n}\sum_{i=1}^{n} y_i T(\bx_i)\right)\right| \\
        \ge& \left|\Lambda_{\pi(j)}\left(\EE yT(\bx)\right)\right| - \left| \Lambda_{\pi(j+1)}\left(\EE yT(\bx)\right)\right| - 2 \mathop{\sup}_{1\le k\le p} \left|\Lambda_{k}\left(\EE yT(\bx)\right)-\Lambda_{k}\left(\frac{1}{n}\sum_{i=1}^{n} y_i T(\bx_i)\right)\right| \\
        =& \left|\lambda_j\right| - \left|\lambda_{j+1}\right| - 2 \mathop{\sup}_{1\le k\le p} \left|\Lambda_{k}\left(\EE yT(\bx)\right)-\Lambda_{k}\left(\frac{1}{n}\sum_{i=1}^{n} y_i T(\bx_i)\right)\right| \\
        \ge& \delta - \frac{4M}{\sqrt{cn}} \sqrt{\log \frac{4}{\nu} + \log 9 \cdot p} \\
        >& 0  \refstepcounter{equation}\tag{\theequation}
    \end{align*}

\noindent on the other hand, for $k_1 < j \le p$,
     \begin{align*}
        &\left|\Lambda_{\pi(k_1)}\left(\frac{1}{n}\sum_{i=1}^{n} y_i T(\bx_i)\right)\right|-\left|\Lambda_{\pi(j)}\left(\frac{1}{n}\sum_{i=1}^{n} y_i T(\bx_i)\right)\right| \\
        \ge& \left|\Lambda_{\pi(k_1)}\left(\EE yT(\bx)\right)\right| - \left| \Lambda_{\pi(j)}\left(\EE yT(\bx)\right)\right| - 2 \mathop{\sup}_{1\le k\le p} \left|\Lambda_{k}\left(\EE yT(\bx)\right)-\Lambda_{k}\left(\frac{1}{n}\sum_{i=1}^{n} y_i T(\bx_i)\right)\right|   \\
        =& \left|\lambda_{k_1}\right| - \left|\lambda_{j}\right| - 2 \mathop{\sup}_{1\le k\le p} \left|\Lambda_{k}\left(\EE yT(\bx)\right)-\Lambda_{k}\left(\frac{1}{n}\sum_{i=1}^{n} y_i T(\bx_i)\right)\right| \\
        \ge& \delta - \frac{4M}{\sqrt{cn}} \sqrt{\log \frac{4}{\nu} + \log 9 \cdot p} \\
        >& 0   \refstepcounter{equation}\tag{\theequation}
    \end{align*}

\noindent  combining with the property that $|\hat{\lambda}_1| \ge |\hat{\lambda}_2|\ge \cdots\ge|\hat{\lambda}_p|$, we have:
    \begin{align}
        \hat{\lambda}_j = \Lambda_{\pi(j)}, \quad 1\le j \le k_1
    \end{align}
    
    \noindent Then, for $1\le j \le k_1$, we have:
\begin{equation}
    \begin{aligned}
        |\lambda_j - \hat{\lambda}_j| &= \left|\Lambda_{\pi(j)}\left(\EE yT(\bx)\right)-\Lambda_{\pi(j)}\left(\frac{1}{n}\sum_{i=1}^{n} y_i T(\bx_i)\right)\right|\\
        &\le \frac{2M}{\sqrt{cn}} \sqrt{\log \frac{4}{\nu} + \log 9 \cdot p}
    \end{aligned}
\end{equation}

    \noindent Furthermore, there exists a one-to-one map $\phi:\left\{k_1+1,\cdots, p\right\} \mapsto \left\{k_1+1,\cdots, p\right\}$ such that:
    \begin{align}
        \hat{\lambda}_{\phi(j)} = \Lambda_{\pi(j)}\left(\EE yT(\bx)\right), \quad k_1< j\le p
    \end{align}
    Note that $\lambda_{k_1+1} = \cdots = \lambda_p = 0$, then for $k_1<j \le p$:

        \begin{align*}
        |\lambda_j - \hat{\lambda}_j| &=  \left|\lambda_{\phi^{-1}(j)} - \hat{\lambda}_j\right|\\
        &= \left| \Lambda_{\pi\left(\phi^{-1}\left(j\right)\right)}\left(\EE yT(\bx)\right)- \Lambda_{\pi\left(\phi^{-1}\left(j\right)\right)}\left(\frac{1}{n}\sum_{i=1}^{n} y_i T(\bx_i)\right)\right| \\
        &\le \frac{2M}{\sqrt{cn}} \sqrt{\log \frac{4}{\nu} + \log 9 \cdot p} \refstepcounter{equation}\tag{\theequation}
        \end{align*}

    \noindent Therefore,
    \begin{align}
        \mathop{\sup}_{1\le j\le p} \left|\lambda_j - \hat{\lambda}_j\right| \le \frac{2M}{\sqrt{cn}}\sqrt{\log\frac{4}{\nu} + \log 9 \cdot p}
    \end{align}
\end{proof}

\begin{thm}\label{thm3}
     Let $\delta = \mathop{\min}_{i\in [k_1]} \left\{\vert \lambda_i\vert - \vert \lambda_{i+1}\vert \right\} > 0$. Suppose that $\tau < \delta/2$, then for any $\nu>0 $, when $n\ge C(\log \nu^{-2}+ p^2)$ for certain positive constant $C$, we have 
    \begin{align*}
        \P\left(\hat{k}_1 = k_1\right) \ge 1-\nu.
    \end{align*}
\end{thm}

\begin{proof}
    \quad By Lemma \ref{eigenvaluebound}, when $n>\max\left\{\frac{1}{c^2}h^2(\nu,p), \frac{16M^2}{c\delta^2}h(\nu,p)\right\}$, with probability at least $1-\nu$,
    \begin{align}
        \mathop{\sup}_{1\le j\le p} \left|\lambda_j - \hat{\lambda}_j\right| \le \frac{2M}{\sqrt{cn}}\sqrt{\log\frac{4}{\nu} + \log 9 \cdot p}
    \end{align}
    where $c,M$ are positive constants.
    
    \noindent Furthermore, when $n>\frac{16M^2}{c\tau^2}\left(\log \frac{4}{\nu}+ \log 9 \cdot p\right)$, for $1\le i \le k_1$, we have:
\begin{equation}
    \begin{aligned}
        \left| \hat{\lambda}_i \right|-\left| \hat{\lambda}_{i+1}\right| &\ge \left| \lambda_i \right| - \left| \lambda_{i+1} \right| - 2 \mathop{\sup}_{1\le j\le p} \left| \lambda_j - \hat{\lambda}_j \right| \\
        &\ge \delta - \frac{4M}{\sqrt{cn}}\sqrt{\log\frac{4}{\nu} + \log 9 \cdot p}\\
        &\ge \delta - \tau \\
        &>\tau
    \end{aligned}
\end{equation}
    
\noindent then we can conclude that $\hat{k}_1 \ge k_1$.

    \noindent On the other hand,
\begin{equation}
    \begin{aligned}
        \left| \hat{\lambda}_{k_1+1} \right|-\left| \hat{\lambda}_{k_1+2}\right| &\le \left| \lambda_{k_1+1} \right| - \left| \lambda_{k_1+2} \right| + 2 \mathop{\sup}_{1\le j\le p} \left| \lambda_j - \hat{\lambda}_j \right| \\
        &\le 0 - 0 + \frac{4M}{\sqrt{cn}}\sqrt{\log\frac{4}{\nu} + \log 9 \cdot p}\\
        &\le \tau
    \end{aligned}
\end{equation}
    
\noindent by the definition of $\hat{k}_1$, we have $\hat{k}_1 \le k_1$.

    \noindent Combining the above two conclusions, under the condition that $n \ge C\left(\log^2 \nu + p^2\right)$ for certain constant $C$, with probability at least $1-\nu$, we have $\hat{k}_1 = k_1$.

\end{proof}

\subsection{Proof of Theorem \ref{thm4}}
\begin{lem}\label{matrixbound}
    Assume the conditions in Theorem \ref{thm4}, for any $\nu>0$, when $n> \max\left\{\frac{1}{c^2}h_1^2(\nu,p), \frac{M}{c}h_1(\nu,p)\right\}$, with probability at least $1-\nu$,
    \begin{align}
        \left\|\frac{1}{n}\sum\limits_{i=1}^n y_i \hat{T}(\bx_i) - \EE yT(\bx) \right\|_2 \le \frac{2M}{\sqrt{cn}}\sqrt{\log \frac{16}{\nu}+ \log 9\cdot p}
    \end{align}
    where $h_1(\nu,p) = h(\nu,p)+\log 4$ and $c, M$ are positive constants.
\end{lem}

\begin{proof}
    \quad For any $\bu\in \mathbb{R}^p$ s.t. $\Vert \bu\Vert_2 = 1$, and $1\le i \le n$, $\bu^T\bx_i$ follows the Gaussian distribution with mean zero, then by Lemma \ref{subexpinequality}, $\left(\bu^T\bx_i\right)^2$ is sub-exponential random variable with:
\begin{equation}
    \begin{aligned}
        \left\|\left(\bu^T\bx_i\right)^2\right\|_{\psi_1} \le 2 \left\|\bu^T\bx_i\right\|_{\psi_2}^2 
        \le 2 \text{Var}\left(\bu^T\bx_i\right) 
        = 2 \bu^T \bSigma \bu  \le 2\left\|\bSigma\right\|_2
    \end{aligned}
\end{equation}
    
\noindent note that $\EE \bx_i\bx_i^T = \bSigma$, then $\left(\bu^T\bx_i\right)^2-\bu^T \bSigma \bu$ is centered sub-exponential random variable and:
    \begin{align}\label{matrixbound-1}
        \left\|\left(\bu^T\bx_i\right)^2-\bu^T \bSigma \bu\right\|_{\psi_1} \le 2 \left\|\left(\bu^T\bx_i\right)^2\right\|_{\psi_1} \le 4\left\|\bSigma\right\|_2
    \end{align}
    By Bernstein's inequality, we have:
\begin{equation}
    \begin{aligned}
        \P\left(\left|\bu^T\left(\hat{\bSigma} - \bSigma\right)\bu\right|\ge t\right) =& \P\left(\left|\frac1n \sum\limits_{i=1}^n \left(\bu^T\bx_i\right)^2 - \bu^T \bSigma \bu\right|\ge t\right) \\
        \le & 2 \exp\left\{-c_1 \mathop{\min}\left(\frac{t^2}{M_1^2}, \frac{t}{M_1}\right)n\right\}
    \end{aligned}
\end{equation}

\noindent where $ M_1 = 4\left\|\bSigma\right\|_2$ and $c_1$ is a positive constant. Let $S_\eps^{p-1}$ be the $\eps$-net of $S^{p-1}$, by Lemma \ref{lm3} and \ref{lm4}, we have:

\begin{equation}
    \begin{aligned}
        &\P\left(\left\|\hat{\bSigma} - \bSigma\right\|_2 \ge 2t\right) \\
        =& \P\left(\mathop{\sup}_{\bu\in S^{p-1}} \left|\bu^T\left(\hat{\bSigma} - \bSigma\right)\bu\right| \ge 2t\right) \\
        \le &\P\left(\mathop{\sup}_{\bu\in S_{1/4}^{p-1}} \left|\bu^T\left(\hat{\bSigma} - \bSigma\right)\bu\right| \ge t\right) \\
        \le& \left| S_{1/4}^{p-1}\right|\cdot 2\exp\left\{-c_1 \mathop{\min}\left(\frac{t^2}{M_1^2}, \frac{t}{M_1}\right)n\right\} \\
        \le &2\cdot 9^p \exp\left\{-c_1 \mathop{\min}\left(\frac{t^2}{M_1^2}, \frac{t}{M_1}\right)n\right\}
    \end{aligned}
\end{equation}
    
\noindent let $t = \frac{M_1}{\sqrt{c_1n}}\sqrt{\log \frac{8}{\nu}+\log9\cdot p}$, when $n > \frac{1}{c_1}\left(\log\frac{8}{\nu}+ \log 9\cdot p\right)$, with probability at least $1-\nu/4$,
    \begin{align}
        \left\|\hat{\bSigma} - \bSigma\right\|_2 \le 2t = \frac{2M_1}{\sqrt{c_1n}}\sqrt{\log \frac{8}{\nu}+\log9\cdot p}
    \end{align}

    \noindent Furthermore, when $n>\frac{16M_1^2}{c_1\phi_{\min}^2\left(\bSigma\right)}\left(\log\frac{8}{\nu}+ \log 9\cdot p\right) = \frac{16M_1^2\left\|\bSigma^{-1}\right\|_2^2}{c_1}\left(\log\frac{8}{\nu}+ \log 9\cdot p\right)$, by Weyl's inequality, we have:
    
    \begin{align*}\label{matrixbound-2}
        \left\|\hat{\bSigma}^{-1} - \bSigma^{-1}\right\|_2 
        \le& \left\|\hat{\bSigma}^{-1}\right\|_2\cdot \left\|\bSigma-\hat{\bSigma}\right\|_2\cdot\left\|\bSigma^{-1}\right\|_2 \\
        \le&\frac{2t}{\phi_{\min}\left(\bSigma\right)\phi_{\min}\left(\hat{\bSigma}\right)} \\
        \le& \frac{2t}{\phi_{\min}\left(\bSigma\right)\left[\phi_{\min}\left(\bSigma\right) -\left\|\hat{\bSigma} - \bSigma\right\|_2 \right]} \refstepcounter{equation}\tag{\theequation}\\  
        \le&\frac{2t}{\phi_{\min}\left(\bSigma\right)\left[\phi_{\min}\left(\bSigma\right) -2t \right]} \\
        \le& \frac{4t}{\phi_{\min}^2\left(\bSigma\right)} \\
        =&\frac{4M_1\left\|\bSigma^{-1}\right\|_2^2}{\sqrt{c_1n}}\sqrt{\log \frac{8}{\nu}+\log9\cdot p} 
    \end{align*}

\noindent  From inequality (\ref{matrixbound-1}), we have:
    \begin{align}
        \mathop{\sup}_{\|\bu\|_2=1} \left\|\bu^T \left(\bx\bx^T-\bSigma\right) \bu\right\|_{\psi_1} \le 2 \mathop{\sup}_{\|\bu\|_2=1} \left\|\bu^T\bx\bx^T\bu\right\|_{\psi_1} \le 4\left\|\bSigma\right\|_2
    \end{align}
    then by Lemma \ref{e2bound}, there exists constants $c_2, M_2>0$, when $n > \frac{1}{c_2^2}\left(\log\frac{16}{\nu}+ \log 9\cdot p\right)^2$, each of the following inequalities holds with probability at least $1-\nu/4$:
    
    \begin{align}
        &\left\|\frac{1}{n}\sum\limits_{i=1}^n y_i \bx_i\bx_i^T - \EE y
        \bx\bx^T\right\|_2 \le \frac{2M_2}{\sqrt{c_2n}} \sqrt{\log \frac{16}{\nu}+ \log 9\cdot p}  \label{matrixbound-3}\\
        & \left\|\frac{1}{n}\sum\limits_{i=1}^n y_i \left(\bx_i\bx_i^T-\bSigma\right) - \EE y
        \left(\bx\bx^T-\bSigma\right)\right\|_2 \le \frac{2M_2}{\sqrt{c_2n}} \sqrt{\log \frac{16}{\nu}+ \log 9\cdot p} \label{matrixbound-4} \\
        &\left\|\frac{1}{n}\sum\limits_{i=1}^n y_i T(\bx_i) - \EE yT(\bx) \right\|_2 \le \frac{2M_2}{\sqrt{c_2n}} \sqrt{\log \frac{16}{\nu}+ \log 9\cdot p} \label{matrixbound-5}
    \end{align}
    where $c_2, M_2$ are positive constants. Note that $\frac{M_2}{\sqrt{c_2n}} \sqrt{\log \frac{16}{\nu}+ \log 9\cdot p} \le 1$, by Cauchy's inequality, we have:
        \begin{align*}
        \left\|\frac{1}{n}\sum\limits_{i=1}^n y_i \bx_i\bx_i^T\right\|_2 &\le \left\|\EE y
        \bx\bx^T\right\|_2 + 2 \\
        & = \mathop{\sup}_{\|\bu\|_2=1} \left| \bu^T \EE y \bx\bx^T\bu\right| + 2 \\
        & \le \mathop{\sup}_{\|\bu\|_2=1} \left( \EE y^2 \cdot \EE \left[\bu^T\bx\bx^T\bu\right]^2\right)^{1/2} +2\\
        &\le \mathop{\sup}_{\|\bu\|_2=1} \sqrt{2}\left\|y\right\|_{\psi_2}\cdot 2 \left\|\bu^T\bx\bx^T\bu\right\|_{\psi_1} +2\\
        &\le 4\sqrt{2} \left\|\bSigma\right\|_2\left\|y\right\|_{\psi_2} +2
        \refstepcounter{equation}\tag{\theequation}
    \end{align*}
and,
    \begin{align*}
        \left\|\frac{1}{n}\sum\limits_{i=1}^n y_i \left(\bx_i\bx_i^T-\bSigma\right)\right\|_2 &\le \left\|\EE y
        \left(\bx\bx^T-\bSigma\right)\right\|_2 + 2 \\
        &= \mathop{\sup}_{\|\bu\|_2=1} \left| \bu^T \EE y \left(\bx\bx^T-\bSigma\right)\bu\right| + 2 \\
        &\le 8\sqrt{2} \left\|\bSigma\right\|_2\left\|y\right\|_{\psi_2} +2 \refstepcounter{equation}\tag{\theequation}
    \end{align*}
 \noindent Let $M_3 = \max\left\{1,16M_1^2\left\| \bSigma^{-1}\right\|_2^2\right\}$, $M_4 = 8\sqrt{2} \left\|\bSigma\right\|_2\left\|y\right\|_{\psi_2} + 2$ and $h_1(\nu,p) = h(\nu,p) + \log 4$, by inequalities (\ref{matrixbound-2}), (\ref{matrixbound-3}) and (\ref{matrixbound-4}), when $n> \max\left\{\frac{1}{c_2^2}h_1^2(\nu,p), \frac{M_3}{c_1}h_1(\nu,p)\right\}$, with probability at least $1-\frac{3\nu}{4}$,
 

         \begin{align*}
        &\left\|\frac{1}{n}\sum\limits_{i=1}^n y_i \hat{T}(\bx_i) - \frac{1}{n}\sum\limits_{i=1}^n y_i T(\bx_i)\right\|_2 \\
        =& \left\| \left(\bSigma^{-1}-\hat{\bSigma}^{-1}\right) \frac1n \sum\limits_{i=1}^n y_i\bx_i\bx_i^T\left(\bSigma^{-1}-\hat{\bSigma}^{-1}\right) -\left(\bSigma^{-1}-\hat{\bSigma}^{-1}\right)\frac1n \sum\limits_{i=1}^n y_i\bx_i\bx_i^T \bSigma^{-1} \right. \\
        &\quad \left.- \bSigma^{-1} \frac1n \sum\limits_{i=1}^n y_i\left(\bx_i\bx_i^T-\bSigma\right)\left(\bSigma^{-1}-\hat{\bSigma}^{-1}\right) \right\|_2 \\
        \le& \left\| \left(\bSigma^{-1}-\hat{\bSigma}^{-1}\right) \frac1n \sum\limits_{i=1}^n y_i\bx_i\bx_i^T\left(\bSigma^{-1}-\hat{\bSigma}^{-1}\right) \right\|_2 + \left\| \left(\bSigma^{-1}-\hat{\bSigma}^{-1}\right)\frac1n \sum\limits_{i=1}^n y_i\bx_i\bx_i^T \bSigma^{-1}\right\|_2 \\
        &\quad + \left\|\bSigma^{-1} \frac1n \sum\limits_{i=1}^n y_i\left(\bx_i\bx_i^T-\bSigma\right)\left(\bSigma^{-1}-\hat{\bSigma}^{-1}\right) \right\|_2 \\
        \le& M_4\left\| \bSigma^{-1}-\hat{\bSigma}^{-1}\right\|_2^2 + M_4\left\| \bSigma^{-1}-\hat{\bSigma}^{-1}\right\|_2 \left\| \bSigma^{-1}\right\|_2 + M_4\left\| \bSigma^{-1}\right\|_2\left\| \bSigma^{-1}-\hat{\bSigma}^{-1}\right\|_2 \\
        \le& M_4\cdot\left(\frac{4M_1\left\|\bSigma^{-1}\right\|_2^2}{\sqrt{c_1n}}\sqrt{\log \frac{8}{\nu}+\log9\cdot p} + 2\left\| \bSigma^{-1}\right\|_2\right)\cdot \frac{4M_1\left\|\bSigma^{-1}\right\|_2^2}{\sqrt{c_1n}}\sqrt{\log \frac{8}{\nu}+\log9\cdot p} \\
        \le& 2M_4\left(2\left\| \bSigma^{-1}\right\|_2^2 + \left\| \bSigma^{-1}\right\|_2 \right)\cdot \frac{4M_1\left\| \bSigma^{-1}\right\|_2^2}{\sqrt{c_1n}}\sqrt{\log \frac{8}{\nu}+\log9\cdot p} \refstepcounter{equation}\tag{\theequation}
    \end{align*}

    \noindent combining with inequality (\ref{matrixbound-5}),  with probability at least $1-\nu$, we have:
        \begin{align*}
        &\left\|\frac{1}{n}\sum\limits_{i=1}^n y_i \hat{T}(\bx_i) - \EE yT(\bx)\right\|_2 \\
        \le& \left\|\frac{1}{n}\sum\limits_{i=1}^n y_i \hat{T}(\bx_i) - \frac{1}{n}\sum\limits_{i=1}^n y_i T(\bx_i)\right\|_2 + \left\|\frac{1}{n}\sum\limits_{i=1}^n y_i T(\bx_i) - \EE yT(\bx) \right\|_2 \\
        \le& \frac{2M_5}{\sqrt{c_1n}}\sqrt{\log \frac{8}{\nu}+ \log 9\cdot p} + \frac{2M_2}{\sqrt{c_2n}}\sqrt{\log \frac{16}{\nu}+ \log 9\cdot p} \refstepcounter{equation}\tag{\theequation}
    \end{align*}
 
\noindent   where $M_5 = 4M_1\left\| \bSigma^{-1}\right\|_2^2\cdot M_4\left(2\left\| \bSigma^{-1}\right\|_2^2 + \left\| \bSigma^{-1}\right\|_2 \right)$. Therefore, let $M = \max\left\{M_2+M_5, M_3\right\}$ and $c = \min\left\{c_1,c_2\right\}$, when $n> \max\left\{\frac{1}{c^2}h_1^2(\nu,p), \frac{M}{c}h_1(\nu,p)\right\}$, with probability at least $1-\nu$, we have:
\begin{equation}
    \begin{aligned}
    &\left\|\frac{1}{n}\sum\limits_{i=1}^n y_i \hat{T}(\bx_i) - \EE yT(\bx)\right\|_2\\
    \le &\frac{2M_5}{\sqrt{c_1n}}\sqrt{\log \frac{8}{\nu}+ \log 9\cdot p} + \frac{2M_2}{\sqrt{c_2n}}\sqrt{\log \frac{16}{\nu}+ \log 9\cdot p}\\
    \le& \frac{2M}{\sqrt{cn}}\sqrt{\log \frac{16}{\nu}+ \log 9\cdot p}
\end{aligned}
\end{equation}
\end{proof}

\begin{lem}\label{eigenvectorbound2}
    Assume the conditions in Theorem \ref{thm4} hold, then for any $\nu>0$, when $n>\max\left\{ \left(\frac{1}{c_1^2}+\frac{c_2k_1^6\mu^4}{\lambda_{k_1}^2}\right)h_1^2(\nu,p), \left(\frac{4M^2}{c_1\delta^2}+\frac{M}{c_1}\right)h_1(\nu,p)\right\}$, with probability at least $1-\nu$:
    \begin{align}\label{eigenvectorbound2-1}
        \|\bW - \hat\bW_1\|_\text{max} \le \frac{2M}{\sqrt{c_1n}}\left(\frac{k_1^4\mu^2}{\vert \lambda_{k_1}\vert}+ \frac{k_1^{3/2}\mu^{1/2}}{\delta\sqrt{p} }\right)\left(\sqrt{\log \frac{16}{\nu} + \log 9 \cdot p}+K\right)
    \end{align}
    where $h_1(\nu,p) = h(\nu,p)+K$ and $c_1, c_2, M, K$ are positive constants.
\end{lem}

\begin{proof}
\quad    Define matrix $\bE_3 = \frac{1}{n}\sum\limits_{i=1}^n y_i \hat{T}(\bx_i) - \EE g(\bx)T(\bx)$. By Lemma \ref{matrixbound}, when $n> \max\left\{\frac{1}{c_1^2}h_2^2(\nu,p), \frac{M_1}{c_1}h_2(\nu,p)\right\}$, with probability at least $1-\nu$:
    \begin{align}
        \left\|\frac{1}{n}\sum\limits_{i=1}^n y_i \hat{T}(\bx_i) - \EE yT(\bx)\right\|_2 \le \frac{2M_1}{\sqrt{c_1n}}\sqrt{\log \frac{16}{\nu}+ \log 9\cdot p}
    \end{align}
    where $h_2(\nu,p) = h(\nu, p)+\log 4$ and $c_1, M_1$ are positive constants. Combining with Proposition \ref{dnnprop}, we have:
    \begin{equation}
        \begin{aligned}
        \left\|\bE_3\right\|_2 &\le \left\|\frac{1}{n}\sum\limits_{i=1}^n y_i \hat{T}(\bx_i) - \EE yT(\bx)\right\|_2 + \left\| \mathbb{E} 
        \left[\left\{G(\bx_{\mS_0})-g(\bx)\right\}T(\bx)\right] \right\|_2 \\
        \le& \frac{2M_1}{\sqrt{c_1n}}\sqrt{\log \frac{16}{\nu}+ \log 9\cdot p} + \frac{M_2}{\sqrt{n}}
    \end{aligned}
    \end{equation}
    
\noindent  where $M_2>0$ is a constant. From Lemma \ref{A2inf}, we can further obtain that:
    \begin{align}
        \left\|\bE_3\right\|_\infty \le \sqrt{p}\left\|\bE_3\right\|_2 \le \frac{2M_1\sqrt{p}}{\sqrt{c_1n}}\sqrt{\log \frac{16}{\nu}+ \log 9\cdot p} + \frac{M_2\sqrt{p}}{\sqrt{n}}
    \end{align}
    Let $K = \max\left\{\frac{\sqrt{c}M_2}{2M_1},\log 4\right\}$ and $h_1(\nu,p) = h(\nu,p)+K$, if $n>\max\left\{ \frac{c_2k_1^6\mu^4}{\lambda_{k_1}^2}h_1^2(\nu,p), \frac{4M_1^2}{c_1\delta^2}h_1(\nu,p)\right\}$ for some positive constant $c_2$, then $|\lambda_{k_1}| = \Omega\left(k_1^3\mu^2(\bW)\left\|\bE_3\right\|_\infty\right)$ and $\Vert \bE_3\Vert_2 < \delta$, by the Theorem 3 in \cite{fan2017ellinftyeigenvectorperturbationbound}, we have:
    \begin{equation}
        \begin{aligned}
            \|\bW - \hat\bW_1\|_\text{max}&\le M_3 \left(\frac{k_1^4\mu^2(\bW)\Vert \bE_3\Vert_\infty}{\vert \lambda_{k_1}\vert \sqrt{p}}+ \frac{k_1^{3/2}\mu^{1/2}(\bW)\Vert \bE_3 \Vert_2}{\delta \sqrt{p}}\right) \\
            &\le \frac{2M_1M_3'}{\sqrt{c_1n}}\left(\frac{k_1^4\mu^2}{\vert \lambda_{k_1}\vert}+  \frac{k_1^{3/2}\mu^{1/2}}{\delta\sqrt{p} }\right)\left(\sqrt{\log \frac{16}{\nu} + \log 9 \cdot p}+K\right)
        \end{aligned}
    \end{equation}
    where $M_3,M_3'$ are positive constants. Therefore, we may let $M = \max\left\{M_1, M_1M_3\right\}$, so that when $n>\max\left\{ \left(\frac{1}{c_1^2}+\frac{c_2k_1^6\mu^4}{\lambda_{k_1}^2}\right) h_1^2(\nu,p), \left(\frac{4M^2}{c_1\delta^2}+\frac{M}{c_1}\right)h_1(\nu,p)\right\}$, the inequality (\ref{eigenvectorbound2-1}) holds with probability at least $1-\nu$.
    
\end{proof}

\vspace{1em}
\begin{thm}\label{thm4}
    Suppose $(\bx_i,y_i)$ follows model (\ref{model1}) and assume the conditions in Theorem \ref{thm2}. Suppose that $\bx_i\sim \mathcal{N}(\mathbf{0}, \bSigma)$. Let $\hat{\bSigma} = (1/n) \sum\limits_{i=1}^n \bx_i\bx_i^T$ and  $\hat{T}(\bx) = \hat{\bSigma}^{-1}\bx\bx^T \hat{\bSigma}^{-1} - \hat{\bSigma}^{-1}$.  Then, for any $\nu>0$, when
  $n \ge C\left(\log^2 \nu + p^2\right)$ for certain positive constant $C$, Algorithm \ref{alg1} guarantees feature selection consistency for H{\"o}lder smooth functions, i.e.,
$\P(\hat{\mS}_0 =\mS_0) \ge 1-\nu$.
Moreover, if $(\bx_i,y_i)$ follows model (\ref{model2}) and the conditions in Theorem \ref{thm1} hold, we have feature selection consistency for DNNs.
\end{thm}

\begin{proof}
    \quad We only need to prove the case that $(\bx_i,y_i)$ follows model (\ref{model1}). By Lemma \ref{eigenvectorbound2}, when $n>\max\left\{ \left(\frac{1}{c_1^2}+\frac{c_2k_1^6\mu^4}{\lambda_{k_1}^2}\right)h_1^2(\nu,p), \left(\frac{4M^2}{c_1\delta^2}+\frac{M}{c_1}\right)h_1(\nu,p)\right\}$, with probability at least $1-\nu$:
    \begin{equation}
        \begin{aligned}
        \|\bW - \hat\bW_1\|_\text{max} \le \frac{2M}{\sqrt{c_1n}}\left(\frac{k_1^4\mu^2}{\vert \lambda_{k_1}\vert}+ \frac{k_1^{3/2}\mu^{1/2}}{\delta\sqrt{p} }\right)\left(\sqrt{\log \frac{16}{\nu} + \log 9 \cdot p}+K\right)
        \end{aligned}
    \end{equation}
    where $h_1(\nu,p) = h(\nu, p) + K$ and $c_1,c_2,M,K$ are positive constants.
    
    \noindent Then for $j\in \hat\mS_0$, when $n > \frac{4M^2}{c_1\kappa^2}\left(\frac{k_1^{9/2}\mu^2}{\vert \lambda_{k_1}\vert}+\frac{k_1^2\mu^{1/2}}{\delta \sqrt{p}}\right)^2\left(\sqrt{\log \frac{16}{\nu} + \log 9 \cdot p}+K\right)^2$, we have:
    
    \begin{equation}
        \begin{aligned}
        \|\bW_{\cdot j}\|_2 &\ge \|\{\hat\bW_1\}_{\cdot j}\|_2 - \sqrt{k_1}\Vert \bW-\hat\bW_1\Vert_\text{max} \\
        &\ge \kappa - \sqrt{k_1}\Vert \bW-\hat\bW_1 \Vert_\text{max} \\
        &\ge \kappa - \sqrt{k_1} \cdot\frac{2M}{\sqrt{c_1n}}\left(\frac{k_1^4\mu^2}{\vert \lambda_{k_1}\vert}+ \frac{k_1^{3/2}\mu^{1/2}}{\delta\sqrt{p} }\right)\left(\sqrt{\log \frac{16}{\nu} + \log 9 \cdot p}+K\right)\\
        &> 0
    \end{aligned}
    \end{equation}
    
   \noindent then $j\in \mS_0$. Furthermore, $ \hat\mS_0 \subseteq \mS_0$.

   \noindent On the other hand, let $j \in \mS_0$, then we have:
    \begin{equation}
        \begin{aligned}
        \|\hat\bW_{\cdot j}\|_2 &\ge \|\bW_{\cdot j}\|_2 - \sqrt{k_1}\|\bW - \hat\bW_1\|_\text{max} \\
        &\ge \min_{j \in \mS_0} \|\bW_{\cdot j}\|_2 - \sqrt{k_1} \cdot \frac{2M}{\sqrt{c_1n}}\left(\frac{k_1^4\mu^2}{\vert \lambda_{k_1}\vert}+ \frac{k_1^{3/2}\mu^{1/2}}{\delta\sqrt{p} }\right)\left(\sqrt{\log \frac{16}{\nu} + \log 9 \cdot p}+K\right)\\
        &\ge \min_{j \in \mS_0} \|\bW_{\cdot j}\|_2 - \kappa \\
        &\ge \kappa
    \end{aligned}
    \end{equation}
    
   \noindent then $j \in \hat\mS_0$. Furthermore, $\mS_0 \subseteq \hat\mS_0$.

 \noindent Combining the above two conclusions, under the condition that  $n \ge C\left(\log^2 \nu + p^2\right)$ for certain positive constant $C$,  with probability at least $1-\nu$, we have $\mS_0 = \hat\mS_0$.
    
\end{proof}

\section{Additional Results of the Numerical Analysis}
In this section, we provide additional details on the analysis of the Alzheimer's Disease Neuroimaging Initiative (ADNI) dataset, as well as the implementation specifics of the competing methods and some additional simulation results.

\subsection{Real Data Analysis}
 The origin dataset is obtained by filtering out patients with MMSE scores from three phases of the study: ADNI-1, ADNI-GO, and ADNI-2, which gives a dataset containing 755 samples with over 620000 SNPs. Since we primarily focus on the B Allele Frequency (BAF), we filter out the SNP columns with BAF mean values exceeding 0.1 to ensure the inclusion of significant minor alleles. In terms of the sure screening process, we standardize the BAF values and select top $50\%$ of the sample size SNPs according to absolute marginal correlation values between BAF and MMSE scores. The final processed data consists of 755 samples, each containing 377 SNP BAF values.    

The SNPs selected by the proposed method have been shown in the manuscript. Here, we list the SNPs selected by LassoNet (Table \ref{table:lassonet}), DFS (Table \ref{table:dfs}), as well as Lasso (Table \ref{table:lasso}). 

\begin{table}[htbp!]
\centering
\begin{tabular}{p{3cm} p{3cm} p{8cm}}
\toprule
\makecell[c]{SNP} & \makecell[c]{Genes} & \makecell[c]{Reported brain-related/cognitive trait(s)} \\
\midrule

\makecell[c]{rs2385522} & \makecell[c]{FER1L6} & \makecell[c]{N/A} \\
\cdashline{1-3}[0.5pt/1pt]
\makecell[c]{rs958127} & \makecell[c]{NOL4} & \makecell[c]{N/A} \\
\cdashline{1-3}[0.5pt/1pt]
\makecell[c]{rs4921944} & \makecell[c]{PSD3} & \makecell[c]{Alzheimer Disease in Hippocampus \\ \cite{quan2020related}} \\
\cdashline{1-3}[0.5pt/1pt]
\makecell[c]{rs2412971} & \makecell[c]{HORMAD2} & \makecell[c]{N/A} \\
\cdashline{1-3}[0.5pt/1pt]
\makecell[c]{rs7923523} & \makecell[c]{MYOF} & \makecell[c]{N/A} \\
\cdashline{1-3}[0.5pt/1pt]
\makecell[c]{rs1481596} & \makecell[c]{DLC1} & \makecell[c]{N/A} \\
\cdashline{1-3}[0.5pt/1pt]
\makecell[c]{rs9855289} & \makecell[c]{ERC2} & \makecell[c]{N/A} \\
\cdashline{1-3}[0.5pt/1pt]
\makecell[c]{rs3121458} & \makecell[c]{NRAP} & \makecell[c]{Recessive dilated cardiomyopathy\\ \cite{koskenvuo2021biallelic}} \\
\cdashline{1-3}[0.5pt/1pt]
\makecell[c]{rs9907824} & \makecell[c]{NXN} & \makecell[c]{Alzheimer's Disease \\ \cite{blanco2022nxn}} \\

\bottomrule
\end{tabular}
\caption{Associated SNPs with gene identifiers by LassoNet}
\label{table:lassonet}
\end{table}

\begin{table}[htbp!]
\centering
\begin{tabular}{p{3cm} p{3cm} p{8cm}}
\toprule
\makecell[c]{SNP} & \makecell[c]{Genes} & \makecell[c]{Reported brain-related/cognitive trait(s)} \\
\midrule

\makecell[c]{rs7861396 \\ rs10976040} & \makecell[c]{KDM4C} & \makecell[c]{N/A} \\
\cdashline{1-3}[0.5pt/1pt]
\makecell[c]{rs1481596} & \makecell[c]{DLC1} & \makecell[c]{N/A} \\
\cdashline{1-3}[0.5pt/1pt]
\makecell[c]{rs11256433} & \makecell[c]{IL2RA} & \makecell[c]{Cerebral Palsy \\ \cite{qiao2022association}} \\
\cdashline{1-3}[0.5pt/1pt]
\makecell[c]{rs6569364} & \makecell[c]{NKAIN2} & \makecell[c]{Developmental delay and recurrent infections\\ \cite{yue2006disruption}}  \\
\cdashline{1-3}[0.5pt/1pt]
\makecell[c]{rs12480922} & \makecell[c]{SULF2} & \makecell[c]{Brain development and Neuronal plasticity \\ \cite{kalus2009differential}} \\
\cdashline{1-3}[0.5pt/1pt]
\makecell[c]{rs7568590} & \makecell[c]{EPCAM-DT} & \makecell[c]{N/A} \\
\cdashline{1-3}[0.5pt/1pt]
\makecell[c]{rs9907824} & \makecell[c]{NXN} & \makecell[c]{Alzheimer's Disease \\ \cite{blanco2022nxn}} \\

\bottomrule
\end{tabular}
\caption{Associated SNPs with gene identified by DFS}
\label{table:dfs}
\end{table}

\begin{table}[htbp!]
\centering
\begin{tabular}{p{3cm} p{3cm} p{8cm}}
\toprule
\makecell[c]{SNP} & \makecell[c]{Genes} & \makecell[c]{Reported brain-related/cognitive trait(s)} \\
\midrule

\makecell[c]{rs17068548} & \makecell[c]{SYNPR} & \makecell[c]{N/A} \\
\cdashline{1-3}[0.5pt/1pt]
\makecell[c]{rs12499028} & \makecell[c]{KLF3-AS1} & \makecell[c]{Improve Cerebral Ischemia-Reperfusion Injury \\ \cite{cao2024msc} } \\
\cdashline{1-3}[0.5pt/1pt]
\makecell[c]{rs7807724} & \makecell[c]{IQUB} & \makecell[c]{Bipolar Disorder with Anxiety Disorders \\ \cite{kerner2013rare} } \\
\cdashline{1-3}[0.5pt/1pt]
\makecell[c]{rs1008917} & \makecell[c]{RBPMS2} & \makecell[c]{ Alzheimer Disease \\ \cite{patel2021set}} \\
\cdashline{1-3}[0.5pt/1pt]
\makecell[c]{rs2385522} & \makecell[c]{FER1L6} & \makecell[c]{N/A} \\
\cdashline{1-3}[0.5pt/1pt]
\makecell[c]{rs2456200} & \makecell[c]{ITGA1} & \makecell[c]{N/A} \\
\cdashline{1-3}[0.5pt/1pt]
\makecell[c]{rs2075650} & \makecell[c]{TOMM40} & \makecell[c]{Mitochondrial Dysfunction \\ \cite{lee2021tomm40}} \\
\cdashline{1-3}[0.5pt/1pt]
\makecell[c]{rs4072374} & \makecell[c]{ RNASEH1} & \makecell[c]{N/A} \\

\bottomrule
\end{tabular}
\caption{Associated SNPs with gene identified by Lasso}
\label{table:lasso}
\end{table}


\subsection{Implementation details}
In this section, we present the technical details of the implementation of different methods. 


The LassoNet is implemented with the LassoNet package in \href{https://github.com/lasso-net/lassonet}{github.com/lassonet}.
The dropout rate was set to be $0$ by default. Two hidden layers are considered in LassoNet, with widths being 100 and 50 respectively. For all the cases in the simulation studies, we use the ``path'' function provided by the official package, we set the ``lambda\_start'' to be $10$ and the ``path\_multiplier'' to be $1.25$ so that we could achieve a balance between performance and efficiency. Since LassoNet will train a series of models with different penalty level $\lambda$, we select the model with lowest validation MSE which is partitioned from $20\%$ of the training set.

The DFS algorithm was implemented using the source code available at \href{https://github.com/cyustcer/Deep-Feature-Selection}{github.com/Deep-Feature-Selection}. The initial learning rate and the weight decay rate were set to their default values of 0.1 and 0.0025, respectively. Similar to LassoNet, we utilized a default neural network structure with two hidden layers with widths $100$ and $50$. 
The parameter $Ts$, which controls the optimization on the given support, was set to 25 for default nonlinear case in their package. We note that increasing the value of $Ts$ can improve accuracy a little bit, but it would also significantly increase the training time.

All the experiments are conducted using servers equipped with Intel Xeon Gold 5218R CPUs. The LassoNet and DFS methods are implemented with CUDA acceleration and run on NVIDIA GeForce RTX 3090 GPUs.

\subsection{Simulation Results on Correlation Intensity $\rho$}
In addition to the experiments reported in the main text, we further evaluate the robustness of our method and the baselines under correlated features by varying the correlation intensity $\rho$. As described, the covariance matrix is defined as $\Sigma_{j,k} = \rho^{|j-k|}$ for $1 \leq j,k \leq p$. In the main text, we set $\rho = 0$ (i.e., independent features). Here, we extend the analysis by allowing $\rho$ to range from $0$ to $0.5$, while keeping the sample size fixed at $n=2500$ and the feature dimension at $p=2000$. The results are presented in Figure \ref{fig:high_rho}.

\begin{figure}[htbp!]
    \centering
    \includegraphics[width=1\linewidth]{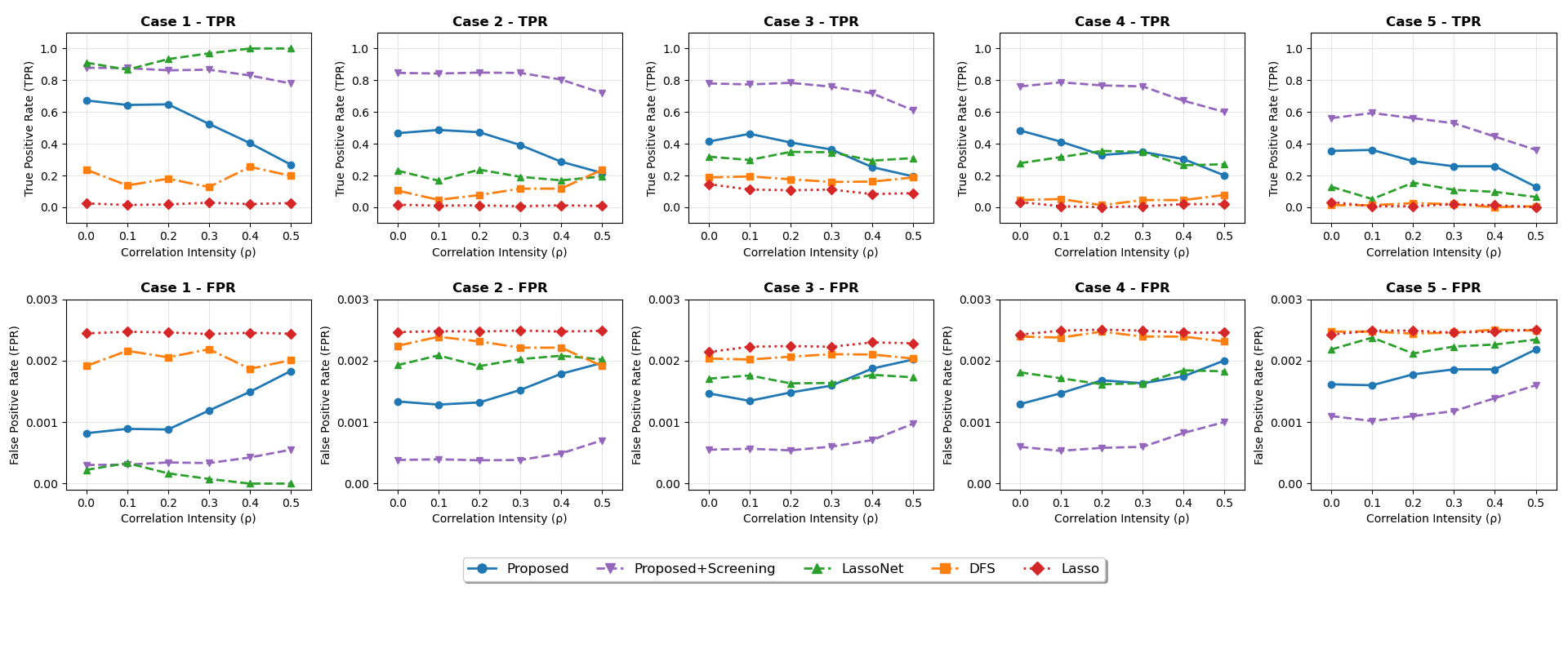}
    \caption{TPR and FPR of different methods with respect to different correlation intensity $\rho$ }
    \label{fig:high_rho}
\end{figure}

As expected, the performance of our two approach's  decline as $ \rho $ increases. Nevertheless, our proposed approach with screening mechanism consistently outperform the deep learning baselines across nearly all settings. The sole exception is Case 1, where LassoNet achieves near-perfect performance regardless of the value of $\rho$.

\subsection{Covariance Estimation in High Dimensional Scheme}
As mentioned in the main text, when $p >n$, the  empirical covariance  $\hat\bSigma=(1/n)\sum_{i=1}^n \bx_i\bx_i^\T$ is singular and therefore cannot be directly employed for the second-order score $\bT(\bx)$, which relies on the inverse covariance for Gaussian inputs. To address this challenge and evaluate the robustness of our approach, we estimate the precision matrix using the graphical lasso \citep{friedman2008sparse}. We conduct simulations with $p=3000$ and $n=2000$, drawing covariates from $\mathcal{N}(\mathbf{0}, \bSigma)$ where $\Sigma_{j,k} = \rho^{|j-k|}$ and $\rho \in [0, 0.6]$. The regularization parameter $\alpha$ is tuned for each $\rho$ to ensure the estimated precision matrix is positive definite and numerically stable.

\begin{table}[htbp!]
\centering
\caption{High Dimensional Feature Selection Performance with Unknown Covariance Matrix}
\label{tab:graphical_lasso}
\begin{tabular}{S[table-format=1.1] 
                 S[table-format=1.3] 
                 S[table-format=1.4] 
                 S[table-format=1.2e-2]}
\toprule
{$\rho$} & {$\alpha$} & {TPR} & {FPR} \\
\midrule
0.0 & 0.05 & 1.0000 & 0.00e+0 \\
0.1 & 0.05 & 1.0000 & 0.00e+0 \\
0.2 & 0.1 & 1.0000 & 0.00e+0 \\
0.3 & 0.1 & 0.9862 & 2.30e-5 \\
0.4 & 0.1 & 0.9724 & 4.61e-5 \\
0.5 & 1 & 0.9310 & 1.15e-4 \\
0.6 & 1 & 0.7862 & 3.57e-4 \\
\bottomrule
\end{tabular}
\end{table}
The results are reported in Table \ref{tab:graphical_lasso}. Our method remains robust even when $p>n$: TPR exceeds 78\% and FPR remains near zero even at $\rho=0.6$. Notably, when $\rho = 0.6$, the regularized condition number is approximately 114, yet the method still performs well, demonstrating its practical applicability to real-world data.

\subsection{Computational Efficiency of 2-step Approach}

In Section \ref{mse} of the main text, we adopt a two-step approach so that our model could fit the data. Here, we compare the computational efficiency of our method with the competing deep learning-based approaches, with detailed runtimes summarized in Table \ref{tab:detailed_time_comparison_sample_size}.
The computation times of LassoNet and DFS are highly sensitive to their respective tuning parameters — the path multiplier ($M_\lambda$) for LassoNet and the intersection parameter ($Ts$) for DFS. Accordingly, we report runtimes across a range of realistic values for these parameters. As shown in the table, our two-step approach consistently exhibits substantially superior computational efficiency compared to both LassoNet and DFS across all evaluated sample sizes ($n = 100$ to $5000$). This advantage arises because the feature selection phase of our method does not rely on iterative gradient-based optimization of a neural network objective, thereby avoiding the heavy computational burden associated with training deep models on the full high-dimensional input.

\begin{table}[htbp!]
\centering
\caption{Runtime Comparison for Different Methods varying Sample Sizes ($n$) (Mean ± SD, Seconds)}
\label{tab:detailed_time_comparison_sample_size}
\begin{adjustbox}{max width=\textwidth} 
\begin{threeparttable}
\small 
\begin{tabular}{lccccc}
\toprule
Method & \multicolumn{5}{c}{Sample Size $n$} \\
\cmidrule{2-6}
 & 100 & 500 & 1000 & 2000 & 5000 \\
\midrule
2-Step Approach & 0.027 ± 0.004 & 0.185 ± 0.010 & 0.188 ± 0.041 & 0.126 ± 0.030 & 0.849 ± 0.038 \\
DFS (Ts=25) & 0.417 ± 0.010 & 0.364 ± 0.011 & 0.388 ± 0.022 & 0.659 ± 0.035 & 1.254 ± 0.130 \\
DFS (Ts=40) & 0.813 ± 0.017 & 0.431 ± 0.028 & 0.433 ± 0.060 & 0.801 ± 0.079 & 1.967 ± 0.221 \\
DFS (Ts=55) & 1.089 ± 0.008 & 0.605 ± 0.084 & 0.565 ± 0.047 & 1.042 ± 0.059 & 2.726 ± 0.349 \\
LassoNet ($M_\lambda$=1.3) & 0.940 ± 0.009 & 2.114 ± 0.040 & 2.203 ± 0.075 & 3.565 ± 0.287 & 8.883 ± 0.921 \\
LassoNet ($M_\lambda$=1.2) & 1.231 ± 0.015 & 2.585 ± 0.079 & 2.414 ± 0.155 & 4.289 ± 0.179 & 9.679 ± 1.132 \\
LassoNet ($M_\lambda$=1.1) & 1.665 ± 0.015 & 3.657 ± 0.038 & 4.065 ± 0.230 & 5.417 ± 0.178 & 13.235 ± 1.536 \\
\bottomrule
\end{tabular}
\begin{tablenotes}
\footnotesize
\item Note: Experiments are done under Case $1$ with parameters $p=200$, $k_1=5$, $s=5$.
\end{tablenotes}
\end{threeparttable}
\end{adjustbox}
\end{table}

\subsection{Simulation Results on Selection of $k_1$ and $s$}
For selecting $k_1$, a straightforward approach is to use the eigengap of $(1/n)\sum_{i=1}^n y_i\bT(\bx_i)$, i.e., $|\lam_k|-|\lam_{k+1}|$. Intuitively, one can identify $k_1$ at the index where the eigengap drops sharply, i.e., choose $k$ where $|\hat{\lambda}_{k-1}| - |\hat{\lambda}_{k}|$ is much larger than $|\hat{\lambda}_{k}| - |\hat{\lambda}_{k+1}|$. More formally, inspired by the gap-statistic idea for determining the number of clusters \citep{tibshirani2001gap} we define the absolute eigengap ratio $r(k)=\frac{|\hat{\lambda}_{k-1}| - |\hat{\lambda}_{k}|}{|\hat{\lambda}_{k}| - |\hat{\lambda}_{k+1}| + \gamma_{\text{reg}}} $, where $\gamma_{\text{reg}}$ is a small regularization constant to ensure numerical stability, and select $k_1$ by maximizing $r(k)$. This ratio-based rule yields stable choices of $k_1$ across a wide range of settings. Taking Case 1 as an illustration, Figure \ref{fig:k1} below illustrates $k_1$ selection via both the eigengap and the ratio metric, where $k_1 = 5$ is clearly identified by both methods.

\begin{figure}[htbp!]
\centering
\includegraphics[width=0.8\linewidth]{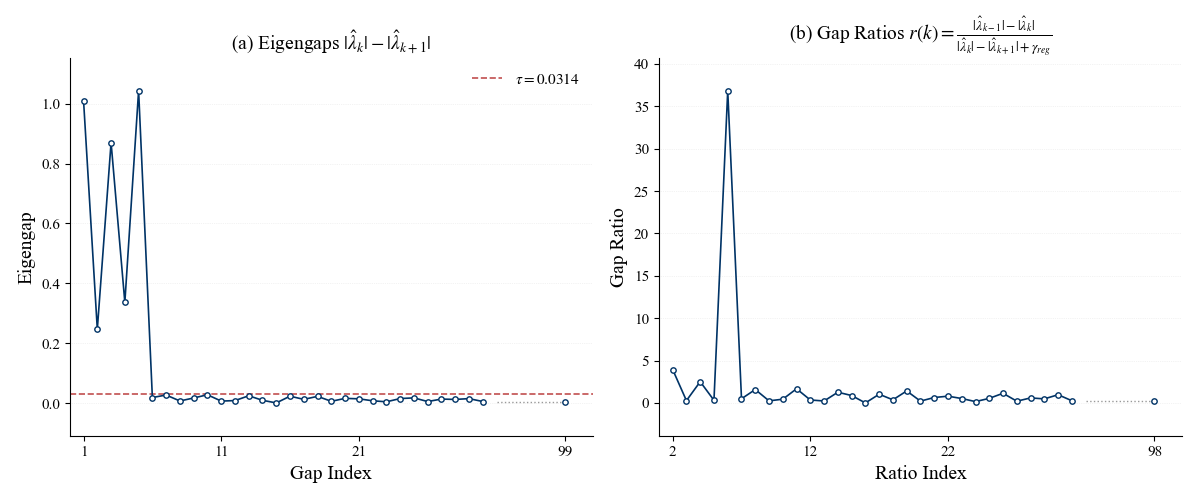}
\caption{Visualization of the Eigenvalues of the empirical $\mathbb{E}[y T(\bx)]$}
\label{fig:k1}
\end{figure}

To determine the number of selected features $s$ in a fully data-driven manner, we adapt a Bayesian Information Criterion (BIC)-based approach inspired by the nonlinear feature selection framework of \citet{chen2021nonlinear}. As described in Section~\ref{sec_prediction} of the main text, we first rank the features using our method, then retrain a neural network using only the top $s$ features. This two-step strategy yields substantially lower prediction MSE compared to training on the full set of variables. To choose $s$, we evaluate the refitted neural network for a range of candidate $s$ values and compute the corresponding BIC:
\begin{align}
   \text{BIC} = n \cdot \ln\big(\text{MSE}(s)\big) + \lam_s \cdot \ln(n),
\end{align}
where $\text{MSE}(s)$ is the prediction MSE obtained from a neural network trained on the $s$ selected features. Strictly speaking, $\lam_s$ corresponds to the total number of parameters in a neural network with $s$ features. To avoid the intricacies of architectural tuning, we fix $\lam_s=100 s$ throughout all simulation experiments.

\begin{figure}[htbp!]
    \centering
    \includegraphics[width=0.75\linewidth]{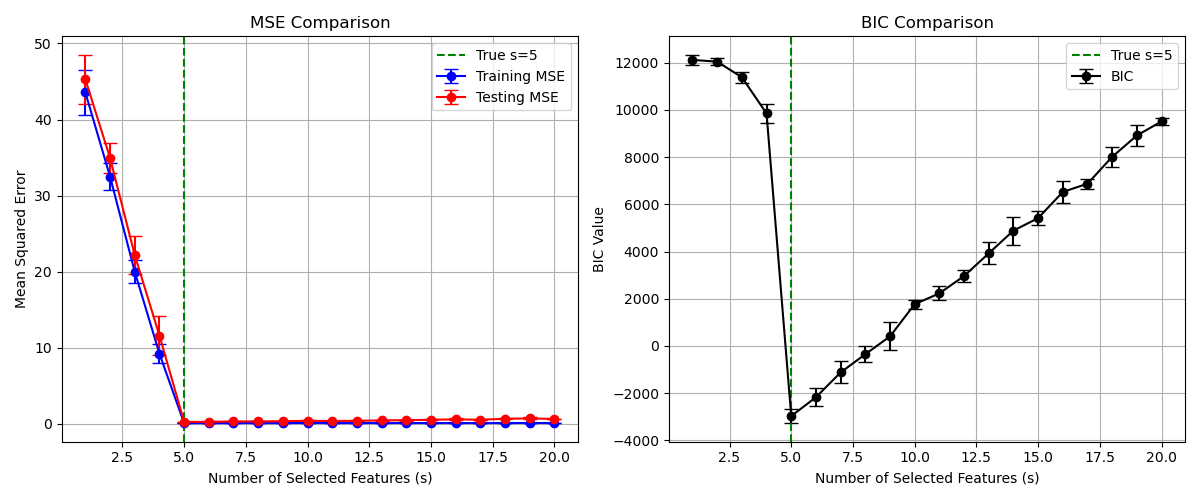}
    \caption{Selection of $s$ using the BIC criterion in Case~1 ($n=3000$, $\lambda_{\text{s}}=100s$).}
    \label{fig:kappa}
\end{figure}

As an illustration, Figure~\ref{fig:kappa} displays the BIC curve for Case~1 with $n=3000$. The criterion correctly identifies the true number of active features $s=5$, achieving a markedly lower BIC than neighboring values. This demonstrates that the proposed BIC-based procedure can reliably recover the correct sparsity level in practice, providing a principled and fully automated way to choose $s$ without prior knowledge of the ground truth form of function.

\end{document}